\documentclass{article} 
\usepackage{iclr2026_conference,times}

\usepackage{amsmath,amsfonts,bm}

\def\eqref#1{equation~\ref{#1}}

\def\1{\bm{1}}

\DeclareMathAlphabet{\mathsfit}{\encodingdefault}{\sfdefault}{m}{sl}
\SetMathAlphabet{\mathsfit}{bold}{\encodingdefault}{\sfdefault}{bx}{n}

\newcommand{\E}{\mathbb{E}}

\DeclareMathOperator{\Tr}{Tr}

\DeclareMathOperator{\Cov}{Cov}

\usepackage[dvipsnames]{xcolor}
\usepackage[colorlinks=true,
            linkcolor=blue,
            citecolor=Rhodamine,
            urlcolor=blue]{hyperref}
\usepackage{url}
\usepackage{amsmath}
\usepackage{amssymb}
\usepackage{mathtools}
\usepackage{amsthm}
\usepackage[utf8]{inputenc} 
\usepackage[T1]{fontenc}    	
\usepackage{url}            		
\usepackage{booktabs}       	
\usepackage{amsfonts}       	
\usepackage{nicefrac}       	
\usepackage{microtype}      	
\usepackage{xcolor}         		
\usepackage{amsmath}
\usepackage{mathrsfs}
\usepackage{amssymb}
\usepackage{graphicx}
\usepackage{algorithm}
\usepackage{makecell}
\usepackage{lipsum}
\usepackage{enumitem}
\usepackage{graphicx}
\usepackage{listings}
\usepackage{tcolorbox}
\usepackage{multirow}
 \usepackage{algorithm}
 \usepackage{algpseudocode}
 \usepackage{booktabs}
\usepackage{multirow}
\usepackage{graphicx}
\usepackage[capitalise]{cleveref}
\crefname{equation}{Eq.}{Eqs.}
\iclrfinalcopy
\title{Scaling Online Optimal Transport Minibatches for Flow Matching}

\theoremstyle{plain}
\newtheorem{theorem}{Theorem}[section]
\newtheorem{proposition}[theorem]{Proposition}
\newtheorem{lemma}[theorem]{Lemma}

\theoremstyle{definition}
\newtheorem{definition}[theorem]{Definition}
\newtheorem{assumption}[theorem]{Assumption}
\newtheorem{fact}[theorem]{Fact}
\theoremstyle{plain}
\newtheorem{remark}[theorem]{Remark}

\begin{document}

\author{Peter Chen \qquad Sergey Levine \\
\\
University of California, Berkeley }

\maketitle

\begin{abstract}
A central challenge in scaling flow models is maintaining stability in training and integration. In flow matching, this stability depends critically on how source samples are paired with data. Independent random pairing introduces substantial stochasticity, motivating minibatch optimal transport (OT) to construct online source-target couplings that reduce this randomness \citep{tong2023improving}. However, as applications extend to larger and higher-dimensional datasets, small minibatches can limit these benefits, while scaling the coupling to substantially larger minibatches at every iteration is computationally prohibitive. Moreover, how OT coupling affects the full stochastic gradient and integration remains theoretically unresolved. In this work, we develop a scalable online entropic OT approximation that enables substantially larger source-target couplings while preserving the benefits of OT. We close the theoretical gap for both unregularized and entropic OT by proving strict reductions in full stochastic-gradient variance and integration error relative to independent pairing, and further characterize how these improvements strengthen as the online coupling minibatch grows. Experiments on single-cell gene interpolation-flow and image-flow training validate these predictions, showing that larger online couplings consistently improve training stability and final generation performance.
\end{abstract}

\section{Introduction} \label{sec:intro}
Much of the recent progress in generative AI has been driven by scale. Autoregressive language models benefit from a standardized learning objective: next-token prediction under a softmax distribution. Once this objective is fixed, much of the challenge of scaling shifts to systems and optimization. Flow models introduce an additional dimension: their supervision is constructed along a generative path from noise to data, and the resulting regression targets depend on how source noise samples are coupled with target data samples. Thus, scaling flow models requires consideration not only of how the model is optimized, but also of how its training targets are constructed.

The choice of such source-target coupling is therefore part of the learning problem itself. In flow matching, each training example pairs a source sample from the noise distribution with a target data sample and regresses the velocity along the resulting interpolation path. Under the standard independent coupling, these pairings are random. Consequently, nearby intermediate states can arise from different endpoint pairs and receive substantially different velocity targets, increasing ambiguity in the regression problem and the variance of stochastic gradients. A natural way to reduce this randomness is to determine the source-target correspondence during training. 

Optimal transport (OT) then provides a natural solution by constructing a coupling between source and target samples within each minibatch during training. Such online coupling minimizes the transport cost between the two empirical distributions, yielding more coherent source-target assignments that induce straighter flows and reduce the ambiguity of independent pairing \citep{pooladian2023multisample,tong2023improving}. However, as flow matching is applied to larger and more heterogeneous datasets, such as high-dimensional single-cell trajectories and large-scale image datasets, a fixed small coupling minibatch can become increasingly restrictive, where each minibatch only exposes OT to a limited subset of the source and target distributions, constraining the candidate correspondences available at each iteration. As the underlying data distribution grows in diversity and complexity, it therefore becomes natural to scale the OT minibatch accordingly.

Still, scaling the minibatch size is computationally challenging, where dense pairwise interactions in standard OT restrict practical minibatches to hundreds of samples. Nevertheless, recent advances in scalable entropic OT make substantially larger contexts accessible. In particular, FlashSinkhorn \citep{ye2026flashsinkhorn} implements Sinkhorn iterations with FlashAttention-style HBM kernels without materializing the full coupling matrix in high-bandwidth memory. To this end, we propose a simple, tractable approximation by replacing online OT coupling with scalable entropic OT, enabling coupling contexts to grow from hundreds to thousands of samples during training.

The price of this tractability is that entropic regularization changes the coupling itself. The regularization strength \(\varepsilon\) controls this deviation from standard OT, while the minibatch size \(k\) controls the scale of the online coupling. Understanding whether the stability benefits of OT persist therefore requires characterizing the joint effects of regularization and coupling scale. Moreover, this question remains unresolved even for standard unregularized minibatch OT. Existing analyses motivate OT through straighter paths and reduced conditional uncertainty \citep{pooladian2023multisample,boite2026expected,zhang2026on,chambolle2026limit}, but do not provide a general strict characterization of its effect on the full stochastic-gradient variance or the induced numerical integration error.

By making larger online couplings feasible, we determine whether their stability benefits survive the approximation required for scaling. Our contributions can be summarized below:

\vspace{-0.5em}
\paragraph{Contributions.} We provide a unified and principled characterization of how online OT coupling affects optimization and numerical integration. On the optimization side, we characterize the full stochastic-gradient variance by (i) proving that population quadratic OT strictly reduces this variance relative to independent coupling, (ii) showing that unregularized minibatch-OT variance reduces as the minibatch size $k$ grows, and (iii) extending same guarantees to entropic-regularized couplings. On the integration side, we establish parallel strict reductions in Euler integration error and characterize how these guarantees improve with scaled minibatch size. These results provide a better theoretical understanding for using and scaling entropic OT as a replacement to the standard online OT couplings.

\vspace{-0.5em}
\paragraph{Related works.} OT-based conditional flow matching (OT-CFM) improves both training and generation through better source-target pairings \citep{tong2023improving,pooladian2023multisample}. Given the difficulty to scale OT-CFM, recent works scale the minibatch context by separating pairing from training: \citet{zhang2026on} precomputed million-scale couplings offline, which requires buffering enough pairs to cover the entire training run before optimization begins; semidiscrete approximation methods offer better efficiency by reusing fitted dataset-level potentials, but need a separate potential-optimization stage \citep{mousavi2026flow} or face non-convergence when scaled to large unconditional datasets \citep{kong2026alignflow}. Therefore, we return to the original challenge, by finding an affordable and stable method for scaling the minibatch that flow matching recomputes from fresh samples at every training iteration.

The existing theory has so far addressed separate pieces. \citet{tong2023improving} connect population OT to dynamic OT, while \citet{pooladian2023multisample} study conditional gradient covariance and asymptotic BatchOT straightness. Finite-batch theory remains limited: \citet{zhang2026on} characterize uncertainty in the resulting OT coupling, while \citet{boite2026expected} analyze its integration benefit only in a restricted two-atom model; more recent work also studies full parameter-gradient variance in restricted Gaussian models \citep{reu2026gradient} and integration error through the geometry of a prescribed velocity field \citep{tao2026role}. Thus, no prior result establishes that replacing independent pairing with minibatch OT strictly reduces full stochastic-gradient variance while inducing a general integration advantage; we defer further related work to Appendix~\ref{app:related-works}.

\section{Preliminaries}\label{sec:prelim}
We introduce flow matching (FM) and coupling definitions used throughout the later analyses.

\vspace{-0.5em}
\paragraph{Continuous normalizing flows.} Let \(q_0,q_1\in\mathcal P_2(\mathbb R^d)\) denote the source (noise) and target (image-data) distributions, respectively. A time-dependent vector field \(v_\theta:[0,1]\times\mathbb R^d\to\mathbb R^d\), parameterized by \(\theta\in\mathbb R^p\), assigns an instantaneous velocity to each time-state pair \((t,x)\), thereby defining a flow from the source to the target. Specifically, a continuous normalizing flow (CNF) \citep{lipman2022flow} transports source samples through the flow map
\[
\frac{\mathrm d}{\mathrm dt}\phi_t(x)=v_\theta(t,\phi_t(x)),\qquad \phi_0(x)=x,
\]
with the goal that \((\phi_1)_\#q_0\approx q_1\). Flow matching trains this vector field from paths induced by a coupling of the source and target endpoints.

\vspace{-0.5em}
\paragraph{Coupling-based FM.} Let $\pi$ be a coupling of $q_0$ and $q_1$. Each FM training tuple is generated first by sampling $(X_0,X_1)\sim\pi$ and then independently sampling $t\sim\mathcal U([0,1])$. Unless stated otherwise, we take all expectations $\E_\pi[\cdot]$ and covariances $\Cov_\pi(\cdot)$ over the randomness induced by the sampling procedure associated with coupling $\pi$. For simplicity, we denote the complete stochastic FM training tuple by $\xi^\pi=(t,X_0,X_1)$ and its associated variables, including intermediate interpolation point $X_t^\pi$ and pathwise velocity target $U^\pi$, by
\[
    X_t^\pi=(1-t)X_0+tX_1,\qquad U^\pi=X_1-X_0.
\]
At a fixed model input \((t,X_t^\pi)\), different endpoint pairs may produce different velocity targets. Given one sampled FM training tuple \(\xi^\pi\), its loss is \[\ell_\pi(\theta;\xi^\pi)=\|v_\theta(t,X_t^\pi)-U^\pi\|_2^2.\] Under this squared loss, the optimal deterministic velocity prediction at input \((t,x)\) is the conditional mean, defined by \(u^\pi(t,X_t^\pi)\!\coloneqq\!\mathbb E_\pi[U^\pi\mid t,X_t^\pi]\). The conditional covariance, \(\Cov_\pi(U^\pi\mid t,X_t^\pi)\), then measures how much the possible velocity targets still vary at that input; the population FM objective further averages $\ell_\pi(\theta;\xi^\pi)$ over all possible training tuples,
\[
    \mathcal L_\pi(\theta):=\mathbb E_\pi[\ell_\pi(\theta;\xi^\pi)].
\]
Let $J_\theta(t,x)\!\coloneqq\!\nabla_\theta v_\theta(t,x)\in\mathbb R^{d\times p}$, the gradient over $\ell_\pi(\theta;\xi^\pi)$ can be expressed as
\[
    G_\pi(\theta)\coloneqq\nabla_\theta\ell_\pi(\theta;\xi^\pi)=2J_\theta(t,X_t^\pi)^\top\bigl(v_\theta(t,X_t^\pi)-U^\pi\bigr).
\]
Whenever \(\mathbb E_\pi[\|G_\pi(\theta)\|_2^2]<\infty\), we can derive the total variance of the per-sample stochastic gradient induced by sampling FM training tuples under coupling \(\pi\):
\[
    \sigma_{\text{grad},\pi}^2(\theta):=\Tr\{\Cov_\pi(G_\pi(\theta))\}=\mathbb E_\pi\!\left[\|G_\pi(\theta)-\mathbb E_\pi G_\pi(\theta)\|_2^2\right].
\]
Here, \(\Cov_\pi(G_\pi(\theta))\in\mathbb R^{p\times p}\) is the covariance matrix of the per-sample stochastic gradient, whereas \(\sigma_{\text{grad},\pi}^2(\theta)\) is its trace, measuring the total sample-to-sample gradient variance across FM training tuples under coupling $\pi$. The law of total covariance further gives us the following decomposition,
\begin{equation}
\label{eq:var-decomposition}
\sigma_{\text{grad},\pi}^2(\theta)=\mathbb E_\pi\!\left[\Tr\{\Cov_\pi(G_\pi(\theta)\mid t,X_t^\pi)\}\right]+\Tr\!\left\{\Cov_\pi\!\left(\mathbb E_\pi[G_\pi(\theta)\mid t,X_t^\pi]\right)\right\}.
\end{equation}
The first term is the gradient noise from target ambiguity at a fixed model input, and the second term is the variation of the conditional mean gradient across model inputs.

\section{Scalable OT for Flow Matching}\label{sec:method}

We refer to our method as conditional flow matching with entropic OT (EOT-CFM). Our approach applies a larger context of source and target samples to construct better couplings, with FlashSinkhorn \citep{ye2026flashsinkhorn} enabling efficient computation at scale. Specifically, at each training iteration, we draw fresh source and target minibatches, compute an approximate entropic transport plan, and sample source-target pairs from this plan to form the regression batch.

We consider three coupling constructions as background: (i) independent coupling \(\pi_{\text{ind}}\), which pairs samples from \(q_0\) and \(q_1\) randomly; (ii) optimal OT coupling \(\pi^\star\), which minimizes \(\mathbb E_\pi[\|X_0-X_1\|_2^2]\) over all noise and data points\footnote{If \(q_0\ll\mathcal L^d\), Brenier's theorem \citep{brenier1991polar} states that this coupling is unique, with the deterministic transport map: \(\pi^\star\!=\!(\mathrm{id},T)_\#q_0\), where \(T\!=\!\nabla\varphi\) is \(q_0\)-almost everywhere for some convex function \(\varphi\).}; (iii) in practice, minibatch OT solves the corresponding problem on source and target minibatches of size \(k\), where we denote this empirical plan as \(\widehat\pi_k\).

To formalize our method, we denote \(k\) as the number of samples to be paired in the coupling minibatch and \(\mathcal{B}\leq k\) as the regression batch used per model update. Given independent samples \(x_i\sim q_0\) and \(y_j\sim q_1\), we write \(C_{ij}=\|x_i-y_j\|_2^2/(2d)\), with dimension normalization that preserves the unregularized OT matching. Assigning uniform mass \(1/k\) to each sample, we aim to solve the following entropic transport plan:
\begin{equation}\label{eq:eot}
\widehat\pi_{k,\varepsilon}
=\operatorname*{arg\,min}_{\substack{
P\geq0,\;P\mathbf1=\mathbf1/k\\
P^\top\mathbf1=\mathbf1/k}}
\left\{
\langle\mathbf C,P\rangle
+\varepsilon\operatorname{KL}\!\left(
P\middle\|\mathbf1\mathbf1^\top/k^2
\right)
\right\}.
\end{equation}
The regularization $\varepsilon>0$ balances between the cost and entropy. Specifically, the exact row conditional is $\mathbb P(J=j\!\mid \!I=i)=k\cdot\widehat\pi_{k,\varepsilon,ij}$, which normalizes probabilities from the approximate target potential $\widetilde g$. Algorithm~\ref{alg:eot-cfm} then constructs the interpolation points $X_{t_b,b}$ and velocity targets $U_b$ defined in \S\ref{sec:prelim}. Without materializing the full transport plan or invoking a separate JAX process, FlashSinkhorn solves the minibatch coupling directly within the PyTorch training loop using tiled cost evaluations and dual potentials, with iteration count $L$ that controls the numerical approximation. 

We also note that our later EOT analyses consider exact entropic transport plans. For computational efficiency, FlashSinkhorn uses a finite number of Sinkhorn iterations, introducing additional numerical approximation error that is not included in the bounds below.

\begin{algorithm}[t]
\caption{Continuous normalizing flow with online entropic couplings (EOT-CFM)}\label{alg:eot-cfm}
\small
\begin{algorithmic}[1]
\Require Minibatch size $k$, batch size $\mathcal{B}$, regularization $\varepsilon>0$, Sinkhorn iterations $L$, training steps $T$.
\For{$t=1,\ldots,T$}
    \State Sample noise and data from source $x_i\sim q_0$ and target $y_j\!\sim\!q_1$ with minibatch size $k$
    \State Approximate the target dual potential $\widetilde g$ of \cref{eq:eot} with $L$ FlashSinkhorn iterations
    \State Select $\mathcal{B}$ distinct source indices $I_1,\ldots,I_\mathcal{B}$ uniformly
    \For{$b=1,\ldots,\mathcal{B}$}
        \State Compute $(\widetilde r_{I_b,j})_{j=1}^{k}
        =\operatorname{softmax}\!\left(
        \left(\tfrac{\widetilde g_j-C_{I_bj}}{\varepsilon}\right)_{j=1}^{k}
        \right)$
        \State Sample $J_b\sim\operatorname{Cat}\!\left(
        (\widetilde r_{I_b,j})_{j=1}^{k}\right)$
        and $t_b\sim\mathcal U([0,1])$
    \EndFor
    \State Compute $\widehat{\mathcal L}(\theta)=\frac{1}{\mathcal{B}d}\sum_{b=1}^\mathcal{B}\|v_\theta(t_b,X_{t_b,b})-U_b\|_2^2$
    \State Update $\theta$ with the optimizer using $\nabla_\theta\widehat{\mathcal L}(\theta)$
\EndFor
\end{algorithmic}
\end{algorithm}

\section{Main Results}

\subsection{On stochastic gradient variance}\label{sec:thy}

We first study the benefit of scaling OT minibatch size through its effect on stochastic-gradient variance. The closest existing result we are aware of is \citet[Lemma~3.2]{pooladian2023multisample}, which analyzes the conditional gradient variance of minibatch OT at model input \((t,x)\), corresponding to the fixed-input component in \cref{eq:var-decomposition}. However, it remains a gap to prove that OT coupling strictly reduces the stochastic-gradient variance relative to independent coupling.

Therefore, in \S\ref{sec:grad-var-std}, we begin by establishing a strict gradient-variance reduction by filling the gap between optimal coupling \(\pi^\star\) and independent coupling \(\pi_{\text{ind}}\). In \S\ref{sec:grad-var-entro}, we extend this result to entropic-OT couplings. Finally, in \S\ref{sec:grad-var-scale}, we establish that scaling to larger minibatch size could further reduce gradient variance for both unregularized and entropic OT couplings.

\subsubsection{Gradient variance reduction: standard unregularized OT}\label{sec:grad-var-std}
We first state two standard assumptions used in theoretical analyses of flow matching and neural differential equations: realizability of the population target field and global optimality of the population objective \citep{islam2025longitudinal,khan2026flow,baheri2026geometrygrounded}. Since the population regression target depends on the coupling \(\pi\), we note that these assumptions are imposed coupling-wise; all the detailed math derivations and proofs for the later theoretical results are deferred to Appendix~\ref{app:proof}.

\begin{assumption}[Realizability]
\label{ass:real}
For every coupling \(\pi\) under comparison, the model class \(\mathcal V=\{v_\theta:\theta\in\mathbb R^p\}\) contains the coupling-induced conditional mean velocity, i.e., there exists \(\theta_\pi\in\mathbb R^p\) such that \(v_{\theta_\pi}(t,X_t^\pi)=u^\pi(t,X_t^\pi)\) almost surely.
\end{assumption}

\begin{assumption}[Optimality]
\label{ass:opt}
For every coupling \(\pi\) under comparison, \(\theta_\pi^\star\) is a global minimizer of \(\mathcal L_\pi\), the map \(\theta\mapsto v_\theta(t,X_t^\pi)\) is differentiable at \(\theta_\pi^\star\) almost surely, and \(\mathbb E_\pi[\|G_\pi(\theta_\pi^\star)\|_2^2]<\infty\).
\end{assumption}
\vspace{-0.5em}
From an optimization perspective, gradient noise at an optimum provides a reference quantity for analyzing SGD convergence \citep{gower2019sgd}. In flow matching setting, evaluating gradients at population optima also removes the variance contribution from imperfect fitting of the conditional-mean velocity, providing a common basis for comparing gradient noise across couplings and motivating later flow matching analyses for using optimality conditions to study stochastic gradients \citep{vaitl2022path,vaitl2022gradients,kim2025adaptive,zhang2025heuristic,vaitl2025path}.
\begin{proposition}
\label{prop:variance-at-optimum}
Under \emph{Assumptions~\ref{ass:real}} and~\emph{\ref{ass:opt}}, along with a trainable additive output bias \(b\in\mathbb R^d\) under OT-CFM architecture (see \emph{Fact~\ref{fact:affine}} for detail), we have
\begin{equation}\label{eq:var-at-optimum}
\sigma_{\emph{grad},\pi}^2(\theta_\pi^\star)=4\,\mathbb E_\pi\left[\Tr\!\left\{J_{\theta_\pi^\star}(t,X_t^\pi)^\top\Cov_\pi(U^\pi\mid t,X_t^\pi)J_{\theta_\pi^\star}(t,X_t^\pi)\right\}\right].
\end{equation}
\end{proposition}
Thus, at the population optimum, conditional target covariance is mapped into parameter-gradient space by the model Jacobian. Its average magnitude also satisfies
\[
    \mathbb E_\pi\!\big[\!\Tr\{\Cov_\pi(U^\pi\mid t,X_t^\pi)\}\big]=\mathbb E_\pi\!\left[\|U^\pi-u^\pi(t,X_t^\pi)\|_2^2\right].
\]
To obtain an explicit positive lower bound for independent coupling, we directly follow from the Assumption~(A2') in \citet{pooladian2023multisample}, using the following technical sufficient condition:

\begin{assumption}
\label{ass:positive-common-support}
There exist a compact convex set \(\Omega\subset\mathbb R^d\) with nonempty interior, \(\alpha\in(0,1]\), \(\gamma>0\), and \(\alpha\)-H\"older functions \(f_i:\Omega\to[\gamma,\infty)\) such that, for \(i\in\{0,1\}\) and every Borel \(A\subseteq\mathbb R^d\),
\[
    q_i(A)=\int_{A\cap\Omega}f_i(x)\,\mathrm dx,\qquad \operatorname{supp}(q_i)=\Omega.
\]
\end{assumption}

Finally, combining the preceding variance identity with this common-support condition, we establish a strict gradient-variance gap between standard unregularized OT and independent coupling.

\begin{theorem}
\label{thm:strict-ot-variance}
Let \(\pi^\star=(\mathrm{id},T)_\#q_0\) be the unregularized OT coupling and \(\pi_{\emph{ind}}\) be the independent coupling under the quadratic cost. Suppose \emph{Assumptions~\ref{ass:real}}, \emph{\ref{ass:opt}}, and \emph{\ref{ass:positive-common-support}} hold for both couplings and define \(\omega_d=\mathcal L^d(B(0,1))\), then, for every \(\bar x\in\Omega^\circ\) and \(r>0\) such that \(B(\bar x,2r)\subset\Omega\), we have
\begin{equation}
\label{eq:strict-ot-gap}
\sigma_{\emph{grad},\pi_{\emph{ind}}}^2(\theta_{\pi_{\emph{ind}}}^\star)-\sigma_{\emph{grad},\pi^\star}^2(\theta_{\pi^\star}^\star)\geq\frac{4d\gamma^2\omega_d^2}{d+2}\,r^{2d+2}>0.
\end{equation}
\end{theorem}

\begin{remark}
Given \emph{Theorem~\ref{thm:strict-ot-variance}}, one can bound full-batch SGD variance  by single-sample variance via Jensen inequality (see \cref{eq:jen}), which extends similar later results also to batch-level.
\end{remark}

\subsubsection{Gradient-variance reduction: entropic OT}\label{sec:grad-var-entro}

We next study the influence of entropic regularization on gradient variance and introduce following setup inherited from previous minibatch flow matching analysis.

\begin{definition}[Uniformly bounded parameter Jacobians]\label{rem:bounded-jacobian}
Following the maximum-Jacobian setup in the flow-matching gradient-variance bound of \citet[Lemma~3.2]{pooladian2023multisample}, a coupling family \(\Pi\) has uniformly bounded parameter Jacobians if there exists \(M<\infty\) such that
\begin{equation}
\max_{t\in[0,1],\,x\in\mathbb R^d}
\|J_{\theta_\pi^\star}(t,x)\|_{\mathrm{op}}^2\leq M,
\qquad \forall\,\pi\in\Pi.
\end{equation}
\end{definition}

\begin{remark}
We apply \emph{Definition~\ref{rem:bounded-jacobian}} to obtain a common Jacobian bound $M$ for all couplings considered in later theorems, allowing us to translate target-uncertainty bounds into full stochastic-gradient variance comparisons. We also note that such Jacobian bound control is widely used in neural-network optimization \citep[Assumption~2]{oymak2019overparameterized} and recent flow-matching theory analysis \citep[Proposition~3]{zeng2025flow}.
\end{remark}

We use \(\pi_\varepsilon\) to denote the population entropic-OT coupling between \(q_0\) and \(q_1\):
\begin{equation}
\label{eq:population-eot}
\pi_\varepsilon\in\operatorname*{arg\,min}_{\pi\in\Pi(q_0,q_1)}\left\{\frac{1}{2d}\,\mathbb E_\pi[\|X_0-X_1\|_2^2]+\varepsilon\operatorname{KL}(\pi\|\pi_{\text{ind}})\right\}.
\end{equation}
To compare its variance to that under independent coupling, we first relate its irreducible flow matching regression error to its quadratic transport cost, where \(W_2\) denotes the 2-Wasserstein distance:

\begin{proposition}\label{prop:fm-loss-cost-gap}
Under \emph{Assumptions~\ref{ass:real}} and~\emph{\ref{ass:opt}}, every coupling \(\pi\) satisfies
\begin{equation}
\label{eq:fm-loss-cost-gap}
0\leq\mathcal L_\pi(\theta_\pi^\star)=\mathbb E_\pi[\Tr\{\Cov_\pi(U^\pi\mid t,X_t^\pi)\}]\leq\mathbb E_\pi[\|X_1-X_0\|_2^2]-W_2^2(q_0,q_1).
\end{equation}
\end{proposition}

\vspace{-0.5em}
Thus, the population loss at optimality is exactly the average target uncertainty remaining at a fixed model input. As an intermediate step, we establish Lemma~\ref{lem:eot-strict-small-epsilon} to show that sufficiently small entropic regularization strictly reduces this uncertainty comparing to independent coupling.

\begin{lemma}\label{lem:eot-strict-small-epsilon}
Suppose \emph{Assumptions~\ref{ass:real}} and~\emph{\ref{ass:opt}} hold for \(\pi_{\emph{ind}}\), \(\pi^\star\), and every \(\pi_\varepsilon\) under comparison, and assume \(\mathcal L_{\pi_{\emph{ind}}}(\theta_{\pi_{\emph{ind}}}^\star)>0\), which holds whenever \(q_0\) has a strictly positive Lebesgue density and \(q_1\) is not a Dirac mass, including the standard choice \(q_0=\mathcal N(0,I_d)\). Then, for any \(\varepsilon_0>0\) that satisfies
\begin{equation}
\label{eq:eot-small-epsilon-certificate}
\mathbb E_{\pi_{\varepsilon_0}}[\|X_1-X_0\|_2^2]-W_2^2(q_0,q_1)<\mathcal L_{\pi_{\emph{ind}}}(\theta_{\pi_{\emph{ind}}}^\star),
\end{equation}
every \(0<\varepsilon\leq\varepsilon_0\) gives us \cref{eq:eot-strict-population-loss-chain}, and such $\varepsilon_0$ is guaranteed to exist by \emph{Remark~\ref{rem:eot-epsilon-zero-existence}},
\begin{equation}\label{eq:eot-strict-population-loss-chain}
\mathcal L_{\pi^\star}(\theta_{\pi^\star}^\star)<\mathcal L_{\pi_\varepsilon}(\theta_{\pi_\varepsilon}^\star)<\mathcal L_{\pi_{\emph{ind}}}(\theta_{\pi_{\emph{ind}}}^\star).
\end{equation}
\end{lemma}

\begin{remark}\label{rem:eot-epsilon-zero-existence}
We note that such \(\varepsilon_0\) always exists. For any \(\widetilde\pi\in\Pi(q_0,q_1)\) with \(0<\operatorname{KL}(\widetilde\pi\|\pi_{\emph{ind}})<\infty\) and \(\mathbb E_{\widetilde\pi}[\|X_1-X_0\|_2^2]-W_2^2(q_0,q_1)<\mathcal L_{\pi_{\emph{ind}}}(\theta_{\pi_{\emph{ind}}}^\star)\), \cref{eq:eot-strict-population-loss-chain} holds whenever
\begin{equation}
\label{eq:eot-explicit-competitor-cutoff}
0<\varepsilon<\frac{\mathcal L_{\pi_{\emph{ind}}}(\theta_{\pi_{\emph{ind}}}^\star)+W_2^2(q_0,q_1)-\mathbb E_{\widetilde\pi}[\|X_1-X_0\|_2^2]}{2d\operatorname{KL}(\widetilde\pi\|\pi_{\emph{ind}})}.
\end{equation}
A coupling \(\widetilde\pi\) satisfying these conditions always exists; we defer its proof to \emph{Appendix~\ref{app:proof-eot-epsilon-existence}}.
\end{remark}

Finally, we transform Lemma~\ref{lem:eot-strict-small-epsilon} into the stochastic-gradient variance reduction.

\begin{theorem}
\label{thm:eot-gradient-variance}
Suppose the conditions of \emph{Lemma~\ref{lem:eot-strict-small-epsilon}} hold. Then, every \(0<\varepsilon\leq\varepsilon_0\) satisfies that
\begin{equation}\label{eq:eot-full-gradient-bound}
0<\sigma_{\emph{grad},\pi_\varepsilon}^2(\theta_{\pi_\varepsilon}^\star)\leq4\{\mathbb E_{\pi_\varepsilon}[\|X_1-X_0\|_2^2]-W_2^2(q_0,q_1)\}\operatorname*{ess\,sup}\|J_{\theta_{\pi_\varepsilon}^\star}(t,X_t^{\pi_\varepsilon})\|_{\mathrm{op}}^2.
\end{equation}
Moreover, under \emph{Definition~\ref{rem:bounded-jacobian}} over all sufficiently small \(\varepsilon\), there exists \(\varepsilon_1\in(0,\varepsilon_0]\) such that every \(0<\varepsilon\leq\varepsilon_1\) satisfies
\begin{equation}
\label{eq:eot-full-gradient-strict}
\sigma_{\emph{grad},\pi^\star}^2(\theta_{\pi^\star}^\star)<\sigma_{\emph{grad},\pi_\varepsilon}^2(\theta_{\pi_\varepsilon}^\star)<\sigma_{\emph{grad},\pi_{\emph{ind}}}^2(\theta_{\pi_{\emph{ind}}}^\star).
\end{equation}
\end{theorem}

\subsubsection{Gradient-variance reduction: larger minibatch}\label{sec:grad-var-scale}

We now move to characterize how the gradient variance changes when minibatch size is scaled. Similar to the previous proof structure, we first derive the result upon standard unregularized OT and then extend to the results to entropic OT. Our math idea follows from by connecting the cost residue in \cref{eq:large-k-cost-gap} to the velocity uncertainty and then bridging this uncertainty to the gradient variance.
\vspace{-0.5em}
\paragraph{Unregularized OT.} For each minibatch, we independently draw \(k\) source endpoints from \(q_0\) and \(k\) target endpoints from \(q_1\) and denote the distribution of this selected pair by \(\pi_k\), with \(\pi_1=\pi_{\text{ind}}\). Since \(\pi_k\) is a coupling between \(q_0\) and \(q_1\), its quadratic transport cost cannot be smaller than the cost of optimal coupling $\pi^*$, i.e., \(W_2^2(q_0,q_1)\). Therefore, we can represent this gap as the excess cost residue $\Delta_k$:
\begin{equation}\label{eq:large-k-cost-gap}
\Delta_k\coloneqq\mathbb E_{\pi_k}[\|X_1-X_0\|_2^2]-W_2^2(q_0,q_1).
\end{equation}
Clearly, \(\Delta_k\geq0\), and we apply this cost residual in an interesting way by directly connecting it to the velocity-target uncertainty that drives stochastic gradient variance:

\begin{proposition}
\label{prop:large-k-intrinsic-consistency}
Given \(q_0,q_1\in\mathcal P_2(\mathbb R^d)\), then without any assumptions, for every \(k\geq1\), we have
\begin{equation}\label{eq:large-k-intrinsic-envelope}
\mathbb E_{\pi_k}\!\left[\Tr\{\Cov_{\pi_k}(U^{\pi_k}\mid t,X_t^{\pi_k})\}\right]\leq\Delta_k\xrightarrow{k\to\infty}0.
\end{equation}
\end{proposition}
\vspace{-0.5em}
Indeed, Proposition~\ref{prop:large-k-intrinsic-consistency} plays an important role by bridging the gap between the velocity-target uncertainty and OT minibatch, which allows us to further connect it to stochastic gradient variance:

\begin{theorem}
\label{thm:large-k-eventual-gradient-reduction}
Suppose that \(q_0,q_1\in\mathcal P_2(\mathbb R^d)\) and \emph{Assumptions~\ref{ass:real}} and~\emph{\ref{ass:opt}} hold for every \(\pi_k\). Then,
\begin{equation}\label{eq:large-k-full-convergence}
0\leq\sigma_{\emph{grad},\pi_k}^2(\theta_{\pi_k}^\star)\leq4\Delta_k\operatorname*{ess\,sup}\|J_{\theta_{\pi_k}^\star}(t,X_t^{\pi_k})\|_{\mathrm{op}}^2.
\end{equation}
Given the setup in \emph{Definition~\ref{rem:bounded-jacobian}}, the right-hand side of \cref{eq:large-k-full-convergence} converges to zero, so for any fixed \(k_0\) with \(\sigma_{\emph{grad},\pi_{k_0}}^2(\theta_{\pi_{k_0}}^\star)>0\), \cref{eq:large-k-full-eventual-reduction} strictly holds for all sufficiently large \(k\):
\begin{equation}\label{eq:large-k-full-eventual-reduction}
\sigma_{\emph{grad},\pi_k}^2(\theta_{\pi_k}^\star)<\sigma_{\emph{grad},\pi_{k_0}}^2(\theta_{\pi_{k_0}}^\star).
\end{equation}
\end{theorem}

\vspace{-1em}
\paragraph{Entropic OT.}
For each \(k\geq1\) and \(\varepsilon>0\), we denote \(\pi_{k,\varepsilon}\) as the distribution of the endpoint pair generated by the entropic minibatch procedure in \S\ref{sec:method} and \(\widehat\pi_{k,\varepsilon}\) as the random empirical plan within minibatches. Similarly, \(\pi_{1,\varepsilon}\!=\!\pi_{\text{ind}}\). For simplicity, we abbreviate the average target uncertainty in Proposition~\ref{prop:large-k-intrinsic-consistency} by \( \mathsf R(\pi)\!=\!\mathbb E_\pi[\Tr\{\Cov_\pi(U^\pi\mid t,X_t^\pi)\}] \). Shown in Proposition~\ref{prop:entropic-large-k-fixed}, we first characterize the effect of increased \(k\) using the similar cost residue as in \cref{eq:large-k-cost-gap}:

\begin{proposition}
\label{prop:entropic-large-k-fixed}
Suppose \(q_0,q_1\in\mathcal P_2(\mathbb R^d)\). For every fixed \(\varepsilon>0\), \(\pi_{k,\varepsilon}\to\pi_\varepsilon\) in \(W_2\), and moreover,
\begin{equation}\label{eq:entropic-large-k-fixed}
\limsup_{k\to\infty}\mathsf R(\pi_{k,\varepsilon})\leq\mathsf R(\pi_\varepsilon)\leq\mathbb E_{\pi_\varepsilon}[\|X_1-X_0\|_2^2]-W_2^2(q_0,q_1).
\end{equation}
Note that \(\pi_{k,\varepsilon}\) and \(\pi_{\emph{ind}}\) are mutually absolutely continuous for every finite \(k\), so \(\mathsf R(\pi_{k,\varepsilon})=0\) if and only if \(\mathsf R(\pi_{\emph{ind}})=0\).
\end{proposition}
\vspace{-0.5em}
Increasing \(k\) at a fixed \(\varepsilon\) indeed approaches the population entropic coupling, but this convergence alone cannot give the complete reduction result, which also depends on the coupling-specific model Jacobian. So, we next establish the gradient variance reduction under larger minibatch size:

\begin{theorem}
\label{thm:entropic-large-k-fixed-gradient}
Fix a finite baseline \(k_0\geq1\) and consider that \(q_0,q_1\in\mathcal P_2(\mathbb R^d)\), \(\mathsf R(\pi_{\emph{ind}})>0\), and \emph{Assumptions~\ref{ass:real}} and~\emph{\ref{ass:opt}} hold for every \(\pi_{k,\varepsilon}\) under comparison. Then, for every \(k\geq1\) and \(\varepsilon>0\), we have
\begin{equation}\label{eq:entropic-large-k-fixed-bound}
4\mathsf R(\pi_{k,\varepsilon})\leq\sigma_{\emph{grad},\pi_{k,\varepsilon}}^2(\theta_{\pi_{k,\varepsilon}}^\star)\leq4\mathsf R(\pi_{k,\varepsilon})\operatorname*{ess\,sup}\|J_{\theta_{\pi_{k,\varepsilon}}^\star}(t,X_t^{\pi_{k,\varepsilon}})\|_{\mathrm{op}}^2.
\end{equation}
Under \emph{Definition~\ref{rem:bounded-jacobian}} over all sufficiently large \(k\) and sufficiently small \(\varepsilon\), there exists \(\bar\varepsilon>0\) such that, for every fixed \(0<\varepsilon\leq\bar\varepsilon\), there is a finite batch threshold \(K>k_0\) for which every \(k\geq K\) satisfies
\begin{equation}\label{eq:entropic-large-k-fixed-reduction}
\sigma_{\emph{grad},\pi_{k,\varepsilon}}^2(\theta_{\pi_{k,\varepsilon}}^\star)<\sigma_{\emph{grad},\pi_{k_0,\varepsilon}}^2(\theta_{\pi_{k_0,\varepsilon}}^\star).
\end{equation}
\end{theorem}

\subsection{On Integration-Error Reduction}\label{sec:int-err}

Beyond optimization, OT couplings induce straighter probability paths, but how this geometry translates into numerical benefits remains unclear. We study integration error at the distribution level by comparing the output of Euler discretization with the target distribution.

\begin{definition}[One-step distributional Euler error]\label{def:one-step-distributional-euler-error}
For a coupling $\pi\in\Pi(q_0,q_1)$, one Euler step from $t=0$ produces $X_0+\mathbb E_\pi[X_1-X_0\mid X_0]=\mathbb E_\pi[X_1\mid X_0]$. Its distributional integration error is
\begin{equation} \label{eq:one-step-distributional-euler-error}
\mathrm{err}(\pi)\coloneqq W_2^2\!\left((\mathbb E_\pi[X_1\mid X_0=\cdot])_\#q_0,q_1 \right),
\end{equation}
i.e., the Wasserstein distance between the one-step output distribution and the target distribution $q_1$.
\end{definition}
\vspace{-0.5em}
Shown in Proposition~\ref{prop:int-err-red}, we first establish the relationship of distributional integration error across optimal, entropic, and independent couplings:
\begin{proposition} \label{prop:int-err-red}
Suppose $q_0,q_1\in\mathcal P_2(\mathbb R^d)$, $q_0\ll\mathcal L^d$, and $q_1$ is not a Dirac mass. Let $\pi^\star=(\mathrm{id},T)_\#q_0$ be the population quadratic-OT coupling, let $\pi_k$ be the unregularized OT plan at batch size $k$, with $\pi_1=\pi_{\emph{ind}}$, and for every fixed $\varepsilon>0$, let $\pi_\varepsilon$ be the population entropic coupling. Then,
\[
\mathrm{err}(\pi^{*})<\mathrm{err}(\pi_{\varepsilon})<\mathrm{err}(\pi_{\emph{ind}}).
\]
\end{proposition}

\vspace{-0.5em}
Finally, based on Proposition~\ref{prop:int-err-red}, we establish the error reduction over larger $k$ by connecting the expected empirical EOT value in \cref{eq:main-expected-eot-value} to the integration error.
\begin{equation}
\label{eq:main-expected-eot-value}
 S_{k,\varepsilon}\coloneqq\mathbb E\!\left[\left\langle\mathbf C,\widehat\pi_{k,\varepsilon}\right\rangle+\varepsilon\operatorname{KL}\!\left(\widehat\pi_{k,\varepsilon}\middle\|\mathbf1\mathbf1^\top/k^2\right)\right].
\end{equation}

\begin{theorem}\label{thm:main-entropic-integration}
Under the conditions of \emph{Proposition~\ref{prop:int-err-red}}, for every fixed $\varepsilon>0$ and $k\geq2$, we have
\begin{equation}
\label{eq:main-entropic-finite-k-chain}
\mathrm{err}(\pi_{k,\varepsilon})
\leq \mathbb E[\|X_1-\mathbb E[X_1]\|_2^2]-\frac{d^2{(S_{1,\varepsilon}-S_{k,\varepsilon})}^2}{\mathbb E[\|X_0-\mathbb E[X_0]\|_2^2]}
<\mathrm{err(\pi_{\emph{ind}})}.
\end{equation}
The upper bound in \cref{eq:main-entropic-finite-k-chain} is also non-increasing in $k$, and
\(
\mathrm{err}(\pi_{k,\varepsilon})\xrightarrow{k\to\infty}\mathrm{err}(\pi_{\varepsilon})
\). Let $\varepsilon_k>0$ be any sequence with $\varepsilon_k\to0$, for every fixed finite baseline $k_0$, there exists $k_1>k_0$ such that all $k>k_1$ satisfies
\begin{equation}
\label{eq:main-entropic-matched-baseline}
\mathrm{err}(\pi_{k,\varepsilon_k}) < \mathrm{err}(\pi_{k_0,\varepsilon_k}).
\end{equation}
\end{theorem}
\vspace{-0.5em}
Figure~\ref{f3} illustrates the empirical benefit of \cref{eq:main-entropic-matched-baseline} through a qualitative comparison of generated samples at the same NFE; we defer further explanations to \S\ref{sec:exp}.
\section{Experiments} \label{sec:exp}

\begin{figure}[!t]
    \centering
\includegraphics[width=\linewidth]{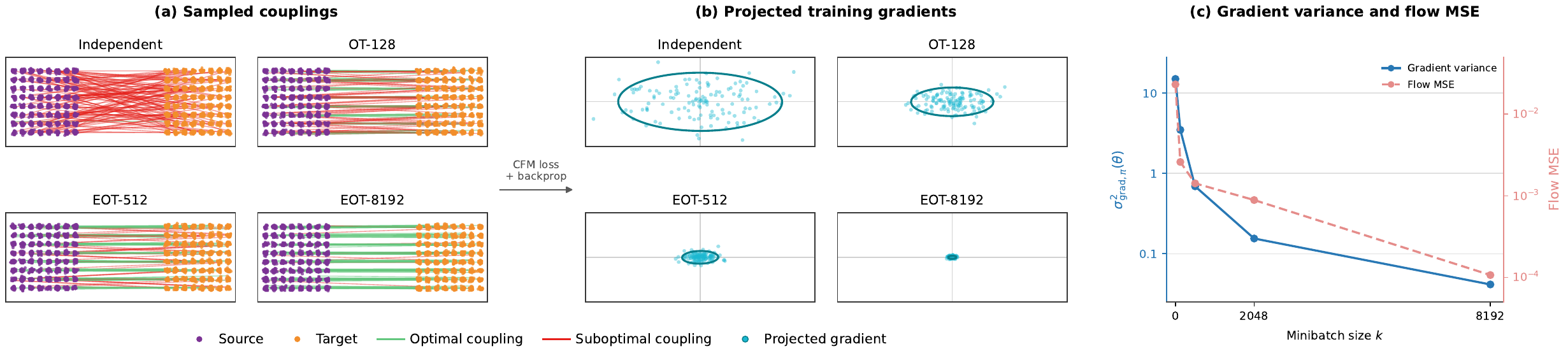}
\vspace{-1.5em}
    \caption{Results on a synthetic multimodal task over independent random coupling, OT-CFM using the minibatch size of $128$, and EOT-CFM with contexts up to $8192$. (a) Source-target pairings, with green and red lines indicating optimal and suboptimal couplings, respectively. Larger minibatches produce more optimal assignments for subsequent flow training. (b) Stochastic gradients evaluated at a trained flow model with regression batch size $128$ and projected onto common two-dimensional axes; ellipses summarize their spread. (c) Full gradient variance (left axis) and final flow-generation MSE (right axis) averaged over three training seeds, scaling along the OT minibatch size, where $k=0$ denotes independent coupling. Scaling the minibatch size consistently reduces both gradient variance and MSE error, yielding more stable training and better flow generation.}
    \label{f1}
    \vspace{-1.5em}
\end{figure}

We study three classes of experiments to show the empirical benefits of scaling minibatch size: (i) first, we consider classical 2D synthetic datasets, where the effects of minibatch scaling can be clearly visualized and interpreted; (ii) we then work on single-cell interpolation flows; apart from the traditional medium-dimensional dataset used in OT-CFM, we include a large-scale, high-dimensional dataset to demonstrate the benefits of scaled minibatches; (iii) finally, we evaluate image flow training on CIFAR-100 \citep{krizhevsky2009learning} and ImageNet-64$\times$64 \citep{deng2009imagenet}.

\vspace{-0.5em}
\paragraph{2D-couplings.}  We consider two synthetic problems: a symmetric problem with $64$ Gaussian modes arranged on an $8\times8$ grid, and a double-ring problem with $8$ modes on each ring. In both problems, the source and target distributions are separated by a horizontal translation, and the ground-truth coupling matches points with the same mode label. We compare independent coupling, OT-CFM with a $128\times128$ minibatch, and EOT with context sizes $n\in\{512,2048,8192\}$, using squared Euclidean cost and training a three-layer MLP for $6{,}000$ updates with batch size $128$ with AdamW.

Results are shown in Figure~\ref{f1}. Specifically, larger OT contexts produce increasingly coherent source-target pairings, as shown by the reduction in red suboptimal crossing connections and the increase in green optimal coupling connections. This improvement directly translates into lower SGD variance, with  tighter projected-gradient distributions that shows scaling the OT context provides cleaner pairings and substantially better training stability. More 2D examples are provided in Appendix~\ref{app:exp}.

\begin{table}[t]
  \centering
  \small
  \setlength{\tabcolsep}{5pt}
  \caption{Results on EB and PBMC (mean $\pm$ std over 5 seeds), with performance metrics ($W_1$, Energy) and gradient variance. Larger contexts improve prediction quality and reduce gradient variance.}
  \resizebox{\columnwidth}{!}{%
  \begin{tabular}{
      lcc
      @{\hspace{12pt}}
      lccc
  }
  \toprule
  \multicolumn{3}{c}{\textbf{EB}}
  &
  \multicolumn{4}{c}{\textbf{PBMC} }
  \\
  \cmidrule(lr){1-3}
  \cmidrule(lr){4-7}
  \textbf{Method}
  & Exact $W_1$ $(\downarrow)$
  & $\sigma_{\text{grad},\pi}^{2}(\theta)$ $(\downarrow)$
  &
  \textbf{Method}
  & Exact $W_1$ $(\downarrow)$
  & CellFlow Energy $(\downarrow)$
  & $\sigma_{\text{grad},\pi}^{2}(\theta)$ $(\downarrow)$
  \\
  \midrule
  OT-CFM-128
  & $0.756 \pm 0.027$
  & $0.247 \pm 0.047$
  & 
  OT-CFM-128
  & $2.612 \pm 0.021$
  & $0.693 \pm 0.102$
  & $1.183 \pm 0.343$
  \\
    \midrule
  EOT-CFM-512
  & $0.746 \pm 0.021$
  & $0.101 \pm 0.017$
  &
  EOT-CFM-1024
  & $2.531 \pm 0.017$
  & $0.435 \pm 0.115$
  & $0.697 \pm 0.096$
  \\

  EOT-CFM-1024
  & $0.719 \pm 0.017$
  & $0.080 \pm 0.012$
  &
  EOT-CFM-4096
  & $2.505 \pm 0.028$
  & $0.385 \pm 0.161$
  & $0.515 \pm 0.035$
  \\
  
  EOT-CFM-1536
  & $0.714 \pm 0.016$
  & $0.071 \pm 0.011$
  &
  EOT-CFM-8192
  & $2.432 \pm 0.015$
  & $0.329 \pm 0.046$
  & $0.384 \pm 0.037$
  \\
  \bottomrule
  \end{tabular}%
  }
\label{tab:singlecell}
\end{table}

\vspace{-0.5em}
\paragraph{High-dimensional single-cell interpolation.} We evaluate on both the embryoid body (EB) benchmark \citep{moon2019visualizing} used in OT-CFM \citep{tong2023improving} and a larger, higher-dimensional PBMC cytokine-response dataset \citep{oesinghaus2025single}, processed by CellFlow \citep{klein2025cellflow}. For EB, it represents $16{,}819$ cells from five time points (the smallest point has 1900 cells) using five whitened principal components and independently split each time point into $80\%/10\%/10\%$ training, validation, and test sets. We leave out each intermediate time point $t\in\{1,2,3\}$, train a three-layer MLP of width $64$ on the remaining transitions for at most $1{,}000$ epochs.

For PBMC, we use $90{,}452$ PBS control cells and $20{,}471$ IFN-$\beta$-stimulated cells from Donor 1, represent each cell using $100$ whitened principal components, and apply the same data split. We train a three-layer MLP of width $256$ for $100{,}000$ steps. Both experiments use AdamW with learning rate $10^{-3}$, weight decay $10^{-5}$, and a regression batch size of $128$, and generate cells using $100$ Euler steps. For EB, we report exact $W_1$ averaged over three held-out time points and five seeds. For PBMC, we compare $2{,}048$ generated and held-out target cells and report exact assignment $W_1$, CellFlow energy, and gradient variance, averaged over five seeds. Results are presented in Table~\ref{tab:singlecell}. Since the smallest time point from EB contains around 1900 cells, the maximum minibatch cannot be scaled as PBMC. For both dataset, larger minibatch leads to both better empirical performance and more stable training, where the improvement consistently gets better with larger minibatch size. With the result deferred to Appendix~\ref{app:exp}, we also use the high dimensional PBMC dataset to study the ablation over $\varepsilon$, moderate thousands-level context already gives a clear effect of scaling and remains robust to the entropic regularization intensity.

\begin{table}[t]
\centering
\caption{
FIDs across different minibatch size and NFEs on CIFAR and ImageNet dataset.
}
\label{tab:fid}
\resizebox{\textwidth}{!}{
\begin{tabular}{l|cccccc|l|cccccc}
\toprule
& \multicolumn{6}{c|}{\textbf{CIFAR-$100$}}
& \multicolumn{7}{c}{\textbf{ImageNet-$64\!\times\!64$}} \\
\cmidrule(lr){2-7}
\cmidrule(lr){8-14}
Minibatch context
& \texttt{NFE} 25
& \texttt{NFE} 50
& \texttt{NFE} 100
& \texttt{NFE} 1000
& \multicolumn{2}{c|}{Adaptive}
& Minibatch context
& \texttt{NFE} 25
& \texttt{NFE} 50
& \texttt{NFE} 100
& \texttt{NFE} 1000
& \multicolumn{2}{c}{Adaptive} \\
\cmidrule(lr){6-7}
\cmidrule(lr){13-14}
&
&
&
&
& \texttt{NFE} & \texttt{FID}
&
&
&
&
&
& \texttt{NFE} & \texttt{FID} \\
\midrule
OT-CFM-$128$
& 12.85
& 10.49
& 9.20
& 8.00
& 142.40 & 7.89
&
OT-CFM-$128$
& 16.95
& 14.88
& 13.98
& 13.33
& 183.91 & 13.26
\\
\midrule
EOT-CFM-$2048$
& 12.77
& 10.28
& 8.77
& 7.79
& 147.16 & 7.71
&
EOT-CFM-$6144$
& 16.82
& 14.69
& 13.80
& 13.24
& 188.17 & 13.20
\\

EOT-CFM-$4096$
& 12.05
& 9.80
& 8.59
& 7.46
& 147.80 & 7.36
&
EOT-CFM-$8192$
& 16.44
& 14.57
& 13.73
& 13.12
& 187.32 & 13.01
\\

EOT-CFM-$8192$
& 11.85
& 9.63
& 8.44
& 7.34
& 146.86 & 7.25
&
EOT-CFM-$10240$
& 16.36
& 14.47
& 13.59
& 13.05
& 187.19 & 12.94
\\
\bottomrule
\end{tabular}
}
\end{table}

\begin{figure*}[t]
    \centering
\includegraphics[width=\linewidth]{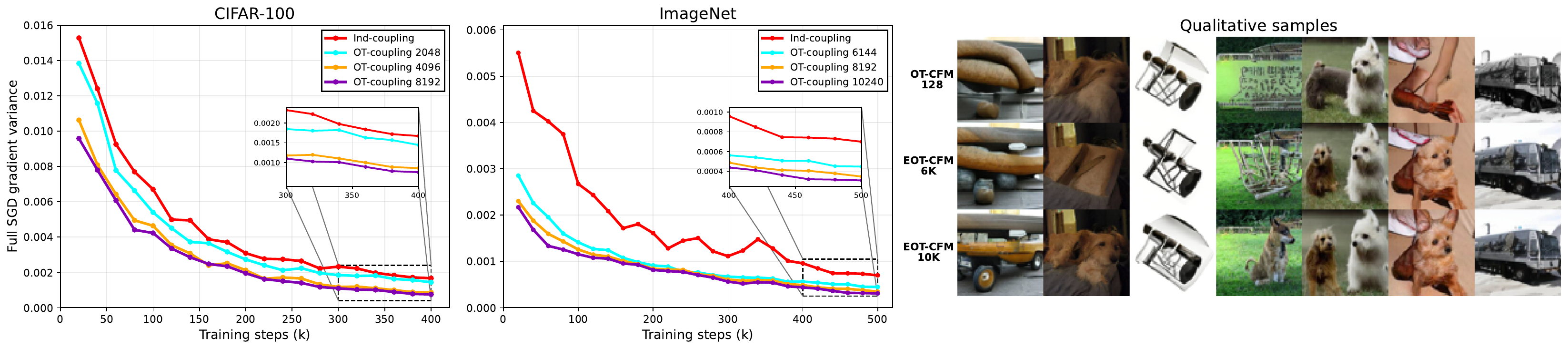}
\vspace{-1.5em}
    \caption{Full SGD gradient variance and qualitative samples. Variance during training on CIFAR-100 (\textbf{left}) and ImageNet (\textbf{middle}) under different minibatch size, with insets highlighting late-stage behavior. Samples from OT-CFM and EOT-CFM from 128 to 10K, generated at NFE 100. (\textbf{right}).}
    \label{f2}
    \vspace{+0.75em}
\includegraphics[width=\linewidth]{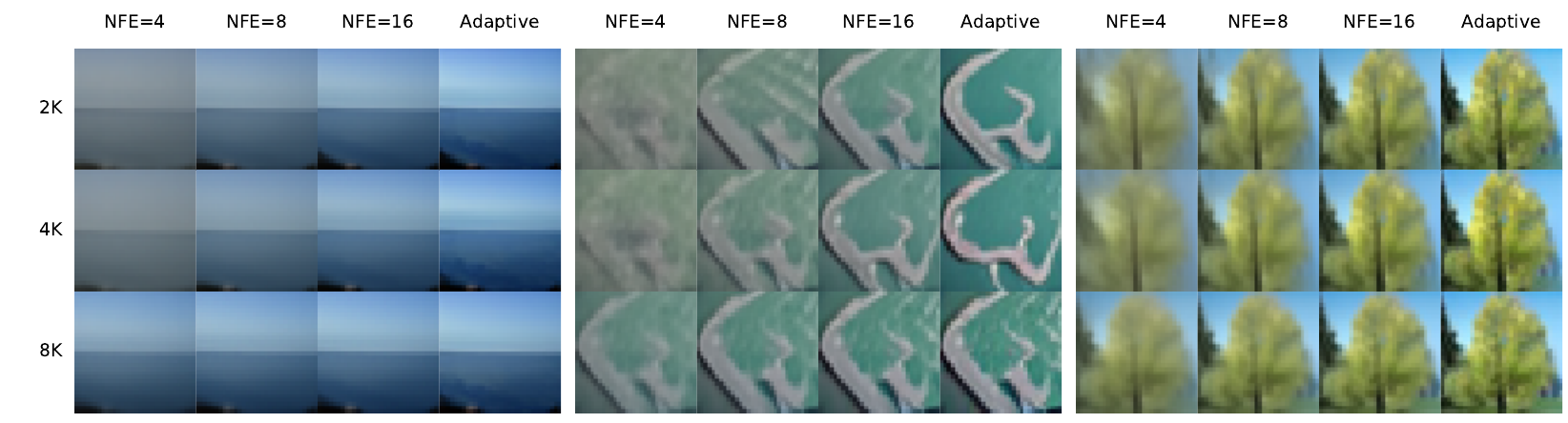}
\vspace{-1.5em}
    \caption{Generated images from initial noise. Rows correspond to minibatch sizes of 2K, 4K, and 8K, while columns compare Euler sampling at NFE 4, 8, and 16 with adaptive Dopri5. Smaller coupling minibatches tend to produce darker, less vivid images, whereas larger minibatches yield brighter samples with more vibrant colors, particularly at low NFEs.}
    \label{f3}
    \vspace{-1.5em}
\end{figure*}

\vspace{-1em}
\paragraph{Image generation.} We evaluate on CIFAR-100 and ImageNet-\(64\!\times\!64\). For CIFAR, all methods follow the same U-Net architecture and optimization hyperparameters from OT-CFM's CIFAR training script \citep{tong2023improving}. We evaluate EMA checkpoints using fixed-step and adaptive Dopri5 integration, and FIDs are computed with CleanFID with \(50{,}000\) generated samples.  

For ImageNet training, we follow the unconditional ADM-style U-Net configuration of \citet{pooladian2023multisample}. We train for \(500{,}000\) updates with a global batch size of \(1{,}024\) across 8$\times$NVIDIA H100 GPUs and a constant learning rate of \(10^{-4}\). Similar to CIFAR-100, all coupling methods share the same architecture and training configuration, and we generate \(50{,}000\) samples from the EMA checkpoint and report FID against the canonical ADM ImageNet-64 reference statistics. Evaluation results are shown in Table~\ref{tab:fid}, and we further present in-training gradient variance and qualitative image generation results in Figure~\ref{f2}. Experiments on both single-cell and image flow matching show that larger coupling minibatches generally reduce gradient variance and improve performance, consistent with our entropic analysis in \S\ref{sec:grad-var-scale}. Figure~\ref{f3} further shows that the gains in generation quality are particularly better at low NFEs. Together, these findings suggest that scaling the entropic coupling context benefits both optimization and sampling, combining lower gradient noise during training with better generation quality under limited integration budgets.

Thanks to its bandwidth-efficient design and compatibility with PyTorch/Triton over HBM GPUs, FlashSinkhorn integrates naturally into efficient FM training. Increasing the minibatch size incurs only a modest increase in peak GPU memory. On ImageNet, the standard OT-CFM baseline reaches a peak memory usage of 39.96 GB, while FlashSinkhorn with a coupling size of 10K uses only 41.08 GB. For runtime, even at the largest context 10K, full OT coupling and pair sampling wall time account for less than $10\%$ of the total time per gradient optimization step. The measured coupling latency is approximately 0.03 seconds at the 4K context and around 0.08 seconds at the 8K context. In addition to the benefit brought by entropic regularization, FlashAttention-style engineering also offers efficient GPU kernels, the native PyTorch integration removes the need for PyTorch-JAX tensor interchange and coordination between separate execution and memory-management runtimes, keeping coupling computation and pair sampling within the flow-matching optimization loop.

\section{Conclusion}

In this work, we study online coupling as a fundamental scaling axis for flow matching and fill the gap by establishing a strict reduction in full stochastic-gradient variance from independent pairing to population OT, characterize how larger online minibatch couplings recover this advantage, and extend the analysis to entropically regularized OT together with parallel guarantees for integration error. These results motivate a scalable approximation based on large online entropic couplings. Across both cell flow and image flow experiments, scaling the online coupling consistently improves training stability and generation quality. In the end, our results provide both a principled theoretical framework for rigorously analyzing the concrete effect brought by the more informative couplings and simple yet practical foundation for scaling flow matching.

\newpage
\bibliography{iclr2026_conference}

@inproceedings{rezende2015variational,
  title={Variational inference with normalizing flows},
  author={Rezende, Danilo and Mohamed, Shakir},
  booktitle={ICML},
  year={2015}
}

@inproceedings{liu2022flow,
  title={Flow straight and fast: Learning to generate and transfer data with rectified flow},
  author={Liu, Xingchao and Gong, Chengyue and Liu, Qiang},
  booktitle={ICLR},
  year={2023}
}

@inproceedings{chen2018neural,
  title={Neural ordinary differential equations},
  author={Chen, Ricky TQ and Rubanova, Yulia and Bettencourt, Jesse and Duvenaud, David K},
  booktitle={NeurIPS},
  year={2018}
}

@inproceedings{lipman2022flow,
  title={Flow matching for generative modeling},
  author={Lipman, Yaron and Chen, Ricky TQ and Ben-Hamu, Heli and Nickel, Maximilian and Le, Matt},
  booktitle={ICLR},
  year={2023}
}

@inproceedings{tong2023improving,
  title={Improving and generalizing flow-based generative models with minibatch optimal transport},
  author={Tong, Alexander and Fatras, Kilian and Malkin, Nikolay and Huguet, Guillaume and Zhang, Yanlei and Rector-Brooks, Jarrid and Wolf, Guy and Bengio, Yoshua},
  booktitle={TMLR},
  year={2024}
}

@inproceedings{pooladian2023multisample,
  title={Multisample flow matching: Straightening flows with minibatch couplings},
  author={Pooladian, Aram-Alexandre and Ben-Hamu, Heli and Domingo-Enrich, Carles and Amos, Brandon and Lipman, Yaron and Chen, Ricky TQ},
   booktitle={ICML},
  year={2023}
}

@inproceedings{zhang2026on,
    title={On Fitting Flow Models with Large Sinkhorn Couplings},
    author={Stephen Y. Zhang and Alireza Mousavi-Hosseini and Michal Klein and Marco Cuturi},
    booktitle={TMLR},
    year={2026}
}

@article{boite2026expected,
  title={Expected Batch Optimal Transport Plans and Consequences for Flow Matching},
  author={Bo{\"\i}t{\'e}, Samuel and Delon, Julie and Nadjahi, Kimia},
  journal={arXiv preprint arXiv:2605.12174},
  year={2026}
}

@inproceedings{cuturi2013sinkhorn,
  title={Sinkhorn distances: Lightspeed computation of optimal transport},
  author={Cuturi, Marco},
  booktitle={NeurIPS},
  year={2013}
}

@inproceedings{altschuler2017near,
  title={Near-linear time approximation algorithms for optimal transport via Sinkhorn iteration},
  author={Altschuler, Jason and Niles-Weed, Jonathan and Rigollet, Philippe},
  booktitle={NeurIPS},
  year={2017}
}

@inproceedings{ye2026flashsinkhorn,
  title={FlashSinkhorn: IO-Aware Entropic Optimal Transport on GPU},
  author={Ye, Felix X-F and Li, Xingjie and Yu, An and Chang, Ming-Ching and Chu, Linsong and Wertheimer, Davis},
  booktitle={ICML},
  year={2026}
}

@inproceedings{flamary2021pot,
  author={Flamary, R{\'e}mi and Courty, Nicolas and Gramfort, Alexandre and Alaya, Mokhtar Z and Boisbunon, Aur{\'e}lie and Chambon, Stanislas and Chapel, Laetitia and Corenflos, Adrien and Fatras, Kilian and Fournier, Nemo and others},
  title   = {P{OT}: Python Optimal Transport},
  booktitle={JMLR},
  year    = {2021}
}

@inproceedings{malamut2025convergence,
  title={Convergence rates of the regularized optimal transport: Disentangling suboptimality and entropy},
  author={Malamut, Hugo and Sylvestre, Maxime},
  booktitle={SIAM Journal on Mathematical Analysis},
  year={2025}
}

@inproceedings{mousavi2026flow,
  title={Flow matching with semidiscrete couplings},
  author={Mousavi-Hosseini, Alireza and Zhang, Stephen and Klein, Michal and Cuturi, Marco},
  booktitle={ICLR},
  year={2026}
}

@inproceedings{kong2026alignflow,
  title={AlignFlow: Improving flow-based generative models with semi-discrete optimal transport},
  author={Kong, Lingkai and Tao, Molei and Liu, Yang and Wang, Bryan and Fu, Jinmiao and Wang, Chien-Chih and Liu, Huidong},
  booktitle={ICLR},
  year={2026}
}

@article{calvo2025weighted,
  title={Weighted conditional flow matching},
  author={Calvo-Ordonez, Sergio and Meunier, Matthieu and Cartea, Alvaro and Reisinger, Christoph and Gal, Yarin and Hernandez-Lobato, Jose Miguel},
  journal={arXiv preprint arXiv:2507.22270},
  year={2025}
}

@inproceedings{reu2026gradient,
  title={Gradient variance reveals failure modes in flow-based generative models},
  author={Reu, Teodora and Dromigny, Sixtine and Bronstein, Michael and Vargas, Francisco},
  booktitle={NeurIPS},
  year={2025}
}

@article{tao2026role,
  title={On the Role of Strain and Vorticity in Numerical Integration Error for Flow Matching},
  author={Tao, Chenxi and Choi, Seung-Kyum},
  journal={arXiv preprint arXiv:2605.06680},
  year={2026}
}

@inproceedings{islam2025longitudinal,
  title={Longitudinal Flow Matching for Trajectory Modeling},
  author={Islam, Mohammad Mohaiminul and Kuipers, Thijs P and Vadgama, Sharvaree and de Vente, Coen and Khan, Afsana and S{\'a}nchez, Clara I and Bekkers, Erik J},
  booktitle={AISTATS},
  year={2026}
}

@article{khan2026flow,
  title={Flow Matching for Offline Reinforcement Learning with Discrete Actions},
  author={Khan, Fairoz Nower and Nahim, Nabuat Zaman and Huang, Ruiquan and Yang, Haibo and Ju, Peizhong},
  journal={arXiv preprint arXiv:2602.06138},
  year={2026}
}

@inproceedings{baheri2026geometrygrounded,
    title={Geometry-Grounded Flow Matching on Compact Manifolds},
    author={Ali Baheri},
    booktitle={ICLR 2026 Workshop on Geometry-grounded Representation Learning and Generative Modeling},
    year={2026}
}

@inproceedings{brenier1991polar,
author = {Brenier, Yann},
title = {Polar factorization and monotone rearrangement of vector-valued functions},
booktitle = {Communications on Pure and Applied Mathematics},
year = {1991}
}

@inproceedings{oymak2019overparameterized,
  title={Overparameterized nonlinear learning: Gradient descent takes the shortest path?},
  author={Oymak, Samet and Soltanolkotabi, Mahdi},
  booktitle={ICML},
  year={2019}
}

@article{zeng2025flow,
  title={Flow Matching in the Low-Noise Regime: Pathologies and a Contrastive Remedy},
  author={Zeng, Weili and Yan, Yichao},
  journal={arXiv preprint arXiv:2509.20952},
  year={2025}
}

@inproceedings{amos2017input,
  title={Input convex neural networks},
  author={Amos, Brandon and Xu, Lei and Kolter, J Zico},
  booktitle={ICML},
  year={2017}
}

@inproceedings{makkuva2020optimal,
  title={Optimal transport mapping via input convex neural networks},
  author={Makkuva, Ashok and Taghvaei, Amirhossein and Oh, Sewoong and Lee, Jason},
  booktitle={ICML},
  year={2020}
}

@inproceedings{fan2020scalable,
  title={Scalable computations of wasserstein barycenter via input convex neural networks},
  author={Fan, Jiaojiao and Taghvaei, Amirhossein and Chen, Yongxin},
  booktitle={ICML},
  year={2021}
}

@inproceedings{chen2024sicnn,
  title={{SICNN}: Sparsity-induced input convex neural network for optimal transport},
  author={Chen, Peter and Xie, Yue and Zhang, Qingpeng},
  booktitle={Optimization for Machine Learning},
  year={2024}
}

@inproceedings{korotin2022neural,
  title={Neural optimal transport},
  author={Korotin, Alexander and Selikhanovych, Daniil and Burnaev, Evgeny},
  booktitle={ICLR},
  year={2023}
}

@inproceedings{asadulaev2024neural,
  title={Neural optimal transport with general cost functionals},
  author={Asadulaev, Arip and Korotin, Aleksandr and Egiazarian, Vage and Mokrov, Petr and Burnaev, Evgeny},
  booktitle={ICLR},
  year={2024}
}

@inproceedings{korotin2019wasserstein,
  title={Wasserstein-2 generative networks},
  author={Korotin, Alexander and Egiazarian, Vage and Asadulaev, Arip and Safin, Alexander and Burnaev, Evgeny},
  booktitle={ICLR},
  year={2021}
}

@inproceedings{korotin2021continuous,
  title={Continuous wasserstein-2 barycenter estimation without minimax optimization},
  author={Korotin, Alexander and Li, Lingxiao and Solomon, Justin and Burnaev, Evgeny},
  booktitle={ICLR},
  year={2021}
}

@inproceedings{lin2019efficient,
  title={On efficient optimal transport: An analysis of greedy and accelerated mirror descent algorithms},
  author={Lin, Tianyi and Ho, Nhat and Jordan, Michael},
  booktitle={ICML},
  year={2019}
}

@inproceedings{lin2022complexity,
  title={On the complexity of approximating multimarginal optimal transport},
  author={Lin, Tianyi and Ho, Nhat and Cuturi, Marco and Jordan, Michael I},
  booktitle={JMLR},
  year={2022}
}

@article{nutz2022entropic,
  title={Entropic optimal transport: Convergence of potentials},
  author={Nutz, Marcel and Wiesel, Johannes},
  journal={Probability Theory and Related Fields},
  year={2022}
}

@inproceedings{fatras2019learning,
  title={Learning with minibatch wasserstein: asymptotic and gradient properties},
  author={Fatras, Kilian and Zine, Younes and Flamary, R{\'e}mi and Gribonval, R{\'e}mi and Courty, Nicolas},
  booktitle={AISTATS},
  year={2020}
}

@article{chambolle2026limit,
  title={Limit Points of Reflow with Minibatch Optimal Transport},
  author={Chambolle, Antonin and Hertrich, Johannes},
  journal={arXiv preprint arXiv:2608.07042},
  year={2026}
}

@inproceedings{kornilov2024optimal,
  title={Optimal flow matching: Learning straight trajectories in just one step},
  author={Kornilov, Nikita and Mokrov, Petr and Gasnikov, Alexander and Korotin, Alexander},
  booktitle={NeurIPS},
  year={2024}
}

@inproceedings{weed2018explicit,
  title={An explicit analysis of the entropic penalty in linear programming},
  author={Weed, Jonathan},
  booktitle={COLT},
  year={2018}
}

@article{rigollet2025sample,
  title={On the sample complexity of entropic optimal transport},
  author={Rigollet, Philippe and Stromme, Austin J},
  journal={The Annals of Statistics},
  year={2025}
}

@article{mordant2024entropic,
  title={The entropic optimal (self-) transport problem: Limit distributions for decreasing regularization with application to score function estimation},
  author={Mordant, Gilles},
  journal={arXiv preprint arXiv:2412.12007},
  year={2024}
}

@article{gonzalez2026finite,
  title={Finite-sample bounds for regularized optimal transport},
  author={Gonz{\'a}lez-Sanz, Alberto and Nutz, Marcel and Stromme, Austin J},
  journal={arXiv preprint arXiv:2606.25947},
  year={2026}
}

@inproceedings{deng2009imagenet,
  title={ImageNet: A large-scale hierarchical image database}, 
  author={Deng, Jia and Dong, Wei and Socher, Richard and Li, Li-Jia and Kai Li and Li Fei-Fei},
  booktitle={CVPR},
  year={2009}
}

@article{krizhevsky2009learning,
  title={Learning multiple layers of features from tiny images},
  author={Krizhevsky, Alex},
  year={2009},
  journal={Tech Report},
  publisher={Toronto, ON, Canada}
}

@article{moon2019visualizing,
  title={Visualizing structure and transitions in high-dimensional biological data},
  author={Moon, Kevin R and Van Dijk, David and Wang, Zheng and Gigante, Scott and Burkhardt, Daniel B and Chen, William S and Yim, Kristina and Elzen, Antonia van den and Hirn, Matthew J and Coifman, Ronald R and others},
  journal={Nature biotechnology},
  year={2019}
}

@article{oesinghaus2025single,
  title={A single-cell cytokine dictionary of human peripheral blood},
  author={Oesinghaus, Lukas and Becker, S{\"o}ren and Vornholz, Larsen and Papalexi, Efthymia and Pangallo, Joey and Moinfar, Amir Ali and Liu, Jenni and La Fleur, Alyssa and Shulman, Maiia and Marrujo, Simone and others},
  journal={BioRXiv},
  year={2025}
}

@article{klein2025cellflow,
  title={CellFlow enables generative single-cell phenotype modeling with flow matching},
  author={Klein, Dominik and Fleck, Jonas Simon and Bobrovskiy, Daniil and Zimmermann, Lea and Becker, S{\"o}ren and Palma, Alessandro and Dony, Leander and Tejada-Lapuerta, Alejandro and Huguet, Guillaume and Lin, Hsiu-Chuan and others},
  journal={bioRxiv},
  year={2025}
}

@article{litman2025scaled,
  title={Scaled-dot-product attention as one-sided entropic optimal transport},
  author={Litman, Elon},
  journal={arXiv preprint arXiv:2508.08369},
  year={2025}
}

@inproceedings{litman2026you,
    title={You Need Better Attention Priors},
    author={Elon Litman and Gabe Guo},
    booktitle={ICML},
    year={2026}
}

@article{beiglbock2017complete,
  title={Complete duality for martingale optimal transport on the line},
  author={Beiglb{\"o}ck, Mathias and Nutz, Marcel and Touzi, Nizar},
  year={2017}
}

@article{benamou2000computational,
  title={A computational fluid mechanics solution to the Monge-Kantorovich mass transfer problem},
  author={Benamou, Jean-David and Brenier, Yann},
  journal={Numerische Mathematik},
  volume={84},
  number={3},
  pages={375--393},
  year={2000},
  publisher={Springer-Verlag Berlin/Heidelberg}
}

@article{eckstein2022quantitative,
  title={Quantitative stability of regularized optimal transport and convergence of Sinkhorn's algorithm},
  author={Eckstein, Stephan and Nutz, Marcel},
  journal={SIAM Journal on Mathematical Analysis},
  volume={54},
  number={6},
  pages={5922--5948},
  year={2022},
  publisher={SIAM}
}

@inproceedings{vaitl2022path,
  title={Path-gradient estimators for continuous normalizing flows},
  author={Vaitl, Lorenz and Nicoli, Kim Andrea and Nakajima, Shinichi and Kessel, Pan},
  booktitle={ICML},
  year={2022}
}

@article{vaitl2022gradients,
  title={Gradients should stay on path: better estimators of the reverse-and forward KL divergence for normalizing flows},
  author={Vaitl, Lorenz and Nicoli, Kim A and Nakajima, Shinichi and Kessel, Pan},
  journal={Machine Learning: Science and Technology},
  year={2022}
}

@inproceedings{kim2025adaptive,
  title={Adaptive non-uniform timestep sampling for accelerating diffusion model training},
  author={Kim, Myunsoo and Ki, Donghyeon and Shim, Seong-Woong and Lee, Byung-Jun},
  booktitle={CVPR},
  year={2025}
}

@inproceedings{gower2019sgd,
  title={SGD: General analysis and improved rates},
  author={Gower, Robert Mansel and Loizou, Nicolas and Qian, Xun and Sailanbayev, Alibek and Shulgin, Egor and Richt{\'a}rik, Peter},
  booktitle={ICML},
  year={2019}
}

@article{zhang2025heuristic,
  title={Why Heuristic Weighting Works: A Theoretical Analysis of Denoising Score Matching},
  author={Zhang, Juyan and Newbury, Rhys and Zhang, Xinyang and Tran, Tin and Kulic, Dana and Burke, Michael},
  journal={arXiv preprint arXiv:2508.01597},
  year={2025}
}

@inproceedings{
vaitl2025path,
title={Path Gradients after Flow Matching},
author={Lorenz Vaitl and Leon Klein},
booktitle={NeurIPS},
year={2025}
}
\bibliographystyle{iclr2026_conference}

\newpage
\appendix
\section{Further Related Works}\label{app:related-works}

We walk through more related works and discuss some related theoretical results from both the FM and OT perspectives:

\vspace{-0.5em}
\paragraph{Flow matching and minibatch couplings.}
Normalizing flows and neural ODEs provide the modeling background for flow matching \citep{rezende2015variational,chen2018neural}, while flow matching and rectified flow introduce regression-based objectives for learning continuous transport dynamics \citep{lipman2022flow,liu2022flow}. More broadly, minibatch Wasserstein objectives have been studied as stochastic surrogates for population OT, including their asymptotic and gradient properties \citep{fatras2019learning}; these results concern the minibatch transport objective itself rather than using its random plan as flow-matching supervision. \citet{tong2023improving} introduce OT-CFM and connect the population OT coupling to dynamic OT, whereas \citet{pooladian2023multisample} analyze gradient covariance conditional on a fixed model input and establish asymptotic BatchOT straightness and transport-cost convergence. 

Very recent work characterizes limit points of Reflow combined with fixed-size minibatch OT through cyclical monotonicity and straightness properties \citep{chambolle2026limit}; we note that this concerns repeated Reflow rather than the optimization or integration behavior of a single online minibatch-CFM objective.

\vspace{-0.5em}
\paragraph{Scaling OT couplings for flow matching.}
\citet{zhang2026on} scales Sinkhorn couplings to millions of samples through an offline construction, where paired samples are precomputed independently of model training and served through an augmented data loader. We note that their finite-batch lower bound motivates increasing coupling size, but it does not establish a strict OT-versus-independent comparison or a full stochastic-gradient result. \citet{boite2026expected} characterize the expected exact-OT batch plan and its transport-cost convergence, while their integration result is restricted to an exact velocity field in a one-dimensional two-atom model, and they also show that when the population OT plan is unique, \(\widehat\pi_k\) converges weakly to \(\pi^\star\), i.e.,
\(
\widehat{\pi}_k \xrightarrow[k\to\infty]{\,w\,} \pi^\star
\). However, we note that this coupling-level convergence does not imply any level of convergence for stochastic gradient variance. 

Semidiscrete methods instead construct dataset-level noise-data alignments, avoiding a fresh discrete OT solve at every iteration \citep{mousavi2026flow,kong2026alignflow}. \citet{calvo2025weighted} replace minibatch OT with Gibbs-weighted independent pairs, avoiding an explicit OT solve but potentially altering the marginals, so larger minibatches provide more weighted pairs rather than scaling a larger candidate pool for optimizing source-target assignments. Our setting instead retains online couplings recomputed from fresh minibatches, with clear scaling benefits merely to thousand-level of minibatch size.

\vspace{-0.5em}
\paragraph{Scalable optimal transport.}
Brenier's characterization motivates representing quadratic-cost OT maps as gradients of convex potentials \citep{brenier1991polar}, for which input-convex neural networks provide a convenient parameterization \citep{amos2017input}. This approach has been used to learn OT maps, plans, and Wasserstein barycenters \citep{makkuva2020optimal,korotin2019wasserstein,fan2020scalable,korotin2021continuous,korotin2022neural,chen2024sicnn,asadulaev2024neural}. Within flow matching, optimal flow matching constrains the velocity parameterization to recover the population quadratic-OT displacement directly \citep{kornilov2024optimal}. These methods learn a population-level map or plan through model structure or a separate neural optimization problem, rather than analyzing the random coupling recomputed from each flow-training minibatch.

\vspace{-0.5em}
\paragraph{Optimization and integration theory.}
\citet{reu2026gradient} analyze the variance of the full parameter gradient in restricted Gaussian transport models under specific analytic velocity parameterizations, and show that low gradient variance can also coexist with undesirable deterministic pairings. We note that these results do not associate independent pairing with random online minibatch OT in a general model class. For numerical integration, \citet{tao2026role} decompose the velocity Jacobian into strain and vorticity and establish favorable Euler properties for exact OT displacement fields. This analysis concerns the geometry of a given velocity field rather than an integration-error ordering induced by finite minibatch couplings. Thus, neither line characterizes how full-gradient variance and integration error jointly change with coupling size and entropic regularization.

\vspace{-0.5em}
\paragraph{Computational entropic OT.}
Entropic regularization makes discrete OT amenable to Sinkhorn matrix scaling \citep{cuturi2013sinkhorn}. Subsequent work establishes approximation-complexity guarantees for Sinkhorn, Greenkhorn, accelerated variants, and multimarginal extensions \citep{altschuler2017near,lin2019efficient,lin2022complexity}, while POT provides widely used implementations \citep{flamary2021pot}. FlashSinkhorn targets GPU efficiency by fusing Sinkhorn updates and avoiding materialization of the full coupling matrix in high-bandwidth memory \citep{ye2026flashsinkhorn}, along with prior results with scaled dotproduct attention for onesided EOT problem \citep{litman2025scaled} and later extension with improved learning prior \citep{litman2026you}.

\vspace{-0.5em}
\paragraph{Statistical and limiting theory of entropic OT.}
For finite linear programs, \citet{weed2018explicit} give explicit non-asymptotic guarantees for convergence of entropically penalized solutions toward unregularized optima. At the population level, EOT potentials, plans, and costs converge toward their unregularized counterparts as $\varepsilon\to0$ \citep{nutz2022entropic,malamut2025convergence}. Statistical results establish parametric sample-complexity bounds at fixed regularization \citep{rigollet2025sample}, limit distributions when regularization decreases with sample size \citep{mordant2024entropic}, and finite-sample bias and variance bounds with explicit dependence on regularization and intrinsic dimension \citep{gonzalez2026finite}.

\section{Missing Proofs}\label{app:proof}

We first clarify a measurability point used throughout the proofs. Since \([0,1]\times\mathbb R^d\) and \(\mathbb R^d\) are standard Borel spaces, a regular conditional law of \(U^\pi\) given \((t,X_t^\pi)\) exists. We fix jointly measurable versions of its conditional mean and covariance. Any two such versions agree almost surely, so their values on null sets do not affect the results below.

\subsection{Proposition~\ref{prop:variance-at-optimum}}\label{app:proof-prop-variance}

\begin{proof}
For any \(\theta\) with \(\mathcal L_\pi(\theta)<\infty\), conditioning on \((t,X_t^\pi)\) and expanding the squared loss gives
\[
    \mathcal L_\pi(\theta)=\mathbb E_\pi\!\left[\|v_\theta(t,X_t^\pi)-u^\pi(t,X_t^\pi)\|_2^2\right]+\mathbb E_\pi\!\left[\|U^\pi-u^\pi(t,X_t^\pi)\|_2^2\right].
\]
The second term is independent of \(\theta\). By realizability, the first term can attain zero. Therefore, global optimality implies \(v_{\theta_\pi^\star}(t,X_t^\pi)=u^\pi(t,X_t^\pi)\) almost surely, and  substituting this into the per-sample gradient yields
\[
G_\pi(\theta_\pi^\star)=2J_{\theta_\pi^\star}(t,X_t^\pi)^\top(u^\pi(t,X_t^\pi)-U^\pi).
\]
So, we have
\[
\mathbb E_\pi[G_\pi(\theta_\pi^\star)\mid t,X_t^\pi]=0.
\]
Thus, the conditional-mean term in \cref{eq:var-decomposition} vanishes. Conditional on \((t,X_t^\pi)\), the Jacobian is fixed, so the conditional gradient covariance is \(4J_{\theta_\pi^\star}^\top\Cov_\pi(U^\pi\mid t,X_t^\pi)J_{\theta_\pi^\star}\), with every Jacobian evaluated at \((t,X_t^\pi)\). Taking its expected trace would give us \cref{eq:var-at-optimum}.
\end{proof}

\subsection{Theorem~\ref{thm:strict-ot-variance}}
\label{app:proof-thm-strict-ot-variance}

Before proving Theorem~\ref{thm:strict-ot-variance}, we derive two supporting facts. First, we define the essential lower and upper model sensitivities at the coupling-specific population optimum by $m_\pi$ and $M_\pi$, respectively:
\[
m_\pi:=\operatorname{ess\,inf}\lambda_{\min}(J_{\theta_\pi^\star}(t,X_t^\pi)J_{\theta_\pi^\star}(t,X_t^\pi)^\top),\qquad M_\pi:=\operatorname{ess\,sup}\|J_{\theta_\pi^\star}(t,X_t^\pi)\|_{\mathrm{op}}^2.
\]
Here, the essential extrema are taken under the distribution of \((t,X_t^\pi)\), and together, they can yield the following variance sandwich bound:

\begin{proposition}[Variance sandwich]\label{prop:gradient-variance-sandwich}
Under \emph{Assumptions~\ref{ass:real}} and \emph{\ref{ass:opt}}, every coupling \(\pi\) under comparison satisfies
\begin{equation}
\label{eq:gradient-variance-lower}
    4m_\pi\mathbb E_\pi\!\left[\|U^\pi-u^\pi(t,X_t^\pi)\|_2^2\right]\leq\sigma_{\emph{grad},\pi}^2(\theta_\pi^\star),
\end{equation}
and, if \(M_\pi<\infty\), we have
\begin{equation}\label{eq:gradient-variance-upper}
\sigma_{\emph{grad},\pi}^2(\theta_\pi^\star)\leq4M_\pi\mathbb E_\pi\!\left[\|U^\pi-u^\pi(t,X_t^\pi)\|_2^2\right].
\end{equation}
\end{proposition}

\begin{proof}
For notation simplicity, we denote \(C\!=\!\Cov_\pi(U^\pi\mid t,X_t^\pi)\succeq0\) and \(A\!=\!J_{\theta_\pi^\star}J_{\theta_\pi^\star}^\top\succeq0\) (with the Jacobian evaluated at \((t,X_t^\pi)\)). Cyclicity of the trace gives \(\Tr(J_{\theta_\pi^\star}^\top C J_{\theta_\pi^\star})=\Tr(CA)\), and
\[
\lambda_{\min}(A)\Tr(C)\leq\Tr(CA)\leq\|A\|_{\mathrm{op}}\Tr(C)=\|J_{\theta_\pi^\star}\|_{\mathrm{op}}^2\Tr(C).
\]
By the definition of the essential infimum, $\lambda_{\min}(A)\ge m_\pi$ almost surely. Applying the lower trace inequality pointwise and taking expectations in \cref{eq:var-at-optimum}, we have
\[
\sigma_{\text{grad},\pi}^{2}(\theta_\pi^\star)
=4\mathbb{E}_\pi[\operatorname{Tr}(CA)]
\ge 4\mathbb{E}_\pi[m_\pi\operatorname{Tr}(C)]
=4m_\pi\mathbb{E}_\pi[\operatorname{Tr}(C)].
\]
Similarly, when $M_\pi<\infty$, the definition of the essential supremum implies that
\[
\|A\|_{\mathrm{op}}
=\|J_{\theta_\pi^\star}(t,X_t^\pi)\|_{\mathrm{op}}^2
\le M_\pi
\qquad\text{almost surely}.
\]
Applying the upper trace inequality pointwise and taking expectations yields that
\[
\sigma_{\text{grad},\pi}^{2}(\theta_\pi^\star)
=4\mathbb{E}_\pi[\operatorname{Tr}(CA)]
\le 4\mathbb{E}_\pi[M_\pi\operatorname{Tr}(C)]
=4M_\pi\mathbb{E}_\pi[\operatorname{Tr}(C)].
\]
Finally, since $u^\pi(t,X_t^\pi)=\mathbb{E}_\pi[U^\pi\mid t,X_t^\pi]$, the definition of conditional covariance and the tower property gives us:
\[
\begin{aligned}
\mathbb{E}_\pi[\operatorname{Tr}(C)]
&=\mathbb{E}_\pi\!\left[
\mathbb{E}_\pi\!\left[
\|U^\pi-u^\pi(t,X_t^\pi)\|_2^2
\,\middle|\,t,X_t^\pi
\right]
\right]\\
&=\mathbb{E}_\pi\!\left[
\|U^\pi-u^\pi(t,X_t^\pi)\|_2^2
\right].
\end{aligned}
\]
Substituting this identity into the preceding lower and upper bounds proves \cref{eq:gradient-variance-lower,eq:gradient-variance-upper}, respectively. We also provide some further discussions over Jacobian norm and essential extrema in Remark~\ref{rmk:ess}.
\end{proof}

\begin{remark}\label{rmk:ess}
At each \((t,x)\), \(\lambda_{\min}(J_{\theta_\pi^\star}(t,x)J_{\theta_\pi^\star}(t,x)^\top)\) and \(\|J_{\theta_\pi^\star}(t,x)\|_{\mathrm{op}}^2\) measure the weakest and strongest model sensitivities across output directions. Accordingly, \(m_\pi\) and \(M_\pi\) are their essential lower and upper bounds under the distribution of \((t,X_t^\pi)\) (also note that the essential extrema are suitable here, because \emph{Assumption~\ref{ass:opt}} requires differentiability only almost everywhere and are therefore unaffected by arbitrary values on null sets). We also note that the operator norm in \(M_\pi\) gives a sharper bound than the Frobenius norm, while playing the same role as the maximum Jacobian factor in \citet[Lemma~3.2]{pooladian2023multisample}.
\end{remark}

\vspace{-0.5em}
Second, we show the nondegeneracy supplied by an affine output layer with trainable bias.
\begin{fact}\label{fact:affine}
Over standard OT-CFM's MLP structure of width-$d$ output layer trainable bias, let \(\theta\!=\!(\eta,b)\), where \(b\in\mathbb R^d\), one always have \(v_{\eta,b}(t,x)=\bar v_\eta(t,x)+b\). We note that is explicitly shown in \citet[Appendix~E.2]{tong2023improving} and its code implementation\footnote{See \href{https://github.com/atong01/conditional-flow-matching/blob/cacd4dc8a5b6f50de63780e8f73754c9371f15db/torchcfm/models/models.py}{code} at commit \texttt{cacd4dc}, where it uses \texttt{torch.nn.Linear} with \texttt{bias=True} by default}. 
\end{fact}

With Fact~\ref{fact:affine}, we then establish Lemma~\ref{lem:additive-output-bias}:
\begin{lemma} \label{lem:additive-output-bias}
Under \emph{Assumptions~\ref{ass:real}} and~\emph{\ref{ass:opt}}, we have \(J_{\theta_\pi^\star}(t,x)J_{\theta_\pi^\star}(t,x)^\top\succeq I_d\) almost everywhere, so \(m_\pi\geq1\) and
\(
\sigma_{\emph{grad},\pi}^2(\theta_\pi^\star)\geq4\mathbb E_\pi\!\left[\|U^\pi-u^\pi(t,X_t^\pi)\|_2^2\right]
\).
\end{lemma}

\begin{proof}
Since \(\nabla_bv_{\eta,b}(t,x)=I_d\), the full parameter Jacobian at the coupling-specific optimum is
\[
J_{\theta_\pi^\star}(t,x)=\begin{bmatrix}\nabla_\eta\bar v_{\eta_\pi^\star}(t,x)&I_d\end{bmatrix}, \quad J_{\theta_\pi^\star}(t,x)J_{\theta_\pi^\star}(t,x)^\top=\nabla_\eta\bar v_{\eta_\pi^\star}(t,x)\nabla_\eta\bar v_{\eta_\pi^\star}(t,x)^\top+I_d\succeq I_d.
\]
The conclusion follows from \cref{eq:gradient-variance-lower}.
\end{proof}
\vspace{-0.5em}

We now prove the theorem:
\vspace{-0.5em}
\begin{proof}[Proof of Theorem~\ref{thm:strict-ot-variance}]
We first consider population OT. Choose a Borel representative of the Brenier map and a Borel set \(D\) of full \(q_0\)-measure on which \(T=\nabla\varphi\). Monotonicity of gradients of convex functions gives us
\[
    \langle T(x)-T(y),x-y\rangle\geq0,\qquad x,y\in D.
\]
For \(t<1\), we write \(F_t(x)=(1-t)x+tT(x)\). If \(x\ne y\), then we have
\[
    \langle F_t(x)-F_t(y),x-y\rangle\geq(1-t)\|x-y\|_2^2>0,
\]
so \(F_t\) is injective on \(D\).

It's clear that Borel map \(\Phi(t,x)=(t,F_t(x))\) from \([0,1)\times D\) into \([0,1)\times\mathbb R^d\) is injective. By the Lusin-Souslin theorem, its image is Borel and its inverse is Borel as well. Extending the inverse by a fixed point off its image yields a Borel map \(\Psi\) such that, almost surely, we have
\[
    X_0=\operatorname{pr}_2[\Psi(t,X_t^{\pi^\star})].
\]
Since \(t\!=\!1\) has probability zero, this recovers \(X_0\) from \((t,X_t^{\pi^\star})\) almost surely, so \(U^{\pi^\star}=T(X_0)-X_0\) is measurable with respect to \(\sigma(t,X_t^{\pi^\star})\) (and also \(\Cov_{\pi^\star}(U^{\pi^\star}\mid t,X_t^{\pi^\star})=0\) almost surely). 

Assumptions~\ref{ass:real} and \ref{ass:opt} then give \(v_{\theta_{\pi^\star}^\star}(t,X_t^{\pi^\star})=u^{\pi^\star}(t,X_t^{\pi^\star})=U^{\pi^\star}\) almost surely. Therefore, every per-sample gradient vanishes almost surely and \(\sigma_{\text{grad},\pi^\star}^2(\theta_{\pi^\star}^\star)=0.\)

We next consider independent coupling and first extend \(f_0,f_1\) by zero outside \(\Omega\). For fixed \(t\), the map
\[
    (x_0,x_1)\mapsto(x,u)=((1-t)x_0+tx_1,x_1-x_0)
\]
has inverse \(x_0=x-tu\), \(x_1=x+(1-t)u\), and absolute Jacobian determinant one. Therefore, \((X_t,U)\) has joint density \(h_t(x,u)=f_0(x-tu)f_1(x+(1-t)u)\). For notation clarity, we define
\[
\rho_t(x)=\int h_t(x,u)\,\mathrm du,\qquad m_t(x)=\rho_t(x)^{-1}\int u\,h_t(x,u)\,\mathrm du,
\]
and, when \(\rho_t(x)>0\), we can derive that
\[
\Cov_{\pi_{\text{ind}}}(U^{\pi_{\text{ind}}}\mid t,X_t^{\pi_{\text{ind}}}=x)=\rho_t(x)^{-1}\int(u-m_t(x))(u-m_t(x))^\top h_t(x,u)\,\mathrm du,
\]
which sets \(m_t(x)=0\) when \(\rho_t(x)=0\). Thus, \(\rho_t\) is the density of \(X_t\), while note that \(m_t\) and above are versions of its conditional mean velocity and covariance, respectively.

Next, we fix \(\bar x\in\Omega^\circ\) and \(r>0\) with \(B(\bar x,2r)\subset\Omega\). 

For \(x\in B(\bar x,r)\), \(u\in B(0,r)\), and \(t\in[0,1]\), it's clear that both \(x-tu\) and \(x+(1-t)u\) lie in \(\Omega\). Therefore, we have
\[
    h_t(x,u)\geq\gamma^2,\qquad \rho_t(x)\geq\gamma^2\omega_dr^d>0.
\]
For every \(a\in\mathbb R^d\) and such \((t,x)\), we can further derive
\begin{align*}
    \rho_t(x)a^\top\Cov_{\pi_{\text{ind}}}(U^{\pi_{\text{ind}}}\mid t,X_t^{\pi_{\text{ind}}}=x)a&\geq\gamma^2\int_{B(0,r)}(a^\top u-a^\top m_t(x))^2\,\mathrm du \\
    & \geq\frac{\gamma^2\omega_dr^{d+2}}{d+2}\|a\|_2^2,
\end{align*}
where the second step follows because the cross term integrates to zero by symmetry and \(\int_{B(0,r)}uu^\top\,\mathrm du=\omega_dr^{d+2}I_d/(d+2)\). Therefore, we can finally derive
\[
    \rho_t(x)\Cov_{\pi_{\text{ind}}}(U^{\pi_{\text{ind}}}\mid t,X_t^{\pi_{\text{ind}}}=x)\succeq\frac{\gamma^2\omega_dr^{d+2}}{d+2}I_d.
\]
The joint law of \((t,X_t)\) has density \(\rho_t(x)\,\mathrm dt\,\mathrm dx\). Taking the trace and integrating over \(t\in[0,1]\) and \(x\in B(\bar x,r)\) eventually yields
\[
    \mathbb E_{\pi_{\text{ind}}}\!\left[\|U^{\pi_{\text{ind}}}-u^{\pi_{\text{ind}}}(t,X_t^{\pi_{\text{ind}}})\|_2^2\right]\geq\frac{d\gamma^2\omega_d^2}{d+2}r^{2d+2}.
\]
By Lemma~\ref{lem:additive-output-bias}, we have
\[
    \sigma_{\text{grad},\pi_{\text{ind}}}^2(\theta_{\pi_{\text{ind}}}^\star)\geq\frac{4d\gamma^2\omega_d^2}{d+2}r^{2d+2}>0.
\]
Subtracting the zero population-OT variance, we can conclude \cref{eq:strict-ot-gap}.
\end{proof}

\subsection{Proposition~\ref{prop:fm-loss-cost-gap}}
\label{app:proof-fm-loss-cost-gap}

\begin{proof}
By Assumptions~\ref{ass:real} and \ref{ass:opt}, \(v_{\theta_\pi^\star}(t,X_t^\pi)=u^\pi(t,X_t^\pi)\) almost surely. Therefore, the loss at optimality equals the conditional target variance, and conditional orthogonality gives us that
\begin{equation}
\label{eq:velocity-second-moment-decomposition}
\mathbb E_\pi[\|U^\pi\|_2^2]=\mathbb E_\pi[\|u^\pi(t,X_t^\pi)\|_2^2]+\mathcal L_\pi(\theta_\pi^\star).
\end{equation}
For every smooth compactly supported test function \(\phi\), differentiation along the linear interpolation and conditioning on \((t,X_t^\pi)\) at each fixed \(t\), we have
\[
\frac{\mathrm d}{\mathrm dt}\mathbb E[\phi(X_t^\pi)]=\mathbb E[\nabla\phi(X_t^\pi)^\top U^\pi]=\mathbb E[\nabla\phi(X_t^\pi)^\top u^\pi(t,X_t^\pi)].
\]
Thus, together with \(u^\pi\), the marginal path induced by \(\pi\) satisfies the continuity equation and is admissible in the Benamou--Brenier formulation \citep{benamou2000computational}. So, its integrated kinetic energy is at least \(W_2^2(q_0,q_1)\), i.e.,
\[
\mathbb E_\pi[\|u^\pi(t,X_t^\pi)\|_2^2]\geq W_2^2(q_0,q_1).
\]
Since \(\mathbb E_\pi[\|U^\pi\|_2^2]=\mathbb E_\pi[\|X_1-X_0\|_2^2]\), combining this inequality with \cref{eq:velocity-second-moment-decomposition} proves \cref{eq:fm-loss-cost-gap}. For \(\pi^\star\), recall its transport cost is exactly \(W_2^2(q_0,q_1)\), so the same bound also gives \(\mathcal L_{\pi^\star}(\theta_{\pi^\star}^\star)=0\).
\end{proof}

\subsection{Lemma~\ref{lem:eot-strict-small-epsilon}}
\label{app:proof-eot-small-epsilon}

\begin{proof}
Let \(0\!<\!\varepsilon_1\!<\!\varepsilon_2\). Since \(q_0,q_1\in\mathcal P_2(\mathbb R^d)\), independent coupling is a finite-objective competitor, so both quantities are finite. We use \(\pi_{\varepsilon_2}\) as a competitor for the \(\varepsilon_1\)-problem and \(\pi_{\varepsilon_1}\) as a competitor for the \(\varepsilon_2\)-problem. Their optimality gives us that
\[
\begin{aligned}
\frac{\mathbb E_{\pi_{\varepsilon_1}}[\|X_1-X_0\|_2^2]}{2d}
+\varepsilon_1\operatorname{KL}(\pi_{\varepsilon_1}\mid\pi_{\text{ind}})
&\leq
\frac{\mathbb E_{\pi_{\varepsilon_2}}[\|X_1-X_0\|_2^2]}{2d}
+\varepsilon_1\operatorname{KL}(\pi_{\varepsilon_2}\mid\pi_{\text{ind}}),\\
\frac{\mathbb E_{\pi_{\varepsilon_2}}[\|X_1-X_0\|_2^2]}{2d}
+\varepsilon_2\operatorname{KL}(\pi_{\varepsilon_2}\mid\pi_{\text{ind}})
&\leq
\frac{\mathbb E_{\pi_{\varepsilon_1}}[\|X_1-X_0\|_2^2]}{2d}
+\varepsilon_2\operatorname{KL}(\pi_{\varepsilon_1}\mid\pi_{\text{ind}}).
\end{aligned}
\]
Adding these inequalities and canceling the transport costs yields
\[
(\varepsilon_2-\varepsilon_1)
\left(
\operatorname{KL}(\pi_{\varepsilon_2}\mid\pi_{\text{ind}})
-\operatorname{KL}(\pi_{\varepsilon_1}\mid\pi_{\text{ind}})
\right)
\leq 0.
\]
Because \(\varepsilon_2>\varepsilon_1\), we have
\[
\operatorname{KL}(\pi_{\varepsilon_1}\mid\pi_{\text{ind}})
\geq
\operatorname{KL}(\pi_{\varepsilon_2}\mid\pi_{\text{ind}}).
\]
Substituting this into the first optimality inequality further gives
\[
\mathbb E_{\pi_{\varepsilon_1}}[\|X_1-X_0\|_2^2]
-\mathbb E_{\pi_{\varepsilon_2}}[\|X_1-X_0\|_2^2]
\leq
2d\varepsilon_1
\left(
\operatorname{KL}(\pi_{\varepsilon_2}\mid\pi_{\text{ind}})
-\operatorname{KL}(\pi_{\varepsilon_1}\mid\pi_{\text{ind}})
\right)
\leq 0.
\]
Therefore, the relative entropy is non-increasing in \(\varepsilon\), while the transport cost is non-decreasing:
\begin{equation}
\label{eq:eot-cost-monotonicity}
\begin{aligned}
\operatorname{KL}(\pi_{\varepsilon_1}\mid\pi_{\text{ind}})
&\geq\operatorname{KL}(\pi_{\varepsilon_2}\mid\pi_{\text{ind}}),\\
\mathbb E_{\pi_{\varepsilon_1}}[\|X_1-X_0\|_2^2]
&\leq\mathbb E_{\pi_{\varepsilon_2}}[\|X_1-X_0\|_2^2].
\end{aligned}
\end{equation}
For every \(0<\varepsilon\leq\varepsilon_0\), we can now apply \cref{eq:fm-loss-cost-gap,eq:eot-cost-monotonicity,eq:eot-small-epsilon-certificate} in order to obtain
\[
\begin{aligned}
\mathcal L_{\pi_\varepsilon}(\theta_{\pi_\varepsilon}^\star)
&\leq \mathbb E_{\pi_\varepsilon}[\|X_1-X_0\|_2^2]-W_2^2(q_0,q_1)\\
&\leq \mathbb E_{\pi_{\varepsilon_0}}[\|X_1-X_0\|_2^2]-W_2^2(q_0,q_1)\\
&<\mathcal L_{\pi_{\text{ind}}}(\theta_{\pi_{\text{ind}}}^\star),
\end{aligned}
\]
which proves the strict upper inequality in \cref{eq:eot-strict-population-loss-chain}.
\[
\exp\!\left(
\frac{\varphi_\varepsilon(x_0)+\psi_\varepsilon(x_1)-\|x_1-x_0\|_2^2/(2d)}{\varepsilon}
\right)>0
\]
for \(\pi_{\text{ind}}\)-almost every \((x_0,x_1)\) \citep{nutz2022entropic}. Thus, \(\pi_\varepsilon\!\sim\!\pi_{\text{ind}}\), meaning that the two couplings have exactly the same null sets. Therefore, we have
\[
\mathbb E_\pi[U^\pi\mid t,X_t^\pi]
=
f(t,X_t^\pi)
=
U^\pi
\qquad\text{almost surely},
\]
so the loss is zero.
\[
\bigl(t,x-tu,x+(1-t)u\bigr).
\]
We write \(P_\pi=\Phi_\#(\lambda\otimes\pi)\) for the induced law of \((t,X_t^\pi,U^\pi)\). For any Borel set \(A\), the definition of pushforward gives
\[
\begin{aligned}
P_{\pi_\varepsilon}(A)=0
&\quad\Longleftrightarrow\quad
(\lambda\otimes\pi_\varepsilon)(\Phi^{-1}(A))=0\\
&\quad\Longleftrightarrow\quad
(\lambda\otimes\pi_{\text{ind}})(\Phi^{-1}(A))=0\\
&\quad\Longleftrightarrow\quad
P_{\pi_{\text{ind}}}(A)=0.
\end{aligned}
\]
Therefore, \(P_{\pi_\varepsilon}\sim P_{\pi_{\text{ind}}}\).
For any Borel measurable \(f:[0,1]\times\mathbb R^d\to\mathbb R^d\), its graph
\[
A_f:=\{(t,x,u):u=f(t,x)\}
\]
is Borel. Applying the preceding equivalence to \(A_f^c\) gives
\[
P_{\pi_\varepsilon}(A_f)=1
\quad\Longleftrightarrow\quad
P_{\pi_{\text{ind}}}(A_f)=1.
\]
Thus, a measurable function determines \(U\) from \((t,X_t)\) almost surely under one law exactly when it does under the other. Combining this with the zero-loss characterization above, we obtain
\[
\mathcal L_{\pi_\varepsilon}(\theta_{\pi_\varepsilon}^\star)=0
\quad\Longleftrightarrow\quad
\mathcal L_{\pi_{\text{ind}}}(\theta_{\pi_{\text{ind}}}^\star)=0.
\]
Therefore, as \(\mathcal L_{\pi_\varepsilon}(\theta_{\pi_\varepsilon}^\star)>0\), with Proposition~\ref{prop:fm-loss-cost-gap} and optimality of \(\pi^\star\), we end up with
\[
0\leq\mathcal L_{\pi^\star}(\theta_{\pi^\star}^\star)
\leq
\mathbb E_{\pi^\star}[\|X_1-X_0\|_2^2]-W_2^2(q_0,q_1)
=0.
\]
Consequently, \(\mathcal L_{\pi^\star}(\theta_{\pi^\star}^\star)=0<\mathcal L_{\pi_\varepsilon}(\theta_{\pi_\varepsilon}^\star)\), which proves the strict lower inequality.
\end{proof}

\subsection{Remark~\ref{rem:eot-epsilon-zero-existence}}
\label{app:proof-eot-epsilon-existence}

\begin{proof}
Because \(q_0,q_1\in\mathcal P_2(\mathbb R^d)\), the normalized quadratic cost satisfies
\[
\begin{aligned}
\int \frac{\|x_1-x_0\|_2^2}{2d}\,(q_0\otimes q_1)(dx_0,dx_1)
\leq \frac{1}{d}\left(
\int \|x_0\|_2^2\,q_0(dx_0)+
\int \|x_1\|_2^2\,q_1(dx_1)
\right) <\infty.
\end{aligned}
\]
The cost is also continuous, so vanishing-regularization convergence of entropic OT implies that every weak cluster point of \(\pi_\varepsilon\) as \(\varepsilon\downarrow0\) is a quadratic-OT optimizer \citep{nutz2022entropic}.

To pass from weak convergence to convergence of the unbounded quadratic costs, we make the uniform-integrability argument explicit. This is the same de~la~Vall\'ee--Poussin argument used below in Lemma~\ref{lem:entropic-fixed-batch-zero-temperature}; for similar one usage, also see the proof of Lemma~6.4 in \citet{beiglbock2017complete}. Since the squared norms under both marginals are integrable, the de~la~Vall\'ee--Poussin criterion supplies a common nondecreasing, convex function \(\Phi:[0,\infty)\to[0,\infty)\) such that
\[
\lim_{s\to\infty}\frac{\Phi(s)}{s}=\infty,
\qquad
\int\Phi(\|x_0\|_2^2)\,q_0(dx_0)
+
\int\Phi(\|x_1\|_2^2)\,q_1(dx_1)
<\infty.
\]
Indeed, we can apply the criterion to the equally weighted mixture of the two squared-norm distributions, which has finite first moment. For every pair \((x_0,x_1)\), we have
\[
\frac{\|x_1-x_0\|_2^2}{4}
\leq
\frac{\|x_0\|_2^2+\|x_1\|_2^2}{2}.
\]
Therefore, monotonicity and convexity of \(\Phi\) give us the following:
\[
\Phi\!\left(\frac{\|x_1-x_0\|_2^2}{4}\right)
\leq
\frac12\Phi(\|x_0\|_2^2)
+
\frac12\Phi(\|x_1\|_2^2).
\]
Integrating under any coupling \(\pi\in\Pi(q_0,q_1)\), its fixed marginals yield the uniform bound
\[
\sup_{\pi\in\Pi(q_0,q_1)}
\int\Phi\!\left(\frac{\|x_1-x_0\|_2^2}{4}\right)\,d\pi
\leq
K=
\frac12\int\Phi(\|x_0\|_2^2)\,q_0(dx_0)
+
\frac12\int\Phi(\|x_1\|_2^2)\,q_1(dx_1)
<\infty.
\]
For any \(a>0\), superlinearity gives a threshold \(M_a<\infty\) such that \(s\leq a\Phi(\tfrac{s}{4})\) whenever \(s>M_a\). Consequently, for every \(M\geq M_a\), we have
\[
\sup_{\pi\in\Pi(q_0,q_1)}
\int \|x_1-x_0\|_2^2\,\mathbf 1_{\{\|x_1-x_0\|_2^2>M\}}\,d\pi
\leq
a\sup_{\pi\in\Pi(q_0,q_1)}
\int\Phi\!\left(\frac{\|x_1-x_0\|_2^2}{4}\right)\,d\pi
\leq aK.
\]
Since \(a>0\) is arbitrary, we obtain
\[
\tau(M):=
\sup_{\pi\in\Pi(q_0,q_1)}
\int \|x_1-x_0\|_2^2\,\mathbf 1_{\{\|x_1-x_0\|_2^2>M\}}\,d\pi
\longrightarrow0
\qquad\text{as }M\to\infty.
\]
Thus, the quadratic costs are uniformly integrable over all couplings with these fixed marginals.

We now use this bound to establish the cost limit. First, the family \(\{\pi_\varepsilon\}_{\varepsilon>0}\) is tight. To see this, for any \(\delta>0\), we choose compact sets \(K_0,K_1\subset\mathbb R^d\) such that \(q_i(K_i^c)<\delta/2\). The marginal constraints give us:
\[
\sup_{\varepsilon>0}
\pi_\varepsilon\bigl((K_0\times K_1)^c\bigr)
\leq q_0(K_0^c)+q_1(K_1^c)
<\delta.
\]
Therefore, from any sequence \(\varepsilon_n\downarrow0\), we can extract a subsequence such that
\[
\pi_{\varepsilon_{n_j}}\Rightarrow\bar\pi.
\]
The limit \(\bar\pi\) still has marginals \(q_0,q_1\), and the vanishing-regularization result above gives us
\[
\mathbb E_{\bar\pi}[\|X_1-X_0\|_2^2]
=
W_2^2(q_0,q_1).
\]

For fixed \(M>0\), we define
\[
c_M:=\min\{\|x_1-x_0\|_2^2,M\}.
\]
As \(c_M\) is bounded and continuous, weak convergence further gives us
\[
\int c_M\,d\pi_{\varepsilon_{n_j}}
\longrightarrow
\int c_M\,d\bar\pi.
\]
Moreover, since
\[
0\leq \|x_1-x_0\|_2^2-c_M
\leq
\|x_1-x_0\|_2^2\,\mathbf 1_{\{\|x_1-x_0\|_2^2>M\}},
\]
so we have
\[
\begin{aligned}
\left|
\mathbb E_{\pi_{\varepsilon_{n_j}}}[\|X_1\!-\!X_0\|_2^2]\!-\!
\mathbb E_{\bar\pi}[\|X_1\!-\!X_0\|_2^2]
\right|
&\leq
\left|
\int c_M\,d\pi_{\varepsilon_{n_j}}\!-\!
\int c_M\,d\bar\pi
\right|\\
&\quad\!+\!
\int(\|x_1-x_0\|_2^2\!-\!c_M)\,d\pi_{\varepsilon_{n_j}}
\!+\!
\int(\|x_1-x_0\|_2^2\!-\!c_M)\,d\bar\pi\\
&\leq
\left|
\int c_M\,d\pi_{\varepsilon_{n_j}}
-
\int c_M\,d\bar\pi
\right|
+
2\tau(M).
\end{aligned}
\]
Taking \(j\to\infty\) at fixed \(M\), and then \(M\to\infty\), we have
\[
\mathbb E_{\pi_{\varepsilon_{n_j}}}[\|X_1-X_0\|_2^2]
\longrightarrow
\mathbb E_{\bar\pi}[\|X_1-X_0\|_2^2]
=
W_2^2(q_0,q_1).
\]
This subsequence argument also gives convergence of the full family. Otherwise, since
\[
\mathbb E_{\pi_\varepsilon}[\|X_1-X_0\|_2^2]
\geq
W_2^2(q_0,q_1),
\]
there would exist \(\delta>0\) and a sequence \(\varepsilon_n\downarrow0\) such that
\[
\mathbb E_{\pi_{\varepsilon_n}}[\|X_1-X_0\|_2^2]
-
W_2^2(q_0,q_1)
\geq\delta
\]
for every \(n\). The preceding claim would produce a subsequence whose costs converge to \(W_2^2(q_0,q_1)\), which is a contradiction, so we get
\begin{equation}
\label{eq:eot-cost-convergence}
\mathbb E_{\pi_\varepsilon}[\|X_1-X_0\|_2^2]
\longrightarrow W_2^2(q_0,q_1)
\qquad\text{as }\varepsilon\downarrow0.
\end{equation}
Since \(\mathcal L_{\pi_{\text{ind}}}(\theta_{\pi_{\text{ind}}}^\star)>0\), we can now choose \(\varepsilon_0>0\) sufficiently small that
\[
0\leq
\mathbb E_{\pi_\varepsilon}[\|X_1-X_0\|_2^2]
-
W_2^2(q_0,q_1)
<
\frac{\mathcal L_{\pi_{\text{ind}}}(\theta_{\pi_{\text{ind}}}^\star)}{2}
<
\mathcal L_{\pi_{\text{ind}}}(\theta_{\pi_{\text{ind}}}^\star),
\qquad
0<\varepsilon\leq\varepsilon_0.
\]
In particular, \cref{eq:eot-small-epsilon-certificate} holds at \(\varepsilon_0\).

For the explicit cutoff, fix a competitor \(\widetilde\pi\in\Pi(q_0,q_1)\) satisfying
\[
0<
\operatorname{KL}(\widetilde\pi\|\pi_{\text{ind}})
<\infty,
\qquad
\mathbb E_{\widetilde\pi}[\|X_1-X_0\|_2^2]
-
W_2^2(q_0,q_1)
<
\mathcal L_{\pi_{\text{ind}}}(\theta_{\pi_{\text{ind}}}^\star).
\]
Multiplying the objective in \cref{eq:population-eot} by \(2d\), optimality of \(\pi_\varepsilon\) against \(\widetilde\pi\) gives
\[
\mathbb E_{\pi_\varepsilon}[\|X_1-X_0\|_2^2]
+
2d\varepsilon
\operatorname{KL}(\pi_\varepsilon\|\pi_{\text{ind}})
\leq
\mathbb E_{\widetilde\pi}[\|X_1-X_0\|_2^2]
+
2d\varepsilon
\operatorname{KL}(\widetilde\pi\|\pi_{\text{ind}}).
\]
Since \( \operatorname{KL}(\pi_\varepsilon\|\pi_{\text{ind}})\geq0\), dropping its contribution and subtracting \(W_2^2(q_0,q_1)\) yields
\[
\mathbb E_{\pi_\varepsilon}[\|X_1-X_0\|_2^2]
-
W_2^2(q_0,q_1)
\leq
\mathbb E_{\widetilde\pi}[\|X_1-X_0\|_2^2]
-
W_2^2(q_0,q_1)
+
2d\varepsilon
\operatorname{KL}(\widetilde\pi\|\pi_{\text{ind}}).
\]
The competitor condition ensures that
\[
\Delta:=
\mathcal L_{\pi_{\text{ind}}}(\theta_{\pi_{\text{ind}}}^\star)
+
W_2^2(q_0,q_1)
-
\mathbb E_{\widetilde\pi}[\|X_1-X_0\|_2^2]
>0.
\]
Thus, whenever \(0<\varepsilon<\frac{\Delta}{2d\operatorname{KL}(\widetilde\pi\|\pi_{\text{ind}})}\), which is exactly \cref{eq:eot-explicit-competitor-cutoff}, we have
\[
2d\varepsilon
\operatorname{KL}(\widetilde\pi\|\pi_{\text{ind}})
<\Delta.
\]
Substituting this into the preceding cost bound gives
\[
\begin{aligned}
\mathbb E_{\pi_\varepsilon}[\|X_1-X_0\|_2^2]
-
W_2^2(q_0,q_1)
&\leq
\mathbb E_{\widetilde\pi}[\|X_1-X_0\|_2^2]
-
W_2^2(q_0,q_1)
+
2d\varepsilon
\operatorname{KL}(\widetilde\pi\|\pi_{\text{ind}})\\
&<
\mathbb E_{\widetilde\pi}[\|X_1-X_0\|_2^2]
-
W_2^2(q_0,q_1)
+\Delta\\
&=
\mathcal L_{\pi_{\text{ind}}}(\theta_{\pi_{\text{ind}}}^\star).
\end{aligned}
\]
So, \cref{eq:eot-small-epsilon-certificate} holds with this \(\varepsilon\) as the reference \(\varepsilon_0\), and Lemma~\ref{lem:eot-strict-small-epsilon} gives the claimed strict ordering in \cref{eq:eot-strict-population-loss-chain}.

Finally, we verify that a competitor satisfying both conditions always exists. By \cref{eq:eot-cost-convergence}, we choose a sufficiently small \(\eta>0\) such that
\[
\mathbb E_{\pi_\eta}[\|X_1-X_0\|_2^2]
-
W_2^2(q_0,q_1)
<
\mathcal L_{\pi_{\text{ind}}}(\theta_{\pi_{\text{ind}}}^\star),
\]
and set \(\widetilde\pi=\pi_\eta\). To establish finiteness of its relative entropy, we use \(\pi_{\text{ind}}\) as a competitor in the \(\eta\)-problem. Obviously, \(\operatorname{KL}(\pi_{\text{ind}}\|\pi_{\text{ind}})=0\), and optimality gives us
\[
\mathbb E_{\pi_\eta}[\|X_1-X_0\|_2^2]
+
2d\eta\operatorname{KL}(\pi_\eta\|\pi_{\text{ind}})
\leq
\mathbb E_{\pi_{\text{ind}}}[\|X_1-X_0\|_2^2]
<\infty.
\]
Consequently, we have
\[
0\leq
\operatorname{KL}(\pi_\eta\|\pi_{\text{ind}})
\leq
\frac{
\mathbb E_{\pi_{\text{ind}}}[\|X_1-X_0\|_2^2]
-
\mathbb E_{\pi_\eta}[\|X_1-X_0\|_2^2]
}{2d\eta}
\leq
\frac{
\mathbb E_{\pi_{\text{ind}}}[\|X_1-X_0\|_2^2]
}{2d\eta}
<\infty.
\]
We still need to show that
\[
\operatorname{KL}(\pi_\eta\|\pi_{\text{ind}})>0.
\]
Again, we can consider the contradiction, i.e., suppose instead that
\[
\operatorname{KL}(\pi_\eta\|\pi_{\text{ind}})=0.
\]
Since relative entropy vanishes only when the two probability measures agree, this would imply \(\pi_\eta=\pi_{\text{ind}}\). Applying \cref{eq:fm-loss-cost-gap} to independent coupling then gives
\[
\mathcal L_{\pi_{\text{ind}}}(\theta_{\pi_{\text{ind}}}^\star)
\leq
\mathbb E_{\pi_{\text{ind}}}[\|X_1-X_0\|_2^2]
-
W_2^2(q_0,q_1)
=
\mathbb E_{\pi_\eta}[\|X_1-X_0\|_2^2]
-
W_2^2(q_0,q_1)
<
\mathcal L_{\pi_{\text{ind}}}(\theta_{\pi_{\text{ind}}}^\star),
\]
which is a contradiction. Therefore, we can now conclude 
\[
0<
\operatorname{KL}(\pi_\eta\|\pi_{\text{ind}})
<\infty,
\]
and \(\widetilde\pi=\pi_\eta\) satisfies all the required conditions, so the guarantee we claimed holds true.
\end{proof}

\subsection{Theorem~\ref{thm:eot-gradient-variance}}
\label{app:proof-eot-gradient-variance}

Before establish the proof, we first justify the sufficient condition for \(\mathcal L_{\pi_{\text{ind}}}(\theta_{\pi_{\text{ind}}}^\star)>0\) stated in Lemma~\ref{lem:eot-strict-small-epsilon}.

Suppose $q_0$ has a Lebesgue density $f_0>0$ almost everywhere and $q_1$ is not a Dirac mass. Under $\pi_{\text{ind}}$, the joint law of $(X_0,X_1)$ is $f_0(x_0)\,dx_0\,q_1(dy)$. Fix $t\in(0,1)$ and we write $X_t=(1-t)X_0+tX_1$, then for any Borel sets $A,B\subseteq\mathbb R^d$, independence gives us that
\[
\mathbb P(X_t\in A,\ X_1\in B)=\int_B\int_{\mathbb R^d}\mathbf 1_A\bigl((1-t)x_0+ty\bigr)f_0(x_0)\,dx_0\,q_1(dy).
\]
For each fixed $y$, we can do the change of variables $x=(1-t)x_0+ty$, so that $x_0=(x-ty)/(1-t)$ and $dx_0=(1-t)^{-d}dx$. Substituting into the preceding integral and applying Tonelli's theorem yields the following 
\begin{align*}
\mathbb P(X_t\in A,\ X_1\in B)&=\int_B\int_A(1-t)^{-d}f_0\!\left(\frac{x-ty}{1-t}\right)\,dx\,q_1(dy) \\
& =\int_A\int_B(1-t)^{-d}f_0\!\left(\frac{x-ty}{1-t}\right)\,q_1(dy)\,dx.
\end{align*}
By taking $B=\mathbb R^d$, we can show that $X_t$ has Lebesgue density
\[
\rho_t(x)=(1-t)^{-d}\int_{\mathbb R^d}f_0\!\left(\frac{x-tz}{1-t}\right)q_1(dz).
\]
For every $x$ such that $0<\rho_t(x)<\infty$, we abbreviate the following for clarity:
\[
K_t(x,dy)=\frac{(1-t)^{-d}f_0\!\left(\frac{x-ty}{1-t}\right)}{\rho_t(x)}\,q_1(dy).
\]
The preceding joint-probability identity then becomes $\mathbb P(X_t\in A,\ X_1\in B)=\int_A K_t(x,B)\rho_t(x)\,dx$, so $K_t(x,\cdot)$ is a version of the conditional law of $X_1$ given $X_t=x$. Consequently, for almost every $x$ under the law of $X_t$, we have
\[
\frac{d\,\mathcal L(X_1\mid X_t=x)}{dq_1}(y)
=
\frac{f_0\!\left(\frac{x-ty}{1-t}\right)}
{\displaystyle\int_{\mathbb R^d}f_0\!\left(\frac{x-tz}{1-t}\right)q_1(dz)}.
\]
The denominator depends only on $t$ and $x$, so the Radon-Nikodym derivative is proportional to 
\[y\mapsto f_0\big(\frac{x-ty}{1-t}\big).\]
We note that the derivative is strictly positive almost everywhere, so the conditional law is equivalent to the non-Dirac law \(q_1\) and therefore has strictly positive covariance trace. Since \[U^{\pi_{\text{ind}}}=\frac{X_1-X_t^{\pi_{\text{ind}}}}{1-t},\] its conditional covariance also has strictly positive trace. By integrating over \((t,X_t^{\pi_{\text{ind}}})\), we can prove \(\mathcal L_{\pi_{\text{ind}}}(\theta_{\pi_{\text{ind}}}^\star)>0\), and the condition holds for \(q_0=\mathcal N(0,I_d)\). 

We now prove the theorem:
\vspace{-0.5em}
\begin{proof}
Following from Fact~\ref{fact:affine}, we write the affine output layer as \(v_{\eta,b}(t,x)\!=\!\bar v_\eta(t,x)+b\). Since \(\nabla_bv_{\eta,b}(t,x)=I_d\), its bias-gradient block at the population optimum is
\[
\nabla_b\ell_\pi(\theta_\pi^\star;\xi^\pi)=2\bigl(u^\pi(t,X_t^\pi)-U^\pi\bigr).
\]
Its conditional mean at \((t,X_t^\pi)\) is zero, so the law of total covariance therefore gives us that
\begin{equation}
\label{eq:bias-gradient-equals-loss}
\Tr\{\Cov_\pi(\nabla_b\ell_\pi(\theta_\pi^\star;\xi^\pi))\}=4\mathcal L_\pi(\theta_\pi^\star).
\end{equation}
For the complete gradient, \cref{eq:var-at-optimum} and the positive-semidefinite trace bound gives us that
\[
\sigma_{\text{grad},\pi_\varepsilon}^2(\theta_{\pi_\varepsilon}^\star)\leq4\mathcal L_{\pi_\varepsilon}(\theta_{\pi_\varepsilon}^\star)\operatorname*{ess\,sup}\|J_{\theta_{\pi_\varepsilon}^\star}(t,X_t^{\pi_\varepsilon})\|_{\mathrm{op}}^2.
\]
Combining this inequality with \cref{eq:fm-loss-cost-gap} proves the upper bound in \cref{eq:eot-full-gradient-bound}. For the strict lower bound, \cref{eq:bias-gradient-equals-loss} and Lemma~\ref{lem:eot-strict-small-epsilon} show that the bias-gradient block has strictly positive variance for every finite \(\varepsilon\). 

Since the total gradient-variance trace is at least the trace of this principal block, we can show that \(\sigma_{\text{grad},\pi_\varepsilon}^2(\theta_{\pi_\varepsilon}^\star)>0\). At the same time, population OT has zero gradient variance by \cref{eq:var-at-optimum,eq:fm-loss-cost-gap}. Therefore, whenever the right-hand side of \cref{eq:eot-full-gradient-bound} is smaller than the independent-coupling variance, the strict chain in \cref{eq:eot-full-gradient-strict} follows.

Finally, the independent-coupling variance is strictly positive by \cref{eq:bias-gradient-equals-loss}. The theorem's uniform condition means that there exist \(\varepsilon_J\!>\!0\) and \(M\!<\!\infty\) such that \(\{\pi_\varepsilon\!:\!0\!<\varepsilon\!\leq\!\varepsilon_J\}\) satisfies Definition~\ref{rem:bounded-jacobian} with bound \(M\). Thus, for every \(0<\varepsilon\leq\varepsilon_J\), we have
\[
\operatorname*{ess\,sup}\|J_{\theta_{\pi_\varepsilon}^\star}(t,X_t^{\pi_\varepsilon})\|_{\mathrm{op}}^2
\leq \max_{t\in[0,1],\,x\in\mathbb R^d}\|J_{\theta_{\pi_\varepsilon}^\star}(t,x)\|_{\mathrm{op}}^2
\leq M.
\]
We denote \(D_\varepsilon\coloneqq\mathbb E_{\pi_\varepsilon}[\|X_1-X_0\|_2^2]-W_2^2(q_0,q_1)\), \cref{eq:eot-cost-convergence} then gives us that \(4M D_\varepsilon\to0\) as \(\varepsilon\downarrow0\). Therefore, there exists \(\varepsilon_{\mathrm g}>0\) such that
\[
4M D_\varepsilon<\sigma_{\text{grad},\pi_{\text{ind}}}^2(\theta_{\pi_{\text{ind}}}^\star)
\qquad\text{for every }0<\varepsilon\leq\varepsilon_{\mathrm g}.
\]
By setting \(\varepsilon_1=\min\{\varepsilon_0,\varepsilon_J,\varepsilon_{\mathrm g}\}>0\), then for every \(0<\varepsilon\leq\varepsilon_1\), the right-hand side of \cref{eq:eot-full-gradient-bound} is smaller than the independent-coupling variance, which concludes \cref{eq:eot-full-gradient-strict}.
\end{proof}

\subsection{Proposition~\ref{prop:large-k-intrinsic-consistency}}
\label{app:proof-large-k-intrinsic-consistency}

\begin{proof}
We take \(\widehat q_{0,k}\) and \(\widehat q_{1,k}\) to denote the empirical distributions of the two endpoint batches used to construct \(\pi_k\). Note that conditional on these batches, the OT solution is a permutation and the selected index is uniform. 

Therefore, the selected source and target each remain uniformly distributed over their respective batches; after averaging over the batches, we have \(\pi_k\in\Pi(q_0,q_1)\). Moreover, one can derive that 
\[
\mathbb E_{\pi_k}[\|X_1-X_0\|_2^2]=\mathbb E[W_2^2(\widehat q_{0,k},\widehat q_{1,k})].
\]
To establish this, we first show that, for any \(q\in\mathcal P_2(\mathbb R^d)\) along with its empirical distribution \(\widehat q_k\), \(\mathbb E[W_2^2(\widehat q_k,q)]\) will converge to zero. 

For any \(R>0\), let \(P_R\) be the metric projection onto the closed ball \(\overline B(0,R)\), and let \(q^R=(P_R)_\#q\) and \(\widehat q_k^R=(P_R)_\#\widehat q_k\), coupling each sample with its projection gives us that
\[
\mathbb E[W_2^2(\widehat q_k,\widehat q_k^R)]\leq a_R:=\mathbb E_q[\|X-P_R(X)\|_2^2],\qquad W_2^2(q,q^R)\leq a_R.
\]
For fixed \(R>0\), the empirical measure \(\widehat q_k^R\) converges to \(q^R\) almost surely in \(W_2\). Both measures are supported on \(\overline B(0,R)\), whose diameter is \(2R\), so \(W_2^2(\widehat q_k^R,q^R)\leq4R^2\). The almost-sure convergence of \(\widehat q_k^R\) to \(q^R\) in \(W_2\) further gives
\[
W_2^2(\widehat q_k^R,q^R)\longrightarrow0
\qquad\text{almost surely}.
\]
Since both measures are supported on the radius-\(R\) ball, we also have \(0\leq W_2^2(\widehat q_k^R,q^R)\leq4R^2\). This deterministic bound is integrable and independent of \(k\), so dominated convergence gives
\[
\lim_{k\to\infty}\mathbb E[W_2^2(\widehat q_k^R,q^R)]
=
\mathbb E\!\left[\lim_{k\to\infty}W_2^2(\widehat q_k^R,q^R)\right]
=0.
\]
Next, the triangle inequality gives
\[
W_2(\widehat q_k,q)
\leq
W_2(\widehat q_k,\widehat q_k^R)
+
W_2(\widehat q_k^R,q^R)
+
W_2(q^R,q).
\]
First using \((a+b+c)^2\leq3(a^2+b^2+c^2)\) and then simply taking expectations, we have
\[
\begin{aligned}
\mathbb E[W_2^2(\widehat q_k,q)]
&\leq
3\mathbb E[W_2^2(\widehat q_k,\widehat q_k^R)]
+
3\mathbb E[W_2^2(\widehat q_k^R,q^R)]
+
3W_2^2(q^R,q)\\
&\leq
3a_R
+
3\mathbb E[W_2^2(\widehat q_k^R,q^R)]
+
3a_R\\
&=
6a_R+3\mathbb E[W_2^2(\widehat q_k^R,q^R)],
\end{aligned}
\]
where the second inequality uses the two projection bounds established above. Taking the limit superior as \(k\to\infty\), while keeping \(R\) fixed, therefore gives
\[
\begin{aligned}
\limsup_{k\to\infty}\mathbb E[W_2^2(\widehat q_k,q)]
\leq
6a_R
+
3\lim_{k\to\infty}\mathbb E[W_2^2(\widehat q_k^R,q^R)]
=6a_R.
\end{aligned}
\]
Since \(q\in\mathcal P_2(\mathbb R^d)\), \(a_R\to0\) as \(R\to\infty\), which proves our claim. Applying it to both marginals yields
\[
|W_2(\widehat q_{0,k},\widehat q_{1,k})-W_2(q_0,q_1)|\leq W_2(\widehat q_{0,k},q_0)+W_2(\widehat q_{1,k},q_1),
\]
The right-hand side converges to 0 in \(L^2\), so \(W_2(\widehat q_{0,k},\widehat q_{1,k})\!\to\! W_2(q_0,q_1)\) in \(L^2\). Its square therefore converges in \(L^1\), which proves \(\Delta_k\to0\) in \cref{eq:large-k-cost-gap}. Since \(\pi_k\) is a coupling of \(q_0\) and \(q_1\), its cost is at least \(W_2^2(q_0,q_1)\), so \(\Delta_k\geq0\).

For any coupling \(\pi\), conditional orthogonality gives us that
\[
\mathbb E_\pi[\|U^\pi\|_2^2]=\mathbb E_\pi[\|u^\pi(t,X_t^\pi)\|_2^2]+\mathbb E_\pi[\Tr\{\Cov_\pi(U^\pi\mid t,X_t^\pi)\}].
\]
As earlier shown in the proof of Proposition~\ref{prop:fm-loss-cost-gap}, the interpolation induced by \(\pi\), together with \(u^\pi\), is admissible in the Benamou--Brenier formulation. Therefore, \(\mathbb E_\pi[\|u^\pi(t,X_t^\pi)\|_2^2]\geq W_2^2(q_0,q_1)\). 

Applying the preceding kinetic-energy bound to \(\pi_k\), we have
\[
\mathbb E_{\pi_k}\!\left[
\|u^{\pi_k}(t,X_t^{\pi_k})\|_2^2
\right]
\geq W_2^2(q_0,q_1).
\]
We can then rearrange the conditional-orthogonality identity and use \(U^{\pi_k}=X_1-X_0\) to get
\[
\begin{aligned}
0
&\leq
\mathbb E_{\pi_k}\!\left[
\operatorname{Tr}\!\left\{
\operatorname{Cov}_{\pi_k}(U^{\pi_k}\mid t,X_t^{\pi_k})
\right\}
\right]\\
&=
\mathbb E_{\pi_k}\!\left[
\|U^{\pi_k}-u^{\pi_k}(t,X_t^{\pi_k})\|_2^2
\right]
&&\hfill \qquad \quad (\because \text{conditional-variance identity})\\
&=
\mathbb E_{\pi_k}\!\left[\|U^{\pi_k}\|_2^2\right]
-
\mathbb E_{\pi_k}\!\left[
\|u^{\pi_k}(t,X_t^{\pi_k})\|_2^2
\right]
&&\hfill \qquad \quad (\because \text{conditional orthogonality})\\
&\leq
\mathbb E_{\pi_k}\!\left[\|U^{\pi_k}\|_2^2\right]
-
W_2^2(q_0,q_1)
&&\hfill \qquad \quad (\because \text{mean-velocity energy lower bound})\\
&=
\mathbb E_{\pi_k}\!\left[\|X_1-X_0\|_2^2\right]
-
W_2^2(q_0,q_1)=\Delta_k.
\end{aligned}
\]
This proves the upper bound in \cref{eq:large-k-intrinsic-envelope}. Since we have already established that \(\Delta_k\to0\), the squeeze theorem further gives
\[
\mathbb E_{\pi_k}\!\left[
\operatorname{Tr}\!\left\{
\operatorname{Cov}_{\pi_k}(U^{\pi_k}\mid t,X_t^{\pi_k})
\right\}
\right]
\longrightarrow0
\qquad\text{as }k\to\infty.
\]
So, we prove \(\Delta_k\to0\).
\end{proof}

\subsection{Theorem~\ref{thm:large-k-eventual-gradient-reduction}}
\label{app:proof-large-k-eventual-gradient-reduction}

Before moving into the proof of Theorem~\ref{thm:large-k-eventual-gradient-reduction}. We first provide the following supporting lemma, which shows that every finite-batch coupling has zero target uncertainty exactly when independent coupling does.

\begin{lemma}
\label{lem:large-k-finite-nondegeneracy}
For every finite \(k\geq1\), \(\pi_k\) and \(\pi_{\emph{ind}}\) are mutually absolutely continuous, so it has
\[
\mathbb E_{\pi_k}[\Tr\{\Cov_{\pi_k}(U^{\pi_k}\mid t,X_t^{\pi_k})\}]=0\quad\Longleftrightarrow\quad\mathbb E_{\pi_{\emph{ind}}}[\Tr\{\Cov_{\pi_{\emph{ind}}}(U^{\pi_{\emph{ind}}}\mid t,X_t^{\pi_{\emph{ind}}})\}]=0.
\]
If \(q_0=f_0\mathcal L^d\) with \(f_0>0\) Lebesgue-almost everywhere and \(q_1\) is not a Dirac mass, both quantities are strictly positive.
\end{lemma}

\begin{proof}
We write the two endpoint batches as \((X_1,\ldots,X_k)\) and \((Y_1,\ldots,Y_k)\), and let \(I\) be uniform on \(\{1,\ldots,k\}\), independently of both batches, along with \(\tau\) be the permutation returned by a fixed measurable tie-breaking rule for empirical OT. Then, \((X_I,Y_{\tau(I)})\sim\pi_k\). 

For any Borel set \(E\subset\mathbb R^d\times\mathbb R^d\), if \((q_0\otimes q_1)(E)=0\), all finitely many candidate pairs \((X_i,Y_j)\) avoid \(E\) almost surely, so \(\pi_k(E)=0\). Conversely, we denote \(E_x=\{y:(x,y)\in E\}\), and the event that every candidate target paired with \(X_I\) lies in \(E_{X_I}\) forces the selected pair to lie in \(E\). Therefore, we get 
\[
\pi_k(E)\geq\mathbb P((X_I,Y_j)\in E\text{ for every }j)=\int q_1(E_x)^k\,q_0(\mathrm dx).
\]
If \((q_0\otimes q_1)(E)>0\), then \(q_1(E_x)>0\) on a set of positive \(q_0\)-measure, so the final integral is strictly positive. Thus, one have \(\pi_k\sim\pi_{\text{ind}}\).

Next, as we discussed before, the map \((t,x_0,x_1)\mapsto(t,(1-t)x_0+tx_1,x_1-x_0)\) is a Borel bijection with inverse \(x_0=x-tu\), \(x_1=x+(1-t)u\). Therefore, equivalent endpoint couplings induce equivalent laws of \((t,X_t,U)\).

As the covariance trace is nonnegative, its expectation is zero exactly when \(\Cov(U\mid t,X_t)=0\) almost surely, or equivalently when \(U\) is almost surely a measurable function of \((t,X_t)\). This graph property is preserved under equivalent laws, proving the stated equivalence.

Finally, using the stated density conditions and fix \(t\in(0,1)\). Under independent coupling, for almost every \(x\) under the law of \(X_t\), the conditional law of \(X_1\) given \(X_t=x\) has Radon--Nikodym derivative with respect to \(q_1\) proportional to
\[
y\longmapsto f_0\!\left(\frac{x-ty}{1-t}\right).
\]

We claim that this derivative is strictly positive for \(q_1\)-almost every \(y\). To prove it, we set \(t\in(0,1)\) and let
\[
N:=\{z\in\mathbb R^d:f_0(z)=0\}.
\]
The assumption \(f_0>0\) Lebesgue-almost everywhere gives \(\mathcal L^d(N)=0\). For every fixed \(y\in\mathbb R^d\), the change of variables \(x=(1-t)z+ty\) gives
\[
\begin{aligned}
\int_{\mathbb R^d}
\mathbf 1_N\!\left(\frac{x-ty}{1-t}\right)\,dx=
(1-t)^d\int_{\mathbb R^d}\mathbf 1_N(z)\,dz=(1-t)^d\mathcal L^d(N)=0.
\end{aligned}
\]
Thus, for each fixed \(y\), the numerator \(f_0((x-ty)/(1-t))\) is positive for Lebesgue-almost every \(x\). We now interchange the order of integration. Since the indicator is nonnegative, Tonelli's theorem gives us
\[
\begin{aligned}
\int_{\mathbb R^d}
\left[
\int_{\mathbb R^d}
\mathbf 1_N\!\left(\frac{x-ty}{1-t}\right)\,q_1(dy)
\right]dx=
\int_{\mathbb R^d}
\left[
\int_{\mathbb R^d}
\mathbf 1_N\!\left(\frac{x-ty}{1-t}\right)\,dx
\right]q_1(dy)=0.
\end{aligned}
\]
The inner integral on the left is nonnegative, so it must vanish for Lebesgue-almost every \(x\). Equivalently, we have
\[
q_1\!\left(
\left\{y:
f_0\!\left(\frac{x-ty}{1-t}\right)=0
\right\}
\right)=0
\qquad
\text{for Lebesgue-almost every }x.
\]
Therefore, for Lebesgue-almost every \(x\), we have
\[
f_0\!\left(\frac{x-ty}{1-t}\right)>0
\qquad
\text{for }q_1\text{-almost every }y.
\]
We still remain for considering the normalizing factor in the conditional density. Under independent coupling, \(X_t\) has Lebesgue density
\[
\rho_t(x)
=
(1-t)^{-d}
\int_{\mathbb R^d}
f_0\!\left(\frac{x-ty}{1-t}\right)\,q_1(dy).
\]
The preceding positivity gives \(\rho_t(x)>0\) for Lebesgue-almost every \(x\). Moreover, we can show that
\[
\begin{aligned}
\int_{\mathbb R^d}\rho_t(x)\,dx
&=
\int_{\mathbb R^d}
\left[
(1-t)^{-d}
\int_{\mathbb R^d}
f_0\!\left(\frac{x-ty}{1-t}\right)\,dx
\right]q_1(dy)=
\int_{\mathbb R^d}
\left[
\int_{\mathbb R^d}f_0(z)\,dz
\right]q_1(dy)=1,
\end{aligned}
\]
so \(\rho_t(x)<\infty\) for Lebesgue-almost every \(x\). Since the law of \(X_t\) is absolutely continuous with respect to Lebesgue measure, all these properties also hold for almost every \(x\) under the law of \(X_t\). For such \(x\), the conditional Radon--Nikodym derivative therefore satisfies
\[
\frac{d\mathcal L(X_1\mid X_t=x)}{dq_1}(y)
=
\frac{
(1-t)^{-d}f_0\!\left(\frac{x-ty}{1-t}\right)
}{
\rho_t(x)
}
>0
\qquad
\text{for }q_1\text{-almost every }y.
\]
Thus, for almost every \(x\) under the law of \(X_t\), the conditional law of \(X_1\) given \(X_t=x\) is equivalent to \(q_1\), which concludes our claim. Furthermore, the conditional law is equivalent to the non-Dirac distribution \(q_1\) and has strictly positive covariance trace. 

Since \(U\!=\!(X_1-X_t)/(1-t)\), the conditional covariance of \(U\) also has strictly positive trace. Integrating first over \(X_t\) and then over \(t\in(0,1)\) proves positivity for independent coupling and therefore, by the preceding equivalence, for every \(k<\infty\).
\end{proof}
\vspace{-0.5em}
We now formally prove Theorem~\ref{thm:large-k-eventual-gradient-reduction}:
\vspace{-0.5em}
\begin{proof}[Proof of Theorem~\ref{thm:large-k-eventual-gradient-reduction}]
At the coupling-specific population optimum, \cref{eq:var-at-optimum} and the positive-semidefinite trace bound used in Proposition~\ref{prop:gradient-variance-sandwich} give
\[
0\leq\sigma_{\text{grad},\pi_k}^2(\theta_{\pi_k}^\star)\leq4\operatorname*{ess\,sup}\|J_{\theta_{\pi_k}^\star}(t,X_t^{\pi_k})\|_{\mathrm{op}}^2\,\mathbb E_{\pi_k}[\Tr\{\Cov_{\pi_k}(U^{\pi_k}\mid t,X_t^{\pi_k})\}].
\]
Applying \cref{eq:large-k-intrinsic-envelope} proves the upper bound in \cref{eq:large-k-full-convergence}. The uniform condition means that there exist \(K_J\in\mathbb N\) and \(M<\infty\) such that, for every \(k\geq K_J\), the following essential optima control exists:
\[
\operatorname*{ess\,sup}\|J_{\theta_{\pi_k}^\star}(t,X_t^{\pi_k})\|_{\mathrm{op}}^2
\leq\max_{t\in[0,1],\,x\in\mathbb R^d}\|J_{\theta_{\pi_k}^\star}(t,x)\|_{\mathrm{op}}^2
\leq M.
\]
Since \(\Delta_k\to0\), the right-hand side of \cref{eq:large-k-full-convergence} is at most \(4M\Delta_k\to0\) for \(k\geq K_J\). The squeeze theorem therefore gives \(\sigma_{\text{grad},\pi_k}^2(\theta_{\pi_k}^\star)\to0\), so \cref{eq:large-k-full-eventual-reduction} holds for all sufficiently large \(k\) whenever the fixed baseline variance is positive.

For completeness, we also connect the per-sample statement to an update averaged over all matched pairs. 

We denote \(\{G_{k,1},\ldots,G_{k,k}\}\) as the gradients generated at \(\theta_{\pi_k}^\star\) by the \(k\) matched pairs in one empirical OT batch, with an independent uniform time sampled for each pair, and let \(I\) be an independent uniform index. Then, \(G_{k,I}\) has the same law as \(G_{\pi_k}(\theta_{\pi_k}^\star)\). The proof of Proposition~\ref{prop:variance-at-optimum} gives \(\mathbb E[G_{k,I}]=0\), so the averaged update also has mean zero. Jensen's inequality further gives us the following bound:
\begin{equation}\label{eq:jen}
    \Tr\!\left\{\Cov\!\left(k^{-1}\textstyle\sum_{i=1}^kG_{k,i}\right)\right\}=\mathbb E\!\left[\left\|k^{-1}\textstyle\sum_{i=1}^kG_{k,i}\right\|_2^2\right]
    \leq\mathbb E[\|G_{k,I}\|_2^2]
     = \sigma_{\text{grad},\pi_k}^2(\theta_{\pi_k}^\star).
\end{equation}
Thus, the averaged matched-batch update also has variance converging to zero under the same vanishing-product condition.
\end{proof}

\subsection{Proposition~\ref{prop:entropic-large-k-fixed}}\label{app:proof-entropic-large-k}

To ensure rigor, we first establish several lemmas required to pass from convergence of entropic couplings to convergence of their target uncertainty.

\begin{lemma}
\label{lem:entropic-uncertainty-usc}
If \(\pi_n,\pi\in\Pi(q_0,q_1)\), with \(q_0,q_1\in\mathcal P_2(\mathbb R^d)\), and \(\pi_n\to\pi\) in \(W_2\), then we have
\begin{equation}\label{eq:entropic-uncertainty-usc}
\limsup_{n\to\infty}\mathsf R(\pi_n)\leq\mathsf R(\pi).
\end{equation}
\end{lemma}

\begin{proof}
We write \(P_n\) and \(P\) as the laws of \(((t,X_t),U)\) induced by \(\pi_n\) and \(\pi\), respectively. Then, we can couple \(\pi_n\) and \(\pi\) optimally and use the same independent \(t\) for both couplings. 

Since \(X_t=(1-t)X_0+tX_1\) and \(U=X_1-X_0\) are linear in the endpoints, \(P_n\to P\) in \(W_2\). The conditional-mean characterization of squared-error regression gives
\[
\mathsf R(\pi)=\inf_f\mathbb E_P[\|U-f(t,X_t)\|_2^2],
\]
where the infimum may be restricted to bounded continuous \(f\), which are dense in the relevant \(L^2\) space. For every such fixed \(f\), the integrand is continuous with quadratic growth, so its expectation converges under \(W_2\) convergence. Therefore, we end up with
\[
\limsup_{n\to\infty}\mathsf R(\pi_n)\leq\mathbb E_P[\|U-f(t,X_t)\|_2^2].
\]
Taking the infimum over \(f\) proves \cref{eq:entropic-uncertainty-usc}.
\end{proof}

For the convergence of empirical entropic plans, we first try to conclude the following local value bound. Before that, for notation simplicity, given probability measure $\mu$ and $\nu$, we define
\[
\mathsf S_\varepsilon(\mu,\nu):=\inf_{\rho\in\Pi(\mu,\nu)}\left\{\frac{1}{2d}\mathbb E_\rho[\|X_1-X_0\|_2^2]+\varepsilon\operatorname{KL}(\rho\|\mu\otimes\nu)\right\}.
\]
\begin{lemma}\label{lem:entropic-shadow-bound}
Let \(\widehat q_{0,k}\) and \(\widehat q_{1,k}\) be the two empirical endpoint distributions, set \(a_k=W_2(\widehat q_{0,k},q_0)+W_2(\widehat q_{1,k},q_1)\), and let \(C_{\emph{ind}}=\mathbb E_{\pi_{\emph{ind}}}[\|X_1-X_0\|_2^2]\). Then, for every \(\varepsilon>0\), we have
\begin{equation}\label{eq:entropic-shadow-bound}
\mathsf S_\varepsilon(\widehat q_{0,k},\widehat q_{1,k})\leq\mathsf S_\varepsilon(q_0,q_1)+\tfrac{1}{2d}(2\sqrt{C_{\emph{ind}}}\,a_k+a_k^2).
\end{equation}
\end{lemma}

\begin{proof}
Let \(\rho_\varepsilon\) minimize \(\mathsf S_\varepsilon(q_0,q_1)\), and transport each of its marginals to the corresponding empirical marginal through optimal \(W_2\) kernels. Applying the product of these kernels to \(\rho_\varepsilon\) produces a coupling \(\rho_\varepsilon'\in\Pi(\widehat q_{0,k},\widehat q_{1,k})\). 

We claim that \(\operatorname{KL}(\rho_\varepsilon'\|\widehat q_{0,k}\otimes\widehat q_{1,k})\leq\operatorname{KL}(\rho_\varepsilon\|q_0\otimes q_1)\). To prove it, we fix the sampled endpoint batches and let \(K_i(x_i,d\widehat x_i)\) be the transport kernel from \(q_i\) to \(\widehat q_{i,k}\), for \(i\in\{0,1\}\). Thus, for every Borel set \(A_i\subseteq\mathbb R^d\), we have
\(
\int K_i(x_i,A_i)\,q_i(dx_i)=\widehat q_{i,k}(A_i)
\).

We denote \(x=(x_0,x_1)\) and \(\widehat x=(\widehat x_0,\widehat x_1)\), and the product kernel is
\(K_0(x_0,d\widehat x_0)K_1(x_1,d\widehat x_1)\). For any probability measure \(\mu\) on the endpoint space, the law obtained by applying this kernel to \(\mu\) assigns to each Borel set \(A\) with \(\int\left(\int_A K_0(x_0,d\widehat x_0)K_1(x_1,d\widehat x_1)\right)\mu(dx)\). So,
we then have
\[
\rho_\varepsilon'(A)
=
\int
\left(
\int_A K_0(x_0,d\widehat x_0)K_1(x_1,d\widehat x_1)
\right)\rho_\varepsilon(dx).
\]
Applying the same kernel to the independent reference measure gives, for every Borel rectangle \(A_0\times A_1\), we can do the following derivation:
\[
\begin{aligned}
\int
\left(
\int_{A_0\times A_1}
K_0(x_0,d\widehat x_0)K_1(x_1,d\widehat x_1)
\right)&(q_0\otimes q_1)(dx)
\!=\!
\int
K_0(x_0,A_0)K_1(x_1,A_1)\,
q_0(dx_0)q_1(dx_1)\\
&=
\left(\int K_0(x_0,A_0)\,q_0(dx_0)\right)
\left(\int K_1(x_1,A_1)\,q_1(dx_1)\right)\\
&=
\widehat q_{0,k}(A_0)\widehat q_{1,k}(A_1)\\
&=
(\widehat q_{0,k}\otimes\widehat q_{1,k})(A_0\times A_1).
\end{aligned}
\]
Since these rectangles determine the probability measure, for every Borel set \(A\), we have
\[
\int
\left(
\int_A K_0(x_0,d\widehat x_0)K_1(x_1,d\widehat x_1)
\right)(q_0\otimes q_1)(dx)
=
(\widehat q_{0,k}\otimes\widehat q_{1,k})(A).
\]
We next derive the entropy inequality directly. Since independent coupling is a finite-objective competitor and \(\varepsilon>0\), optimality gives
\(\operatorname{KL}(\rho_\varepsilon\mid q_0\otimes q_1)<\infty\). Indeed, since \(\rho_\varepsilon\ll q_0\otimes q_1\), the density ratio
\(\tfrac{d\rho_\varepsilon}{d(q_0\otimes q_1)}(x)\) is then well defined.

Now, we consider the joint reference law \((q_0\otimes q_1)(dx) K_0(x_0,d\widehat x_0)K_1(x_1,d\widehat x_1)\). Under this law, the original endpoint pair \(X\) has law \(q_0\otimes q_1\), while the transformed endpoint pair \(\widehat X\) has law \(\widehat q_{0,k}\otimes\widehat q_{1,k}\). We note that all the expectations below are taken under this joint reference law.

For any Borel set \(A\) in the transformed endpoint space, the definition of \(\rho_\varepsilon\) and the conditional-expectation identity give
\[
\begin{aligned}
\rho_\varepsilon'(A)
&=
\int
\left(
\int_A K_0(x_0,d\widehat x_0)K_1(x_1,d\widehat x_1)
\right)\rho_\varepsilon(dx)\\
&=
\int
\frac{d\rho_\varepsilon}{d(q_0\otimes q_1)}(x)
\left(
\int_A K_0(x_0,d\widehat x_0)K_1(x_1,d\widehat x_1)
\right)(q_0\otimes q_1)(dx)\\
&=
\mathbb E\!\left[
\frac{d\rho_\varepsilon}{d(q_0\otimes q_1)}(X)
\mathbf 1_{\{\widehat X\in A\}}
\right]\\
&=
\mathbb E\!\left[
\mathbb E\!\left[
\left.
\frac{d\rho_\varepsilon}{d(q_0\otimes q_1)}(X)
\right|\widehat X
\right]
\mathbf 1_{\{\widehat X\in A\}}
\right]\\
&=
\int_A
\mathbb E\!\left[
\left.
\frac{d\rho_\varepsilon}{d(q_0\otimes q_1)}(X)
\right|\widehat X=\widehat x
\right]
(\widehat q_{0,k}\otimes\widehat q_{1,k})(d\widehat x).
\end{aligned}
\]
Therefore, the transformed density ratio is
\[
\frac{d\rho_\varepsilon'}
{d(\widehat q_{0,k}\otimes\widehat q_{1,k})}(\widehat x)
=
\mathbb E\!\left[
\left.
\frac{d\rho_\varepsilon}{d(q_0\otimes q_1)}(X)
\right|\widehat X=\widehat x
\right]
\]
for \((\widehat q_{0,k}\otimes\widehat q_{1,k})\)-almost every \(\widehat x\).

Finally, the function \(s\!\mapsto\!s\log s\) (with \(0\log 0\!=\!0\)) is convex on \([0,\infty)\). Conditional Jensen's inequality therefore gives
\[
\begin{aligned}
&\mathbb E\!\left[
\left.
\frac{d\rho_\varepsilon}{d(q_0\otimes q_1)}(X)
\right|\widehat X
\right]
\log\!\left(
\mathbb E\!\left[
\left.
\frac{d\rho_\varepsilon}{d(q_0\otimes q_1)}(X)
\right|\widehat X
\right]
\right)\\
&\qquad\leq
\mathbb E\!\left[
\left.
\frac{d\rho_\varepsilon}{d(q_0\otimes q_1)}(X)
\log\!\left(
\frac{d\rho_\varepsilon}{d(q_0\otimes q_1)}(X)
\right)
\right|\widehat X
\right].
\end{aligned}
\]
Taking expectations and using the two marginal laws of the joint reference law, we get
\[
\begin{aligned}
\operatorname{KL}
\bigl(\rho_\varepsilon'\mid
\widehat q_{0,k}\otimes\widehat q_{1,k}\bigr)
&=\int
\frac{d\rho_\varepsilon'}
{d(\widehat q_{0,k}\otimes\widehat q_{1,k})}(\widehat x)
\log\!\left(
\frac{d\rho_\varepsilon'}
{d(\widehat q_{0,k}\otimes\widehat q_{1,k})}(\widehat x)
\right)
(\widehat q_{0,k}\otimes\widehat q_{1,k})(d\widehat x)\\
&=
\mathbb E\!\left[
\mathbb E\!\left[
\left.
\frac{d\rho_\varepsilon}{d(q_0\otimes q_1)}(X)
\right|\widehat X
\right]
\log\!\left(
\mathbb E\!\left[
\left.
\frac{d\rho_\varepsilon}{d(q_0\otimes q_1)}(X)
\right|\widehat X
\right]
\right)
\right]\\
&\leq
\mathbb E\!\left[
\mathbb E\!\left[
\left.
\frac{d\rho_\varepsilon}{d(q_0\otimes q_1)}(X)
\log\!\left(
\frac{d\rho_\varepsilon}{d(q_0\otimes q_1)}(X)
\right)
\right|\widehat X
\right]
\right]\\
&=
\mathbb E\!\left[
\frac{d\rho_\varepsilon}{d(q_0\otimes q_1)}(X)
\log\!\left(
\frac{d\rho_\varepsilon}{d(q_0\otimes q_1)}(X)
\right)
\right]\\
&=
\int
\frac{d\rho_\varepsilon}{d(q_0\otimes q_1)}(x)
\log\!\left(
\frac{d\rho_\varepsilon}{d(q_0\otimes q_1)}(x)
\right)(q_0\otimes q_1)(dx)\\
&=
\operatorname{KL}(\rho_\varepsilon\mid q_0\otimes q_1).
\end{aligned}
\]
So our claim is proved, and by using Minkowski's inequality, we have
\[
\sqrt{\mathbb E_{\rho_\varepsilon'}[\|X_1-X_0\|_2^2]}\leq\sqrt{\mathbb E_{\rho_\varepsilon}[\|X_1-X_0\|_2^2]}+a_k.
\]
Since independent coupling is feasible and has zero relative entropy, the raw transport cost of \(\rho_\varepsilon\) is at most \(C_{\text{ind}}\). Substituting these bounds into the empirical objective, we can then conclude \cref{eq:entropic-shadow-bound}.
\end{proof}

\begin{lemma}
\label{lem:entropic-index-measure-reduction}
Considering two sampled endpoint batches \((X_1,\!\ldots\!,X_k)\) and \((Y_1,\!\ldots\!,Y_k)\), then for any uniform-marginal matrix \(P\!=(p_{ij})_{i,j=1}^k\!\in\!\mathbb R_+^{k\times k}\) satisfying \(P\mathbf1\!=P^\top\mathbf1\!=\mathbf1/k\), we write \(\Gamma_P=\sum_{i,j=1}^k p_{ij}\,\delta_{(X_i,Y_j)}\in\Pi(\widehat q_{0,k},\widehat q_{1,k})\) as its pushforward to endpoint space. Then, the index-level problem in \cref{eq:eot} and the measure-level problem defining \(\mathsf S_\varepsilon(\widehat q_{0,k},\widehat q_{1,k})\) have the same minimum. 
\end{lemma}
\vspace{-1em}
\begin{proof}
Let \(\{x^1,\ldots,x^A\}\) and \(\{y^1,\ldots,y^B\}\) be the distinct values in the two batches. We write
\[
I_a:=\{i:X_i=x^a\},\qquad
J_b:=\{j:Y_j=y^b\},\qquad
m_a:=|I_a|,\qquad n_b:=|J_b|.
\]
For a feasible matrix \(P\), we define its location-level block masses by
\[
    \overline P_{ab}
    :=\sum_{i\in I_a}\sum_{j\in J_b}p_{ij}.
\]
Then, \(\overline P\) is a coupling between the location masses \((m_a/k)_{a=1}^A\) and \((n_b/k)_{b=1}^B\), and
\[
    \Gamma_P=\sum_{a=1}^A\sum_{b=1}^B
    \overline P_{ab}\,\delta_{(x^a,y^b)}.
\]
Whenever \(\overline P_{ab}>0\), we set \(r_{ij}^{ab}:=p_{ij}/\overline P_{ab}\) on \(I_a\times J_b\); when \(\overline P_{ab}=0\), one can choose \(r^{ab}\) arbitrarily. The relative-entropy chain rule gives us that
\begin{equation}
\operatorname{KL}\!\left(P\middle\|\mathbf1\mathbf1^\top/k^2\right)=
\operatorname{KL}\!\left(
    \overline P
    \middle\|
    \left(m_a n_b/k^2\right)_{a,b}
\right)
+
\sum_{a,b}\overline P_{ab}
\operatorname{KL}\!\left(
    r^{ab}
    \middle\|
    \operatorname{Unif}(I_a\times J_b)
\right).
\label{eq:entropic-index-measure-kl-chain}
\end{equation}
So, clearly, the first term in \cref{eq:entropic-index-measure-kl-chain} is exactly \(\operatorname{KL}(\Gamma_P\|\widehat q_{0,k}\otimes \widehat q_{1,k})\), while the second term is nonnegative. Moreover, the quadratic cost is constant on every block \(I_a\times J_b\), so we have
\[
    \langle\mathbf C,P\rangle
    =\frac1{2d}\int\|x-y\|_2^2\,\Gamma_P(\mathrm dx,\mathrm dy).
\]

Conversely, every location-level coupling \(\overline P\) has the feasible block-uniform lift
\[
    p_{ij}:=\frac{\overline P_{ab}}{m_a n_b},
    \qquad i\in I_a,\ j\in J_b.
\]
Indeed, for \(i\in I_a\), we have
\[
\sum_jp_{ij}
=\frac1{m_a}\sum_b\overline P_{ab}
=\frac1k,
\]
and the column constraints follow identically. For this lift every conditional entropy term in \cref{eq:entropic-index-measure-kl-chain} vanishes. Thus, minimizing the index-level objective first over lifts of a fixed endpoint coupling, and then over endpoint couplings, gives exactly the measure-level entropic problem. The block-uniform lift is the unique entropy-minimizing lift of each endpoint coupling. 

Since relative entropy is strictly convex on the finite location-level coupling polytope, the measure-level optimizer is unique, and its block-uniform lift is the unique index-level optimizer.
\end{proof}
\vspace{-0.5em}
With these lemmas, we can finally now prove Proposition~\ref{prop:entropic-large-k-fixed}:
\vspace{-0.5em}
\begin{proof}[Proof of Proposition~\ref{prop:entropic-large-k-fixed}]
We denote \(\widehat q_{0,k}\) and \(\widehat q_{1,k}\) as the empirical marginals and set \(\widehat\gamma_{k,\varepsilon}\!=\!\Gamma_{\widehat\pi_{k,\varepsilon}}\). By Lemma~\ref{lem:entropic-index-measure-reduction}, conditional on the endpoint batches, \(\widehat\gamma_{k,\varepsilon}\) is exactly the unique minimizer of \(\mathsf S_\varepsilon(\widehat q_{0,k},\widehat q_{1,k})\). The empirical marginals converge almost surely to \(q_0\) and \(q_1\) in \(W_2\).

Fix \(\varepsilon>0\), we can work on an event of probability on which both empirical \(W_2\) convergences hold. We claim that the sequence \(\{\widehat\gamma_{k,\varepsilon}\}_k\) is tight as its two marginals are tight. 

Let \(\gamma\) be the weak limit of an arbitrary convergent subsequence \(\{\widehat\gamma_{k_j,\varepsilon}\}_{j\geq1}\), its marginals are \(q_0\) and \(q_1\). Joint lower semicontinuity of the quadratic cost and relative entropy along this subsequence gives us
\begin{equation}
\frac1{2d}\mathbb E_\gamma\|X_1-X_0\|_2^2 +\varepsilon\operatorname{KL}(\gamma\|q_0\otimes q_1) \nonumber\leq \liminf_{j\to\infty} \mathsf S_\varepsilon(\widehat q_{0,k_j},\widehat q_{1,k_j}).
\label{eq:entropic-empirical-liminf}
\end{equation}
For the entropy term, the joint lower semicontinuity follows directly from its variational representation
\[
\operatorname{KL}(\mu\|\nu) =\sup_{f\in C_b} \left\{\int f\,\mathrm d\mu- \log\int e^f\,\mathrm d\nu\right\},
\]
which is because \(\widehat q_{0,k}\otimes\widehat q_{1,k}\!\Rightarrow\! q_0\otimes q_1\). On the other hand, Lemma~\ref{lem:entropic-shadow-bound} and \(a_k\to0\) (almost surely) give
\begin{equation}
\label{eq:entropic-empirical-limsup}
\limsup_{k\to\infty} \mathsf S_\varepsilon(\widehat q_{0,k},\widehat q_{1,k}) \leq\mathsf S_\varepsilon(q_0,q_1).
\end{equation}
The full-sequence bound in \cref{eq:entropic-empirical-limsup} also holds along \(\{k_j\}\). Therefore, we have
\[
\begin{aligned}
\frac1{2d}\mathbb E_\gamma\|X_1\!-\!X_0\|_2^2 \!+\!\varepsilon\operatorname{KL}(\gamma\|q_0\!\otimes\!q_1) \!\leq\! \liminf_{j\to\infty}\mathsf S_\varepsilon(\widehat q_{0,k_j},\widehat q_{1,k_j}) \!\leq\! \limsup_{j\to\infty}\mathsf S_\varepsilon(\widehat q_{0,k_j},\widehat q_{1,k_j}) \!\leq\! \mathsf S_\varepsilon(q_0,q_1),
\end{aligned}
\]
so \(\gamma\) minimizes the population objective in \cref{eq:population-eot}. We note the objective has a unique minimizer, i.e., independent coupling has finite objective, and relative entropy is strictly convex on the finite-objective domain, so \(\gamma=\pi_\varepsilon\). Since every weak cluster point is the same, we end up with
\[
    \widehat\gamma_{k,\varepsilon}\Rightarrow\pi_\varepsilon
    \qquad\text{almost surely}.
\]
Furthermore, we have
\[
\int(\|x\|_2^2+\|y\|_2^2)\, \widehat\gamma_{k,\varepsilon}(\mathrm dx,\mathrm dy) = \int\|x\|_2^2\,\widehat q_{0,k}(\mathrm dx) + \int\|y\|_2^2\,\widehat q_{1,k}(\mathrm dy),
\]
which converges almost surely to the corresponding second moment under \(q_0\otimes q_1\). Thus, the almost-sure convergence is in \(W_2\). Conditional on the endpoint batches, \(\widehat\gamma_{k,\varepsilon}\) is precisely the law of the selected endpoint pair. Therefore, for every bounded continuous \(g\), we have
\[
\int g\,\mathrm d\pi_{k,\varepsilon} = \mathbb E\!\left[\int g\,\mathrm d\widehat\gamma_{k,\varepsilon} \right] \longrightarrow \int g\,\mathrm d\pi_\varepsilon
\]
by bounded convergence, so \(\pi_{k,\varepsilon}\Rightarrow\pi_\varepsilon\). Every expected coupling \(\pi_{k,\varepsilon}\) has marginals \(q_0,q_1\), and therefore has the same product-space second moment as \(\pi_\varepsilon\). It follows that \(\pi_{k,\varepsilon}\rightarrow\pi_\varepsilon\) in \(W_2\).

For any coupling \(\pi\), the conditional orthogonality and Benamou--Brenier argument in the proof of Proposition~\ref{prop:fm-loss-cost-gap} gives
\begin{equation}
\label{eq:entropic-general-uncertainty-envelope}
0\leq\mathsf R(\pi)
\leq
\mathbb E_\pi[\|X_1-X_0\|_2^2]-W_2^2(q_0,q_1).
\end{equation}
Applying Lemma~\ref{lem:entropic-uncertainty-usc} to the established \(W_2\) convergence yields \(\limsup_{k\to\infty}\mathsf R(\pi_{k,\varepsilon}) \leq\mathsf R(\pi_\varepsilon)\),
and applying \cref{eq:entropic-general-uncertainty-envelope} to \(\pi_\varepsilon\) proves \cref{eq:entropic-large-k-fixed}.

We still need to prove the finite-batch equivalence. Let \(E\subseteq\mathbb R^d\times\mathbb R^d\) be Borel,  we write \(E_x=\{y:(x,y)\in E\}\). If \(\pi_{\text{ind}}(E)=0\), every one of the finitely many candidate source-target pairs avoids \(E\) almost surely, so \(\pi_{k,\varepsilon}(E)=0\). Conversely, if \(\pi_{\text{ind}}(E)>0\), then
\[
\mathbb P\!\left(Y_j\in E_{X_I}\text{ for every }j\right) =\int q_1(E_x)^k\,q_0(\mathrm dx)>0,
\]
and this event forces the selected pair to belong to \(E\), so \(\pi_{k,\varepsilon}\!\sim\!\pi_{\text{ind}}\). The Borel bijection
\[
(t,x_0,x_1)\mapsto
\bigl(t,(1-t)x_0+tx_1,x_1-x_0\bigr)
\]
preserves equivalence of the induced laws. Finally, \(\mathsf R(\pi)=0\) exactly when \(U^\pi\) is almost surely a measurable function of \((t,X_t^\pi)\), and this graph property is invariant under equivalent laws. We now prove the proposition.
\end{proof}

\subsection{Theorem~\ref{thm:entropic-large-k-fixed-gradient}}
We first establish Lemma~\ref{lem:entropic-fixed-batch-zero-temperature} before prove Theorem~\ref{thm:entropic-large-k-fixed-gradient}.
\begin{lemma}
\label{lem:entropic-fixed-batch-zero-temperature}
Suppose \(q_0,q_1\in\mathcal P_2(\mathbb R^d)\), and we choose any \(k\in\mathbb N\). Conditional on two sampled endpoint batches, let \(\mathcal O_k\) be the face of uniform-marginal matrices minimizing the empirical quadratic cost in \cref{eq:eot}, and we define
\begin{equation}
\label{eq:entropic-fixed-batch-me-plan}
\widehat\pi_{k,0}^{\mathrm{ME}}
:=
\operatorname*{arg\,min}_{P\in\mathcal O_k}
\operatorname{KL}\!\left(
P\middle\|\mathbf1\mathbf1^\top/k^2
\right).
\end{equation}
Let \(\pi_{k,0}^{\mathrm{ME}}\) be the averaged selected-pair coupling induced by this plan. Then, it gives that
\begin{equation}
\label{eq:entropic-fixed-batch-zero-temperature}
\left\|\pi_{k,\varepsilon}-\pi_{k,0}^{\mathrm{ME}}\right\|_{\mathrm{TV}}
\longrightarrow0,
\qquad
\mathsf R(\pi_{k,\varepsilon})
\longrightarrow
\mathsf R(\pi_{k,0}^{\mathrm{ME}})
\end{equation}
as \(\varepsilon\downarrow0\). Moreover, we have
\begin{equation}
\label{eq:entropic-fixed-batch-equivalence}
\pi_{k,0}^{\mathrm{ME}}\sim\pi_{\emph{ind}},
\qquad
\mathsf R(\pi_{k,0}^{\mathrm{ME}})=0
\quad\Longleftrightarrow\quad
\mathsf R(\pi_{\emph{ind}})=0.
\end{equation}
\end{lemma}

\begin{proof}
Let
\[
\mathcal T_k
:=
\left\{P\in\mathbb R_+^{k\times k}:
P\mathbf1=\frac1k\mathbf1,
\ P^\top\mathbf1=\frac1k\mathbf1
\right\}
\]
be the empirical transport poly-tope. Conditional on the sampled batches, we write
\[
C(P):=\langle\mathbf C,P\rangle, \qquad H(P):=\operatorname{KL}\!\left( P\middle\|\mathbf1\mathbf1^\top/k^2 \right), \qquad C_\star:=\min_{P\in\mathcal T_k}C(P).
\]
We note that \(\mathcal T_k\) is both compact and convex, and the optimal face \(\mathcal O_k=\{P\in\mathcal T_k:C(P)=C_\star\}\) is nonempty, compact, and convex. Because the reference matrix has strictly positive entries, \(H\) is continuous and strictly convex on \(\mathcal T_k\). Therefore, \cref{eq:entropic-fixed-batch-me-plan} has a unique solution.

We denote \(\widehat\pi_{k,\varepsilon}\) as the unique empirical entropic optimizer, and its optimality against \(\widehat\pi_{k,0}^{\mathrm{ME}}\) gives
\begin{equation*}
C(\widehat\pi_{k,\varepsilon}) +\varepsilon H(\widehat\pi_{k,\varepsilon}) \leq C_\star +\varepsilon H(\widehat\pi_{k,0}^{\mathrm{ME}}).
\end{equation*}
Since \(C(\widehat\pi_{k,\varepsilon})\geq C_\star\) and \(H\geq0\), we can show that
\[
0\leq C(\widehat\pi_{k,\varepsilon})-C_\star \leq \varepsilon H(\widehat\pi_{k,0}^{\mathrm{ME}}), \qquad H(\widehat\pi_{k,\varepsilon}) \leq H(\widehat\pi_{k,0}^{\mathrm{ME}}).
\] 
Therefore, every cluster point of \(\widehat\pi_{k,\varepsilon}\) as \(\varepsilon\downarrow0\) lies in \(\mathcal O_k\) and has entropy no larger than that of \(\widehat\pi_{k,0}^{\mathrm{ME}}\). 

The uniqueness in \cref{eq:entropic-fixed-batch-me-plan} forces every cluster point to equal \(\widehat\pi_{k,0}^{\mathrm{ME}}\). Compactness then gives convergence of the full sequence in the finite-dimensional polytope, equivalently in entrywise \(\ell_1\) and total variation.

For every positive \(\varepsilon\), the unique optimizer is a Borel function of the sampled cost matrix by the measurable maximum theorem. Taking, for example, \(\varepsilon_n=1/n\), the pointwise convergence just proved shows that \(\widehat\pi_{k,0}^{\mathrm{ME}}\) is the pointwise limit of measurable optimizers and is therefore measurable as well. 

For a matrix \(P\in\mathcal T_k\), we denote \( \Gamma_P=\sum_{i,j=1}^kP_{ij}\,\delta_{(X_i,Y_j)} \) as its selected-pair endpoint law conditional on the batches. Pushforward contracts total variation, so we have
\[
\|\Gamma_{\widehat\pi_{k,\varepsilon}}
-\Gamma_{\widehat\pi_{k,0}^{\mathrm{ME}}}\|_{\mathrm{TV}}
\leq
\frac12\|\widehat\pi_{k,\varepsilon}
-\widehat\pi_{k,0}^{\mathrm{ME}}\|_1
\longrightarrow0
\]
for every realized pair of batches. Convexity of total variation and dominated convergence now yield
\[
\begin{aligned}
\|\pi_{k,\varepsilon}-\pi_{k,0}^{\mathrm{ME}}\|_{\mathrm{TV}}
&\leq
\mathbb E\!\left[
\|\Gamma_{\widehat\pi_{k,\varepsilon}}
-\Gamma_{\widehat\pi_{k,0}^{\mathrm{ME}}}\|_{\mathrm{TV}}
\right]
\longrightarrow0.
\end{aligned}
\]
This proves the first convergence in \cref{eq:entropic-fixed-batch-zero-temperature}.

We next prove the continuity of \(\mathsf R\) needed for the second convergence. More generally, let \(\rho_n,\rho\in\Pi(q_0,q_1)\) satisfy \(\|\rho_n-\rho\|_{\mathrm{TV}}\to0\), and let \(P_n,P\) be the corresponding laws of \(((t,X_t),U)\), where \(U=X_1-X_0\). Appending the same independent \(t\sim\mathcal U([0,1])\) preserves total-variation convergence, while measurable pushforward contracts it. Thus, we have \(\|P_n-P\|_{\mathrm{TV}}\to0\).

Let \(T_L:\mathbb R^d\to\overline B(0,L)\) denote metric projection, we set \(\mathsf R_L(\rho')\!=\!\inf_f\mathbb E_{\rho'}\!\left[\|T_L(U)-f(t,X_t)\|_2^2\right]\). The infimum may be restricted to predictors with values in \(\overline B(0,L)\), because projection onto this ball cannot increase the distance to \(T_L(U)\), so the resulting losses are bounded by \(4L^2\). 

Given \(\|\mu-\nu\|_{\mathrm{TV}}=\sup_A|\mu(A)-\nu(A)|\), comparison with an arbitrarily accurate minimizer under the other law gives us that
\begin{equation}
\label{eq:bounded-bayes-tv}
|\mathsf R_L(\rho_n)-\mathsf R_L(\rho)|
\leq8L^2\|P_n-P\|_{\mathrm{TV}}
\longrightarrow0.
\end{equation}
The final missing point is that we need to remove the truncation uniformly. We write \(A=\|X_0\|_2^2\) and \(B=\|X_1\|_2^2\); because their two fixed marginal laws are integrable, the de~la~Vall\'ee--Poussin criterion supplies a nondecreasing, convex, superlinear function \(\Phi\) such that
\[
\mathbb E_{q_0}[\Phi(A)]+\mathbb E_{q_1}[\Phi(B)]<\infty.
\]
For every \(\rho'\in\Pi(q_0,q_1)\), we have
\[
\frac{\|X_1-X_0\|_2^2}{4}
\leq\frac{A+B}{2},
\]
and monotonicity and convexity give the following bound
\[
\sup_{\rho'\in\Pi(q_0,q_1)}
\mathbb E_{\rho'}\!\left[
\Phi\!\left(\frac{\|X_1-X_0\|_2^2}{4}\right)
\right]
\leq
\frac12\mathbb E_{q_0}[\Phi(A)]
+\frac12\mathbb E_{q_1}[\Phi(B)]<\infty.
\]
Therefore, \(\{\|X_1-X_0\|_2^2:\rho'\in\Pi(q_0,q_1)\}\) is uniformly integrable, and we end up with
\begin{equation}\label{eq:uniform-target-truncation}
e_L=\sup_{\rho'\in\Pi(q_0,q_1)}\mathbb E_{\rho'}\!\left[\|U-T_L(U)\|_2^2\right]\longrightarrow 0.
\end{equation}
For any coupling \(\rho'\), the square root of its Bayes risk is the \(\ell_2\)-distance from \(U\) to the closed subspace of \((t,X_t)\)-measurable vector-valued functions. Since the distance to a closed set is $1$-Lipschitz, we have
\begin{equation} \label{eq:bayes-risk-truncation}
\left| \sqrt{\mathsf R(\rho')}-\sqrt{\mathsf R_L(\rho')} \right| \leq \left( \mathbb E_{\rho'}\|U-T_L(U)\|_2^2 \right)^{1/2} \leq e_L^{1/2}.
\end{equation}
Combining \cref{eq:bounded-bayes-tv,eq:uniform-target-truncation,eq:bayes-risk-truncation}, first taking \(n\!\to\!\infty\) at fixed \(L\) and then taking \(L\!\to\!\infty\), we get \(\mathsf R(\rho_n)\to\mathsf R(\rho)\). Applying this fact to the endpoint couplings above proves the second convergence in \cref{eq:entropic-fixed-batch-zero-temperature}.

Finally, consider \(E\subseteq\mathbb R^d\times\mathbb R^d\) be Borel with \(E_x=\{y:(x,y)\in E\}\). If \((q_0\otimes q_1)(E)=0\), all \(k^2\) sampled candidate pairs avoid \(E\) almost surely, which gives \(\pi_{k,0}^{\mathrm{ME}}(E)=0\). Conversely, if \((q_0\otimes q_1)(E)>0\), we just consider \(I\) be any uniformly selected source index, which gives us
\[
\mathbb P\!\left(Y_j\in E_{X_I}\text{ for every }j\right)=\int q_1(E_x)^k\,q_0(\mathrm dx)>0.
\]
On this event, whichever target index is selected from row \(I\), the selected pair still lies in \(E\). Therefore, \(\pi_{k,0}^{\mathrm{ME}}(E)>0\), proving \(\pi_{k,0}^{\mathrm{ME}}\sim\pi_{\text{ind}}\). Again, recall that the map
\[
(t,x_0,x_1) \longmapsto \bigl(t,(1-t)x_0+tx_1,x_1-x_0\bigr)
\]
is a Borel bijection, with inverse \((t,x,u)\mapsto(t,x-tu,x+(1-t)u)\). Therefore, equivalent endpoint laws induce equivalent laws of \((t,X_t,U)\). Since
\[ 
\mathsf R(\rho)=0\quad\Longleftrightarrow\quad U\ \text{is almost surely a measurable function of }(t,X_t),
\]
so this zero-risk graph property is invariant under equivalence, and we can then conclude \cref{eq:entropic-fixed-batch-equivalence}.
\end{proof}
\vspace{-0.5em}

Before establishing the proof for Theorem~\ref{thm:entropic-large-k-fixed-gradient}, we define \(D_\varepsilon=\mathbb E_{\pi_\varepsilon}[\|X_1-X_0\|_2^2]-W_2^2(q_0,q_1)\). The theorem's joint uniform condition means that there exist \(K_J\in\mathbb N\), \(\varepsilon_J>0\), and \(M<\infty\) such that
\begin{equation}\label{eq:entropic-large-k-uniform-jacobian}
\operatorname*{ess\,sup}\|J_{\theta_{\pi_{k,\varepsilon}}^\star}(t,X_t^{\pi_{k,\varepsilon}})\|_{\mathrm{op}}^2
\leq\max_{t\in[0,1],\,x\in\mathbb R^d}\|J_{\theta_{\pi_{k,\varepsilon}}^\star}(t,x)\|_{\mathrm{op}}^2
\leq M,
\ \ k\geq K_J,\ \ 0<\varepsilon\leq\varepsilon_J.
\end{equation}
Moreover, \cref{eq:eot-cost-convergence} gives \(D_\varepsilon\to0\) as \(\varepsilon\downarrow0\). Now, we can start the proof:

\vspace{-0.5em}
\begin{proof}[Proof of Theorem~\ref{thm:entropic-large-k-fixed-gradient}]
At a realizable coupling-specific population optimum, \cref{eq:var-at-optimum} gives the exact gradient-variance identity. The trainable additive output bias implies \(J_{\theta_{\pi_{k,\varepsilon}}^\star}J_{\theta_{\pi_{k,\varepsilon}}^\star}^\top\succeq I_d\) by Fact~\ref{fact:affine}, with both Jacobians evaluated at \((t,X_t^{\pi_{k,\varepsilon}})\). Applying the positive-semidefinite trace bounds used in Proposition~\ref{prop:gradient-variance-sandwich} proves both inequalities in \cref{eq:entropic-large-k-fixed-bound}.

We next obtain a positive lower bound for the prescribed finite baseline as \(\varepsilon\downarrow0\) and set \(r_0=\mathsf R(\pi_{k_0,0}^{\mathrm{ME}})\). By Lemma~\ref{lem:entropic-fixed-batch-zero-temperature} and the assumption \(\mathsf R(\pi_{\text{ind}})>0\), we have
\begin{equation*}
r_0>0,
\qquad
\mathsf R(\pi_{k_0,\varepsilon})\longrightarrow r_0
\quad\text{as }\varepsilon\downarrow0.
\end{equation*}
So, there exists \(\varepsilon_{\mathrm b}>0\) such that
\begin{equation*}
\mathsf R(\pi_{k_0,\varepsilon})>\frac{r_0}{2}
\qquad
\text{for every }0<\varepsilon\leq\varepsilon_{\mathrm b}.
\end{equation*}
By the lower inequality in \cref{eq:entropic-large-k-fixed-bound}, this gives
\begin{equation}
\label{eq:entropic-baseline-gradient-floor}
\sigma_{\text{grad},\pi_{k_0,\varepsilon}}^2
(\theta_{\pi_{k_0,\varepsilon}}^\star)
>2r_0
\qquad
\text{for every }0<\varepsilon\leq\varepsilon_{\mathrm b}.
\end{equation}
Now fix \(0<\varepsilon\leq\varepsilon_J\). Proposition~\ref{prop:entropic-large-k-fixed} and \cref{eq:entropic-large-k-uniform-jacobian} imply
\[
\limsup_{k\to\infty}\mathsf R(\pi_{k,\varepsilon})\leq D_\varepsilon,
\qquad
\limsup_{k\to\infty}\operatorname*{ess\,sup}\|J_{\theta_{\pi_{k,\varepsilon}}^\star}(t,X_t^{\pi_{k,\varepsilon}})\|_{\mathrm{op}}^2\leq M.
\]
Note that both sequences are nonnegative and the two right-hand sides are finite. Therefore, applying \cref{eq:entropic-large-k-fixed-bound} and the elementary limsup product inequality gives
\begin{equation}
\label{eq:entropic-large-k-gradient-limsup}
\limsup_{k\to\infty}
\sigma_{\text{grad},\pi_{k,\varepsilon}}^2
(\theta_{\pi_{k,\varepsilon}}^\star)
\leq
4M D_\varepsilon.
\end{equation}
Since \(D_\varepsilon\to0\), there exists \(\varepsilon_{\mathrm s}>0\) such that
\begin{equation}
\label{eq:entropic-small-temperature-gradient-ceiling}
4M D_\varepsilon<r_0
\qquad
\text{for every }0<\varepsilon\leq\varepsilon_{\mathrm s}.
\end{equation}
We set \(\bar\varepsilon= \min\{\varepsilon_J,\varepsilon_{\mathrm b}, \varepsilon_{\mathrm s}\}>0\). For every fixed \(0<\varepsilon\leq\bar\varepsilon\), \cref{eq:entropic-baseline-gradient-floor,eq:entropic-large-k-gradient-limsup,eq:entropic-small-temperature-gradient-ceiling} give
\[
\limsup_{k\to\infty} \sigma_{\text{grad},\pi_{k,\varepsilon}}^2 (\theta_{\pi_{k,\varepsilon}}^\star)<r_0 <2r_0 < \sigma_{\text{grad},\pi_{k_0,\varepsilon}}^2 (\theta_{\pi_{k_0,\varepsilon}}^\star).
\]
By the limit of superior, we have a finite threshold after which the strict upper comparison holds. Increasing that threshold if necessary, choose a finite \(K>\max\{k_0,K_J\}\), which may depend on \(k_0\) and \(\varepsilon\); then, \cref{eq:entropic-large-k-fixed-reduction} holds for every \(k\geq K\).

Now, it remains for us only to verify strict positivity. For every finite \(k\) and \(\varepsilon>0\), Proposition~\ref{prop:entropic-large-k-fixed} gives \(\pi_{k,\varepsilon}\sim\pi_{\text{ind}}\). As in Lemma~\ref{lem:entropic-fixed-batch-zero-temperature}, equivalence preserves the graph property characterizing zero target uncertainty. Therefore, \(\mathsf R(\pi_{k,\varepsilon})>0\), and the lower inequality in \cref{eq:entropic-large-k-fixed-bound} gives \(\sigma_{\text{grad},\pi_{k,\varepsilon}}^2>0\). This completes the proof, and we can conclude the variance reduction.
\end{proof}

\subsubsection{Proposition~\ref{prop:int-err-red}}
Before establishing the proof, we need to prove two supporting lemmas. We first prove the following one-step coupling bound:

\begin{lemma}\label{lem:one-step-coupling-facts}
Let $q_0,q_1\in\mathcal P_2(\mathbb R^d)$, let $\pi\in\Pi(q_0,q_1)$, and we denote $m_\pi(x)=\mathbb E_\pi[X_1\mid X_0=x]$. Then, we have
\begin{equation}
\label{eq:one-step-coupling-facts}
0\leq\operatorname{err}(\pi)\leq\mathbb E_\pi[\|X_1-m_\pi(X_0)\|_2^2]=\mathbb E[\|X_1-\mathbb E[X_1]\|_2^2]-\mathbb E[\|m_\pi(X_0)-\mathbb E[X_1]\|_2^2].
\end{equation}
Moreover, for every measurable $g\in L^2(q_0;\mathbb R^d)$, the following reduction holds true:
\begin{equation}
\label{eq:one-step-predictor-envelope}
\operatorname{err}(\pi)\leq\mathbb E_\pi[\|X_1-m_\pi(X_0)\|_2^2]\leq\mathbb E_\pi[\|X_1-g(X_0)\|_2^2].
\end{equation}
If $\pi\sim q_0\otimes q_1$ and $q_1$ is not a Dirac mass, then $\operatorname{err}(\pi)>0$.
\end{lemma}

\begin{proof}
The proof is quite straightforward. Specifically, the pair $(m_\pi(X_0),X_1)$ couples $(m_\pi)_\#q_0$ with $q_1$, which gives the first upper bound. Conditional orthogonality gives the equality in \cref{eq:one-step-coupling-facts} and the predictor inequality in \cref{eq:one-step-predictor-envelope}.

Now suppose $\pi\!\sim\!q_0\otimes q_1$. For $q_0$-almost every $x$, the conditional law $\pi(\mathrm dy\mid X_0=x)$ is equivalent to the non-Dirac law $q_1$, so its conditional variance is strictly positive. If $\operatorname{err}(\pi)=0$, then $(m_\pi)_\#q_0=q_1$, while conditional orthogonality gives us that
\[
\mathbb E[\|X_1\|_2^2]=\mathbb E[\|m_\pi(X_0)\|_2^2]+\mathbb E_\pi[\|X_1-m_\pi(X_0)\|_2^2]>\mathbb E[\|X_1\|_2^2],
\]
which is a contradiction, so this proves the lemma.
\end{proof}
\vspace{-0.5em}
We next prove the convergence to Monge coupling:
\begin{lemma}
\label{lem:monge-one-step-convergence}
Let $T\in L^2(q_0;\mathbb R^d)$ and let $\gamma_n\in\Pi(q_0,q_1)$ satisfy $\gamma_n\to(\mathrm{id},T)_\#q_0$ in $W_2$. Then,
\begin{equation}
\label{eq:monge-one-step-convergence}
\mathbb E_{\gamma_n}[\|X_1-T(X_0)\|_2^2]\longrightarrow0,
\qquad
\operatorname{err}(\gamma_n)\longrightarrow0.
\end{equation}
\end{lemma}

\begin{proof}
Fix $g\in C_b(\mathbb R^d;\mathbb R^d)$. Since $(x,y)\mapsto\|y-g(x)\|_2^2$ is continuous with quadratic growth, $W_2$ convergence and the common first marginal give
\[
\limsup_{n\to\infty}\mathbb E_{\gamma_n}[\|X_1-T(X_0)\|_2^2]\leq4\|g-T\|_{L^2(q_0)}^2.
\]
Bounded continuous functions are dense in $L^2(q_0)$, so the first limit follows. The second follows from \cref{eq:one-step-predictor-envelope} with $g=T$.
\end{proof}
\vspace{-0.5em}
Given these two supporting lemmas, we can now prove Proposition~\ref{prop:int-err-red}:
\vspace{-0.75em}
\begin{proof}[Proof of Proposition~\ref{prop:int-err-red}]
For population OT, $X_1=T(X_0)$ almost surely, so the one-step output has law $q_1$ and $\operatorname{err}(\pi^\star)=0$. Under independent coupling, $\mathbb E[X_1\mid X_0]=\mathbb E[X_1]$, so we can show that 
\[
\operatorname{err}(\pi_{\text{ind}})=\mathbb E[\|X_1-\mathbb E[X_1]\|_2^2]>0.
\]
We next prove the two population-entropic claims. First, we write objective as
\[
\mathcal J_\varepsilon(\pi)=\frac{1}{2d}\mathbb E_\pi[\|X_1-X_0\|_2^2]+\varepsilon\operatorname{KL}(\pi\|\pi_{\text{ind}}).
\]
Since $\pi_{\text{ind}}$ has finite objective, $\pi_\varepsilon$ is absolutely continuous with respect to $\pi_{\text{ind}}$, and there exists a Radon--Nikodym density $f$ such that $\mathrm d\pi_\varepsilon=f\,\mathrm d\pi_{\text{ind}}$. 

We claim that $f>0$ almost everywhere, which can be proved via contradiction, i.e., otherwise if $A=\{f=0\}$ had positive $\pi_{\text{ind}}$-mass, the feasible mixture $\pi_s=(1-s)\pi_\varepsilon+s\pi_{\text{ind}}$ would satisfy, with $h(r)=r\log r$, that
\[
\frac{\mathcal J_\varepsilon(\pi_s)-\mathcal J_\varepsilon(\pi_\varepsilon)}{s}\leq\frac{\mathbb E_{\pi_{\text{ind}}}[\|X_1-X_0\|_2^2]-\mathbb E_{\pi_\varepsilon}[\|X_1-X_0\|_2^2]}{2d}+\varepsilon\!\left\{\pi_{\text{ind}}(A)\log s-\operatorname{KL}(\pi_\varepsilon\|\pi_{\text{ind}})\right\},
\]
which tends to $-\infty$ as $s\downarrow0$ and contradicts the optimality, so $\pi_\varepsilon\sim\pi_{\text{ind}}$, and Lemma~\ref{lem:one-step-coupling-facts} gives $\operatorname{err}(\pi_\varepsilon)>0$.

We next compare $\pi_\varepsilon$ strictly with independence. We choose $w$ such that $\langle w,X_1-\mathbb E[X_1]\rangle$ is not a constant and set $b(X_1)=\tanh\langle w,X_1-\mathbb E[X_1]\rangle-\mathbb E[\tanh\langle w,X_1-\mathbb E[X_1]\rangle]$, and it's clear that $v\!\coloneqq\!\mathbb E[(X_1-\mathbb E[X_1])b(X_1)]\neq0$. Since $q_0\ll\mathcal L^d$, the variable $\langle v,X_0-\mathbb E[X_0]\rangle$ is not a constant. 

We define $a(X_0)$ by applying and centering $\tanh$ in the same way, then rescaling it so that both $\|a\|_\infty,\|b\|_\infty\leq1$. For $|s|<1$, the density $1+s,a(x)b(y)$ defines a coupling $\rho_s\in\Pi(q_0,q_1)$, and its entropy derivative at $s=0$ vanishes, whereas
\[
\left.\frac{\mathrm d}{\mathrm ds}\frac{1}{2d}\mathbb E_{\rho_s}[\|X_1-X_0\|_2^2]\right|_{s=0}=-\frac1d\left\langle\mathbb E[(X_0-\mathbb E[X_0])a(X_0)],v\right\rangle<0.
\]
Therefore, $\mathcal J_\varepsilon(\pi_\varepsilon)\leq\mathcal J_\varepsilon(\rho_s)<\mathcal J_\varepsilon(\pi_{\text{ind}})$ for all sufficiently small $s>0$, and non-negativity of relative entropy gives us that
\[
\mathbb E_{\pi_\varepsilon}[\|X_1-X_0\|_2^2]<\mathbb E_{\pi_{\text{ind}}}[\|X_1-X_0\|_2^2].
\]
Expanding the two costs shows that $\mathbb E_{\pi_\varepsilon}[X_1\mid X_0]$ is not almost surely equal to $\mathbb E[X_1]$. The strict upper comparison $\operatorname{err}(\pi_\varepsilon)<\operatorname{err}(\pi_{\text{ind}})$ now follows from \cref{eq:one-step-coupling-facts}.

Finally, choosing any minibatch batch size $k$, we denote $(Z_1,\ldots,Z_k)$ and $(Y_1,\ldots,Y_k)$ as the endpoint batches, let $I$ be any uniform source index, and let $\tau$ be the empirical matching. If $(q_0\otimes q_1)(E)=0$, every candidate pair $(Z_i,Y_j)$ avoids $E$ almost surely, so $\pi_k(E)=0$. Conversely, if $(q_0\otimes q_1)(E)>0$ and $E_z=\{y:(z,y)\in E\}$, then
\[
\pi_k(E)\geq\mathbb P\bigl(Y_j\in E_{Z_I}\text{ for every }j\bigr)=\int q_1(E_z)^kq_0(\mathrm dz)>0.
\]
Therefore, $\pi_k\sim q_0\otimes q_1$, and Lemma~\ref{lem:one-step-coupling-facts} gives $\operatorname{err}(\pi_k)>0$.

Now, recall that Proposition~\ref{prop:large-k-intrinsic-consistency} already proves the convergence of $\pi_k$ to $W_2^2(q_0,q_1)$. Also, the tightness, lower semicontinuity, and uniqueness of the Brenier coupling imply $\pi_k\Rightarrow\pi^\star$. Since all couplings have the same product-space second moment, this convergence holds in $W_2$. Lemma~\ref{lem:monge-one-step-convergence} therefore gives
\[
\operatorname{err}(\pi_k)\leq\mathbb E_{\pi_k}[\|X_1-T(X_0)\|_2^2]\longrightarrow0.
\]
Since every fixed finite baseline has positive error, we can conclude the strict error reduction.
\end{proof}

\subsection{Theorem~\ref{thm:main-entropic-integration}}
\begin{proof}
For simplicity, we abbreviate the following notations: $\mu_i=\mathbb E[X_i]$, $V_i=\mathbb E[\|X_i-\mu_i\|_2^2]$, and $C(\pi)=\mathbb E_\pi[\|X_1-X_0\|_2^2]$ with $C_{\text{ind}}=C(\pi_{\text{ind}})$. Since $q_0\ll\mathcal L^d$ and $q_0$ is non-Dirac, we have \(V_0>0\). For any coupling $\pi$, expanding the two quadratic costs gives
\[
C_{\text{ind}}-C(\pi)=2\mathbb E_\pi[\langle X_0-\mu_0,\mathbb E_\pi[X_1\mid X_0]-\mu_1\rangle].
\]
Together with Cauchy--Schwarz and Lemma~\ref{lem:one-step-coupling-facts}, we have
\begin{equation}
\label{eq:main-proof-cost-certificate}
\operatorname{err}(\pi)\leq V_1-\frac{\{C_{\text{ind}}-C(\pi)\}^2}{4V_0}.
\end{equation}
To apply this inequality, we consider $\mathcal B$ as two sampled endpoint batches. Under the actual sampling law $\mathbb P_{\mathrm{ot}}$, the conditional distribution of the selected indices $(I,J)$ given $\mathcal B$ is $\widehat\pi_{k,\varepsilon}$. We can define a reference law $\mathbb P_{\mathrm{ref}}$ using the same batches but conditionally independent uniform indices, so $(I,J)\mid\mathcal B\sim\mathbf1\mathbf1^\top/k^2$. The map $(\mathcal B,I,J)\mapsto(X_I,Y_J)$ pushes $\mathbb P_{\mathrm{ot}}$ forward to $\pi_{k,\varepsilon}$ and $\mathbb P_{\mathrm{ref}}$ forward to $\pi_{\text{ind}}$, so the conditional KL chain rule, along with the same data processing proof structure from Lemma~\ref{lem:entropic-shadow-bound}, give us that
\begin{equation}\label{eq:doublelaw}
\operatorname{KL}(\pi_{k,\varepsilon}\|\pi_{\text{ind}})\leq\operatorname{KL}(\mathbb P_{\mathrm{ot}}\|\mathbb P_{\mathrm{ref}})=\mathbb E\!\left[\operatorname{KL}\!\left(\widehat\pi_{k,\varepsilon}\middle\|\mathbf1\mathbf1^\top/k^2\right)\right].
\end{equation}
Moreover, selecting $(I,J)$ according to $\widehat\pi_{k,\varepsilon}$ gives us that
\begin{equation}\label{eq:cost-selected-pairs}
\frac{C(\pi_{k,\varepsilon})}{2d}=\mathbb E\!\left[\left\langle\mathbf C,\widehat\pi_{k,\varepsilon}\right\rangle\right].
\end{equation}
Adding \cref{eq:doublelaw} to \cref{eq:cost-selected-pairs} and recalling the definition of $S_{k,\varepsilon}$, we have
\[
\frac{C(\pi_{k,\varepsilon})}{2d}+\varepsilon\operatorname{KL}(\pi_{k,\varepsilon}\|\pi_{\text{ind}})\leq S_{k,\varepsilon}.
\]
Since $S_{1,\varepsilon}=C_{\text{ind}}/(2d)$, we obtain
\[
C_{\text{ind}}-C(\pi_{k,\varepsilon})\geq2d(S_{1,\varepsilon}-S_{k,\varepsilon}).
\]
Substitution into \cref{eq:main-proof-cost-certificate} proves the middle inequality in \cref{eq:main-entropic-finite-k-chain}. Proposition~\ref{prop:entropic-large-k-fixed} gives $\pi_{k,\varepsilon}\sim\pi_{\text{ind}}$ at every finite $k$, so strict positivity follows from Lemma~\ref{lem:one-step-coupling-facts}, and we finish proving the first statement in the theorem.

We now prove the second statement using the idea of induction, which the upper bound is nonincreasing and strict for $k\geq2$. 

First, we set $n=k+1$ and fix two realized $n$-point batches. For each source index $a$ and target index $b$, remove those two points, let $P^{ab}$ solve the resulting $k\times k$ problem, embed it into an $n\times n$ matrix $Q^{ab}$ by inserting a zero row and column, and set $\overline P=n^{-2}\sum_{a,b}Q^{ab}$. We note that this matrix has uniform $n$-point marginals. If
\[
M(a,b,i,j)=n^{-2}Q^{ab}_{ij},\qquad R(a,b,i,j)=\frac{\mathbf1_{\{i\neq a,\,j\neq b\}}}{n^2k^2},
\]
then their $(i,j)$-marginals are respectively $\overline P$ and the uniform $n\times n$ matrix. Data processing gives
\[
\operatorname{KL}(\overline P\|\mathbf1\mathbf1^\top/n^2)\leq\frac1{n^2}\sum_{a,b}\operatorname{KL}(P^{ab}\|\mathbf1\mathbf1^\top/k^2).
\]
The same averaging identity holds for the cost, so taking expectations proves $S_{k+1,\varepsilon}\leq S_{k,\varepsilon}$.

We now prove the strictness part, and it's enough for us to start with $k=2$. 

Let $U_2=\mathbf1\mathbf1^\top/4$ and $D=\left(\begin{smallmatrix}1&-1\\-1&1\end{smallmatrix}\right)$. The feasible perturbation $U_2+sD$ has zero entropy derivative at $s=0$, whereas its cost derivative is
\[
-\frac1d\langle Z_1-Z_2,Y_1-Y_2\rangle.
\]
Since $q_1$ is non-Dirac, $Y_1-Y_2\neq0$ with positive probability; conditional on any such value, absolute continuity of $q_0$ makes the displayed inner product nonzero almost surely. Therefore, the empirical optimum is strictly better than $U_2$ with positive probability. 

Also, the nonnegative gap is integrable, so $S_{2,\varepsilon}<S_{1,\varepsilon}$, and monotonicity extends this strict inequality to every $k\geq2$. This proves both the strict final inequality and monotonicity of the upper bound in \cref{eq:main-entropic-finite-k-chain}.

We now proceed to prove the final \cref{eq:main-entropic-matched-baseline}. Similar to before, we write the population value as
\[
S_\varepsilon\coloneqq \inf_{\pi\in\Pi(q_0,q_1)}\left\{\frac{C(\pi)}{2d}+\varepsilon\operatorname{KL}(\pi\|\pi_{\text{ind}})\right\}.
\]
The index-to-measure reduction and uniform shadow bound in Lemmas~\ref{lem:entropic-index-measure-reduction} and~\ref{lem:entropic-shadow-bound} give the following, respectively, 
\begin{equation}
\label{eq:main-proof-shadow}
0\leq S_{k,\varepsilon}-S_\varepsilon\leq e_k:=\frac{2\sqrt{C_{\text{ind}}}\,\mathbb E[a_k]+\mathbb E[a_k^2]}{2d},\qquad a_k=W_2(\widehat q_{0,k},q_0)+W_2(\widehat q_{1,k},q_1).
\end{equation}
For every $q\in\mathcal P_2(\mathbb R^d)$, $\mathbb E[W_2^2(\widehat q_k,q)]\to0$ (i.e., project onto a fixed ball, use bounded-support empirical $W_2$ convergence and dominated convergence, and then send the radius to infinity), so $e_k\to0$.

We set $C_\star=W_2^2(q_0,q_1)$ and try to first show $S_\varepsilon\to C_\star/(2d)$. First, we note that the optimal coupling $\pi^\star$ may have infinite relative entropy with respect to $\pi_{\text{ind}}$, so we do the following math technique by approximating it through finite-entropy couplings.

Specifically, we first partition both endpoint spaces into increasingly fine finite partitions and set $p_{AB}=\pi^\star(A\times B)$, and we set $\rho_n$ by assigning mass $p_{AB}$ to each block $A\times B$ according to the product measure $q_0(\mathrm dx\mid A)q_1(\mathrm dy\mid B)$. Since $\sum_Bp_{AB}=q_0(A)$ and $\sum_Ap_{AB}=q_1(B)$, $\rho_n$ has marginals $q_0$ and $q_1$. Moreover, we have
\[
\operatorname{KL}(\rho_n\|\pi_{\text{ind}})
=\sum_{A,B}p_{AB}\log\frac{p_{AB}}{q_0(A)q_1(B)}<\infty.
\]
As the partitions become finer, $\rho_n$ converges weakly to $\pi^\star$. Since all $\rho_n$ have the same marginals as $\pi^\star$, their product-space second moments also agree, so $\rho_n\!\to\!\pi^\star$ in $W_2$. Therefore, $C(\rho_n)\to C_\star$, and
\[
\frac{C_\star}{2d}\leq S_\varepsilon\leq\frac{C(\rho_n)}{2d}+\varepsilon\operatorname{KL}(\rho_n\|\pi_{\text{ind}}).
\]
To prove the claim, we simply need to take $\varepsilon\downarrow0$ and $n\to\infty$. Now, let $\varepsilon_k\to0$. The raw selected-pair cost, i.e., \cref{eq:main-proof-shadow}, and the preceding population limit give
\[
\frac{C_\star}{2d}\leq\frac{C(\pi_{k,\varepsilon_k})}{2d}\leq S_{k,\varepsilon_k}\leq S_{\varepsilon_k}+e_k\longrightarrow\frac{C_\star}{2d}.
\]
Tightness and uniqueness of the Brenier coupling imply $\pi_{k,\varepsilon_k}\Rightarrow\pi^\star$, and the common product-space second moment upgrades this to $W_2$ convergence. Lemma~\ref{lem:monge-one-step-convergence} therefore gives $\operatorname{err}(\pi_{k,\varepsilon_k})\to0$, with no rate condition relating $k$ and $\varepsilon_k$.

To derive \cref{eq:main-entropic-matched-baseline}, we still need to use the following continuity fact: let $\rho_n,\rho\in\Pi(q_0,q_1)$ satisfy $\|\rho_n-\rho\|_{\mathrm{TV}}\to0$, and we can apply the technique by disintegrating them as $q_0(\mathrm dx)K_n(x,\mathrm dy)$ and $q_0(\mathrm dx)K(x,\mathrm dy)$. Their common first marginal gives
\[
\|\rho_n-\rho\|_{\mathrm{TV}}=\int\|K_n(x,\cdot)-K(x,\cdot)\|_{\mathrm{TV}}q_0(\mathrm dx).
\]
For metric projection $P_L$ onto the radius-$L$ ball, the corresponding truncated conditional means satisfy
\[
\int\left\|\int P_L(y)K_n(x,\mathrm dy)-\int P_L(y)K(x,\mathrm dy)\right\|_2^2q_0(\mathrm dx)\leq4L^2\|\rho_n-\rho\|_{\mathrm{TV}}.
\]
Conditional Jensen bounds each truncation tail by $\int\|y-P_L(y)\|_2^2q_1(\mathrm dy)$, and we can take $n\to\infty$ and then $L\to\infty$, which proves convergence of the conditional means in $L^2(q_0)$. 

Coupling their pushforward laws through $X_0\sim q_0$ gives $\operatorname{err}(\rho_n)\to\operatorname{err}(\rho)$, and we now derive the final ingredient. Fix \(\varepsilon>0\), and recall the population objective
\[
J_\varepsilon(\rho)
:=
\frac{C(\rho)}{2d}
+
\varepsilon\operatorname{KL}(\rho\|\pi_{\text{ind}}),
\qquad
S_\varepsilon=J_\varepsilon(\pi_\varepsilon).
\]
Since \(\pi_{k,\varepsilon}\in\Pi(q_0,q_1)\), population optimality and the preceding bounds on the averaged coupling and empirical value give
\[
S_\varepsilon
\leq
J_\varepsilon(\pi_{k,\varepsilon})
\leq
S_{k,\varepsilon}
\leq
S_\varepsilon+e_k.
\]
Therefore, the population objective gap satisfies
\[
0
\leq
J_\varepsilon(\pi_{k,\varepsilon})
-
J_\varepsilon(\pi_\varepsilon)
\leq e_k
\longrightarrow0.
\]
To translate this objective convergence into convergence of the couplings, we apply the entropy Pythagorean inequality (see Lemma 4.4 in \citet{eckstein2022quantitative}) to the rescaled objective \(J_\varepsilon/\varepsilon\). This gives
\[
\operatorname{KL}(\pi_{k,\varepsilon}\|\pi_\varepsilon)
\leq
\frac{
J_\varepsilon(\pi_{k,\varepsilon})
-
J_\varepsilon(\pi_\varepsilon)
}{\varepsilon}
\leq
\frac{e_k}{\varepsilon}.
\]
Using the convention \(\|\mu-\nu\|_{\mathrm{TV}}=\sup_A|\mu(A)-\nu(A)|\), Pinsker's inequality then yields
\[
\|\pi_{k,\varepsilon}-\pi_\varepsilon\|_{\mathrm{TV}}
\leq
\sqrt{\frac12\operatorname{KL}(\pi_{k,\varepsilon}\|\pi_\varepsilon)}
\leq
\sqrt{\frac{e_k}{2\varepsilon}}
\longrightarrow0,
\]
where the limit follows because \(\varepsilon>0\) is fixed and \(e_k\to0\). Both couplings have the same marginals \(q_0,q_1\), so the total-variation continuity established above applies and gives
\[
\operatorname{err}(\pi_{k,\varepsilon})
\longrightarrow
\operatorname{err}(\pi_\varepsilon)
\qquad\text{as }k\to\infty.
\]
This proves the fixed-regularization convergence claimed in the theorem.

For any fixed finite $k_0$, Lemma~\ref{lem:entropic-fixed-batch-zero-temperature} gives $\pi_{k_0,\varepsilon}\to\pi_{k_0,0}^{\mathrm{ME}}$ in total variation and $\pi_{k_0,0}^{\mathrm{ME}}\sim\pi_{\text{ind}}$. With the continuity, along with Lemma~\ref{lem:one-step-coupling-facts}, we have
\[
\operatorname{err}(\pi_{k_0,\varepsilon})\longrightarrow\operatorname{err}(\pi_{k_0,0}^{\mathrm{ME}})>0,
\]
i.e., the candidate error tends to zero, so \cref{eq:main-entropic-matched-baseline} holds for all sufficiently large $k$, and we can finally conclude the theorem.
\end{proof}

\begin{figure}[!t]
    \centering
\includegraphics[width=\linewidth]{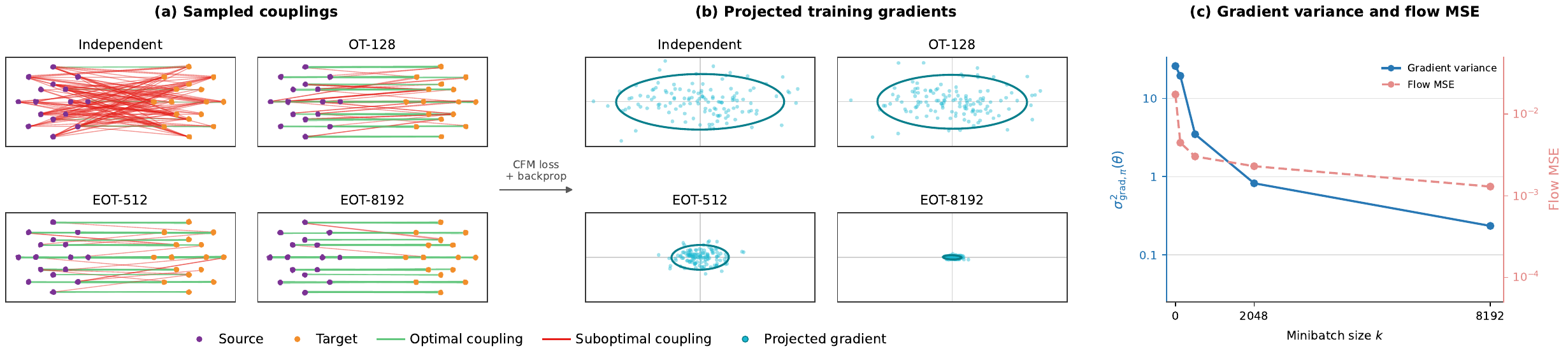}
\vspace{-1.5em}
    \caption{Results on a synthetic ring task over independent random coupling, OT-CFM using the minibatch size of $128$, and EOT-CFM with contexts up to $8192$. From left to right, each panel corresponds the flow couplings, the projected training gradients, and the final gradient variance and generation quality under different context size. All the training setup follows from Figure~\ref{f1}.}
    \label{f4}
\end{figure}

\section{Further Empirical Results}\label{app:exp}

\paragraph{Additional 2D examples.} We provide additional experimental results and details. First, we present additional 2D synthetic experiment on rings in Figure~\ref{f4}, using the same training settings as the multimodal block synthetic experiment in Figure~\ref{f1}. The results show a similar conclusion, where larger minibatch sizes consistently lead to more stable training and lower flow-generation error.

\begin{figure}[!t]
    \centering
\includegraphics[width=\linewidth]{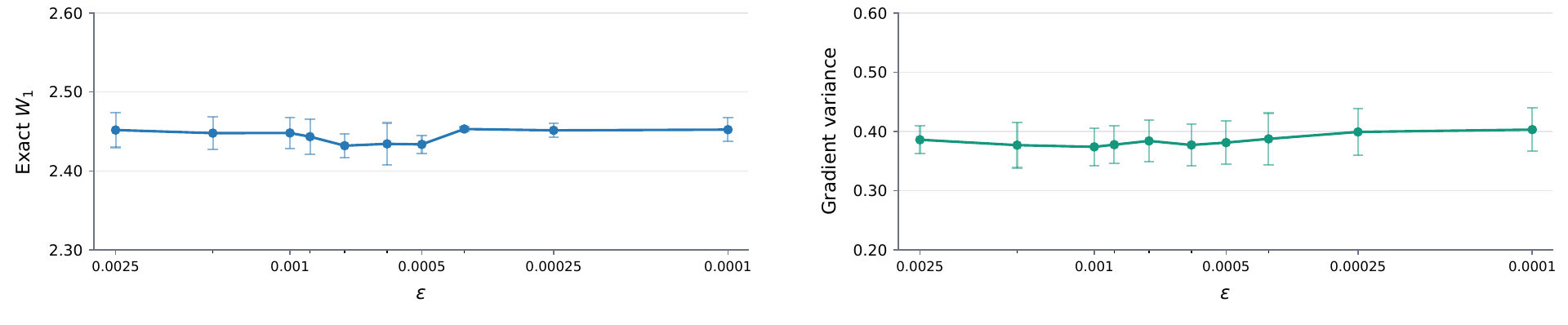}
\vspace{-2em}
    \caption{Ablation over $\varepsilon$ on the PBMC dataset with an 8K entropic OT context, with the exact $W_1$ error of generated samples (left) and training gradient variance (right) shown. We sweep $\varepsilon$ from $\{0.0025, 0.0015, 0.001, 0.0009, 0.00075, 0.0006, 0.0005, 0.0004, 0.00025, 0.0001\}$.}
    \label{f5}
    \vspace{-0.5em}
\end{figure}

\paragraph{Ablation on $\varepsilon$.} We also ablate $\varepsilon$ on the PBMC dataset, whose size lets us reach substantially larger minibatch contexts than the EB dataset. Offline pipelines that scale to million-level contexts must drive $\varepsilon$ to very low levels to keep the coupling sharp, which makes its tuning delicate. Online couplings instead pair each freshly sampled source point within the current minibatch, so that a moderate context of a few thousand points already yields a clear scaling effect while remaining robust to the choice of $\varepsilon$. 

As shown in Figure~\ref{f5}, we sweep $\varepsilon$ from $0.0001$ to $0.0025$ at 8K minibatch entropic OT context. Following the exactly same training configuration detailed in \S\ref{sec:exp}, each $\varepsilon$ trial is running over 5 seeds, with the mean performance and its standard deviation shown. Performance peaks around $\varepsilon = 0.001$ and varies little across the swept range. For EB training, we use a constant $\varepsilon=0.01$ across all context sizes, since the largest context can only be scaled to 1.5K. For image flow training, where the context can be scaled similarly to PBMC, the optimal $\varepsilon$ at the largest minibatch context is $0.002$ for CIFAR and $0.001$ for ImageNet, consistent with the 8K-context ablation.

\end{document}